\documentclass[twoside,11pt]{article}

\usepackage{lastpage}
\PassOptionsToPackage{bookmarks={false}}{hyperref}
\usepackage{general}

\usepackage{mathrsfs}
\usepackage{amsfonts}
\usepackage{amsmath}
\usepackage{amssymb}
\usepackage{bm}
\usepackage{appendix}
\usepackage{url}
\usepackage{verbatim}
\usepackage{graphicx} 
\usepackage{epsfig} 
\usepackage{subfigure}
\usepackage{epstopdf}
\usepackage{multirow}
\usepackage[table]{xcolor}
\usepackage{pifont}
\usepackage{thm-restate}

\usepackage[lined,boxed,ruled, commentsnumbered]{algorithm2e}

\newcommand{\argmin}{\mathop{\arg\min}}

\newcommand{\tabincell}[2]{\begin{tabular}{@{}#1@{}}#2\end{tabular}}
\newcommand{\InexactDane}{I\textsc{nexact}D\textsc{ane} }
\newcommand{\xmark}{\ding{55}}%

\definecolor{gray}{rgb}{0.85,0.85,0.85}
\definecolor{yama}{rgb}{0.98, 0.87, 0.68}
\definecolor{lightskyblue}{rgb}{0.53, 0.80, 0.99}

\definecolor{LightCyan}{rgb}{0.88,1,1}

\ShortHeadings{On Convergence of Distributed Approximate Newton Methods}{Xiao-Tong Yuan and Ping Li}
\firstpageno{1}

\begin{document}

\title{On Convergence of Distributed Approximate Newton Methods: Globalization, Sharper Bounds and Beyond}

\author{\name Xiao-Tong Yuan \email xtyuan1980@gmail.com\\
       \addr Cognitive Computing Lab\\
       Baidu Research\\
       Beijing 100085, China\\
       \AND
        \name Ping Li \email pingli98@gmail.com \\
       \addr Cognitive Computing Lab\\
       Baidu Research \\
       Bellevue, WA 98004, USA
       }

\editor{}

\maketitle

\begin{abstract}
The DANE algorithm is an approximate Newton method popularly used for communication-efficient distributed machine learning. Reasons for the interest in DANE include scalability and versatility. Convergence of DANE, however, can be tricky; its appealing convergence rate is only rigorous for quadratic objective, and for more general convex functions the known results are no stronger than those of the classic first-order methods. To remedy these drawbacks, we propose in this paper some new alternatives of DANE which are more suitable for analysis. We first introduce a simple variant of DANE equipped with backtracking line search, for which global asymptotic convergence and sharper local non-asymptotic convergence rate guarantees can be proved for both quadratic and non-quadratic strongly convex functions. Then we propose a heavy-ball method to accelerate the convergence of DANE, showing that nearly tight local rate of convergence can be established for strongly convex functions, and with proper modification of algorithm the same result applies globally to linear prediction models. Numerical evidence is provided to confirm the theoretical and practical advantages of our methods.
\end{abstract}

\vspace{0.1in}


\section{Introduction}

Distributed learning is a promising tool for alleviating the pressure of ever increasing data and/or model scale in modern machine learning systems. In this paper, we study the distributed optimization algorithms for solving the following empirical risk minimization (ERM) problem
\begin{equation}\label{eqn:general}
\min_{w\in \mathbb{R}^p} F(w):=\frac{1}{N}\sum_{i=1}^N f (w; x_i, y_i),
\end{equation}
where $\{x_i,y_i\}_{i=1}^N$ are training samples, $f$ is a smooth convex loss function. Such a finite-sum formulation encapsulates a large body of statistical learning problems including least square regression, logistic regression and support vector machines, to name a few. We assume without loss of generality that the training data $\mathcal{D}=\{D_1,...,D_m\}$ with $N=mn$ samples is evenly and randomly distributed over $m$ different machines; each machine $j$ locally stores and accesses $n$ training samples $D_j=\{x_{ji}, y_{ji}\}_{i=1}^{n}$. Let us denote $F_j(w):=\frac{1}{n}\sum_{i=1}^{n}f(w; x_{ji},y_{ji})$ the local empirical risk evaluated on $D_j$. The global objective is then to minimize the average of these local empirical risk functions:
\begin{equation}\label{eqn:problem}
\min_{w\in \mathbb{R}^p} F(w)=\frac{1}{m}\sum_{j=1}^m F_j(w).
\end{equation}
Recently, significant interest has been dedicated to designing distributed algorithms and systems that have flexibility to adapt to the communication-computation tradeoffs, e.g., for parameter estimation~\citep{jaggi2014communication,shamir2014communication} and statistical inference~\citep{jordan2018communication,wang2017efficient}. A common spirit of these communication-efficient methods is trying to quickly optimize the objective value (or estimation accuracy) to certain precision using a minimal number of inter-machine communication rounds.

In this paper we revisit the Distributed Approximate NEwton (DANE) algorithm proposed by~\citet{shamir2014communication} for solving~\eqref{eqn:problem}, which is now one of the most popular second-order methods for communication-efficient distributed machine learning. We analyze its convergence behavior, expose problems and issues, and propose alternative algorithms more suitable for the task. We contribute to derive some new results, insights and algorithms, using a unified and more elementary framework of Lyapunov analysis.

\subsection{Review of the DANE algorithm}\label{ssect:review}

For the distributed ERM problem~\eqref{eqn:problem}, the iteration (communication) complexity of first-order distributed approaches including (accelerated) gradient descent and ADMM (alternating direction method of multipliers)~\citep{boyd2011distributed} tend to suffer from the unsatisfactory polynomial dependence on condition number. To tackle this problem, \citet{shamir2014communication} proposed the DANE method that takes advantage of the stochastic nature of problem: the i.i.d. data samples $\{x_i,y_i\}$ are uniformly distributed and each local subproblem should be close to the global problem when data size becomes sufficiently large. At the $t$-th iteration loop of DANE, in parallel each individual \emph{worker} machine $j$ optimizes a local subproblem $w_j^{(t)} = \argmin_{w }P_j^{(t-1)}(w)$ in which
\begin{equation}\label{equat:local_obj}
\begin{aligned}
P_j^{(t-1)}(w):=&\langle \eta\nabla F(w^{(t-1)}) - \nabla F_j(w^{(t-1)}), w \rangle + \frac{\gamma}{2}\|w-w^{(t-1)}\|^2 + F_j(w).
\end{aligned}
\end{equation}
Then the \emph{master} machine computes and broadcasts the averaged model $w^{(t)}=\frac{1}{m}\sum_{j=1}^m w_j^{(t)}$ and its full gradient $\nabla F(w^{(t)})=\frac{1}{m}\sum_{j=1}^m \nabla F_j(w^{(t)})$ in a map-reduce fashion.

The construction of the local objective~\eqref{equat:local_obj} is inspired by the idea of leveraging the global first-order information and local higher-order information for local processing. If $F(w)$ is quadratic with condition number $\kappa = L/\mu$ (see Table~\ref{tab:notation} for notation), the communication complexity (with tail bound $\delta$) of DANE to reach $\epsilon$-precision was shown to be $\mathcal{\tilde O}\left(\frac{\kappa^2}{n}\log\left(\frac{mp}{\delta}\right)\log \left(\frac{1}{\epsilon}\right)\right)$ which has an improved dependency on the condition number $\kappa$ that could scale as large as $\mathcal{O}(\sqrt{mn})$ in statistical learning problems. \InexactDane~\citep{reddi2016aide} is an inexact implementation of DANE that allows the local sub-problem to be solved inexactly but still possess the above improved communication complexity bounds for quadratic problems. By applying Nesterov's acceleration technique, AIDE~\citep{reddi2016aide} and MP-DANE~\citep{wang2017memory} further reduce the communication complexity to $\mathcal{\tilde O}\left(\frac{\sqrt{\kappa}}{n^{1/4}}\log\left(\frac{mp}{\delta}\right)\log \left(\frac{1}{\epsilon}\right)\right)$ in the quadratic case, which is nearly tight in view of the lower bound established by~\citet{arjevani2015communication}.

On top of the high efficiency in communication, another practically appealing aspect of DANE lies in its versatility. This is because by nature DANE is an algorithm-agnostic meta-optimization framework, in the sense that the local subproblems can be solved by applying virtually any algorithms designed for the global problem. From the perspective of implementation, this enables fast transplant of the available single-machine program code onto distributed software platform. This contrasts DANE from those algorithm-specific methods such as DiSCO~\citep{zhang2015disco} (rooted from the damped Newton method) and DSVRG~\citep{lee2017distributed,shamir2016without} (rooted from SVRG). What's more, DANE does not require to access a second-order oracle for its execution, nor does it restrict to any specific problem structure such as the linear prediction models focused by DSCOVR~\citep{xiao2019dscovr} and GIANT~\citep{wang2018giant}.

\begin{figure}
\begin{center}
\mbox{
\subfigure[Quadratic loss: communication complexity]{
\includegraphics[width=3.1in]{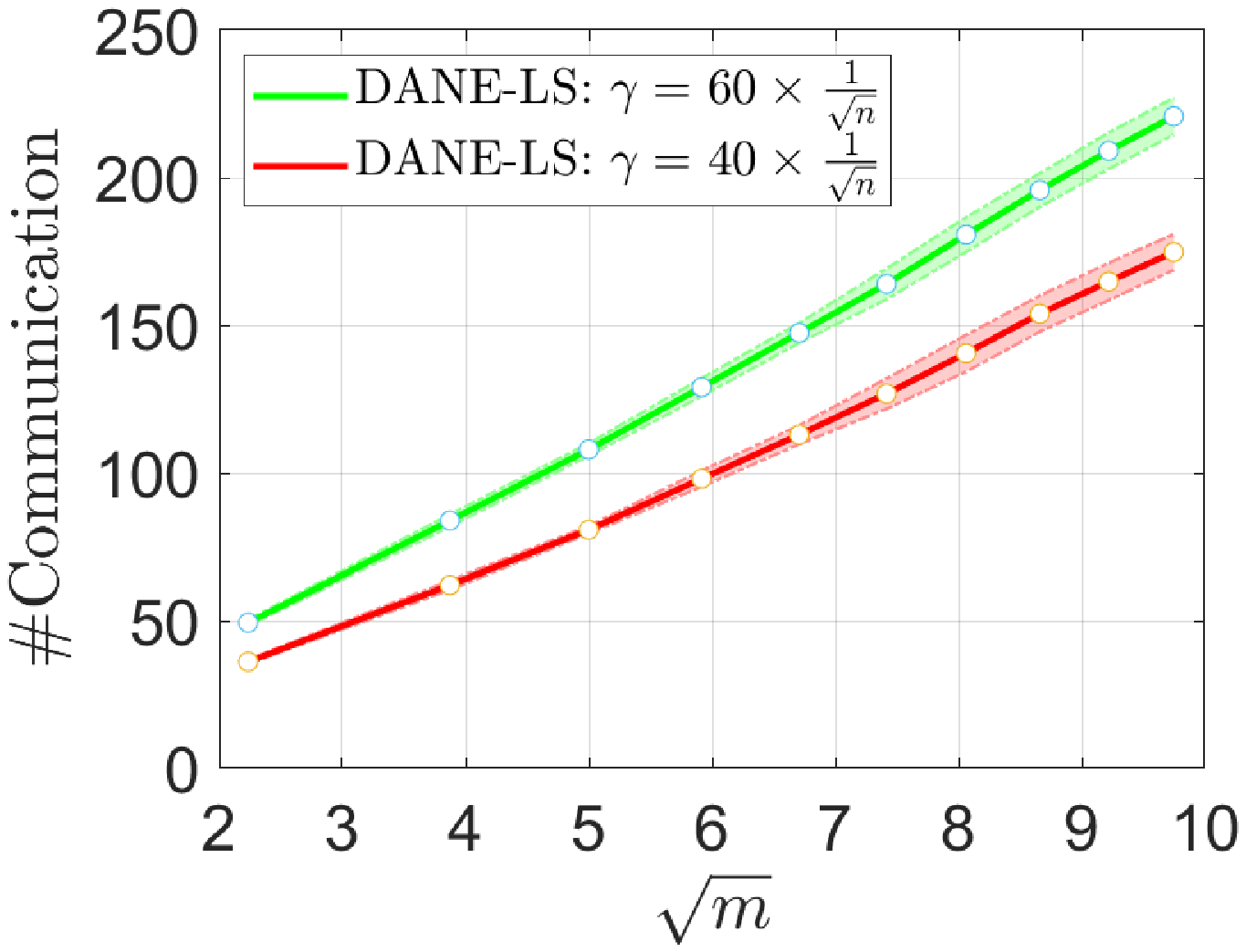}\label{fig:example_tightness_dane}}
\subfigure[Logistic loss: global convergence]{\includegraphics[width=3.1in]{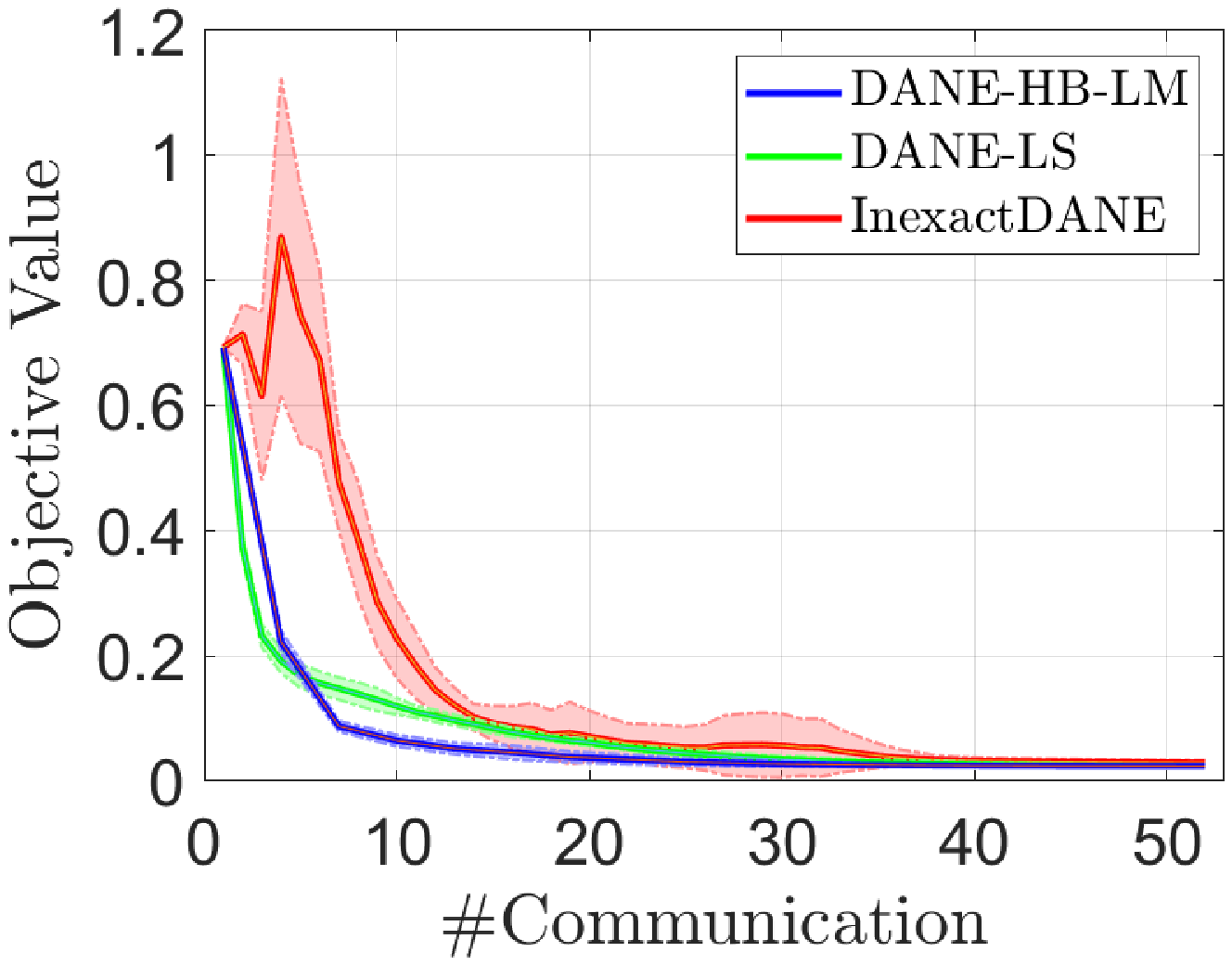}\label{fig:example_global}}}
\end{center}
\caption{(a) The number of communication rounds (y-axis) versus number of machines (x-axis) curves of DANE on a synthetic ridge regression task ($N=2000$, $p=200$). Here we set $\mu=\mathcal{O}(1/\sqrt{mn})$, $\gamma = \mathcal{O}(1/\sqrt{n})$ and precision $\epsilon=10^{-5}$. Roughly speaking, the communication complexity scales linearly with respect to $\sqrt{m}$. (b) Illustration of the global convergence behavior of DANE-LS and \InexactDane on in a synthetic logistic regression task ($N=1000, p=200, m=4$) with $\gamma = \mathcal{O}(1/\sqrt{n})$. Each experiment is randomly replicated $10$ times.}
\end{figure}

\textbf{Open issues and motivation.} Despite the above-mentioned advantages of DANE and its variants, this family of algorithms still exhibits several issues regarding convergence properties that are left open to explore, which are raised below.
\begin{itemize}
  \item \textbf{Question 1.} \textit{Is the convergence bound of plain DANE tight even for quadratic problems?} The communication complexity of plain (exact or inexact) DANE is known to be $\mathcal{\tilde O}\left(\kappa^2/n\log\left(mp/\delta\right)\log \left(1/\epsilon\right)\right)$ for stochastic quadratic problems~\citep{reddi2016aide,shamir2014communication}. Since for outer-loop communication DANE only needs to access a first-order oracle of the global problem, we have strong reason to conjecture that the factor on condition number matching this mechanism should be as sharp as $\kappa/\sqrt{n}$, even without any momentum acceleration. As visualized in Figure~\ref{fig:example_tightness_dane} for a ridge regression example with $\kappa=\mathcal{O}(\sqrt{mn})$, it is roughly the case that the number of communication rounds scales linearly with respect to $\sqrt{m}$. This leaves a potential theoretical gap between $m$ and $\sqrt{m}$ for closing.
  \item \textbf{Question 2.} \textit{Can the strong guarantees of DANE be extended to non-quadratic problems?} The strong communication complexity bounds of DANE-type methods, with or without acceleration, are so far only rigorous for quadratic problems~\citep{shamir2014communication,reddi2016aide,wang2017memory}. For more general convex/non-convex objectives, the related bounds are no stronger than those of the classic first-order methods and thus are less informative. Therefore, a natural question to ask is whether the desirable strong guarantees of DANE can be generalized to a wider problem spectrum beyond ridge regression. In addition, it is not even clear if DANE-type methods converge asymptotically under relatively small $\gamma\ll L$. In Figure~\ref{fig:example_global}, we plot the convergence curves of \InexactDane under $\gamma=\mathcal{O}(L/\sqrt{n})$ on a synthetic logistic regression task, from which we can observe that apparent zigzag effect occurs in the early stage of communication.
\end{itemize}

The primary goal of this work is to answer Question 1 and Question 2 affirmatively so as to gain deeper understanding of the convergence behavior of DANE in theory and practice.

\subsection{Overview of our contribution}

We address the above questions regarding the convergence of DANE and make progress towards fully understanding DANE both for quadratic and non-quadratic convex functions. To achieve this goal, we propose two new alternatives which are more suitable for convergence analysis as well as for algorithm acceleration. We first propose the \emph{DANE-LS} algorithm as a slight modification of DANE equipped with backtracking line search. The motivation of introducing the line search step is to ensure global asymptotic convergence and facilitate local non-asymptotic analysis for non-quadratic convex problems, which is key to answering Question 2. As another notable difference, DANE-LS only requires the master machine (say $F_1$) to solve the local subproblem to obtain the next iterate, while the worker machines (say $F_j, j=2,...,m$) wait. This turns out to lead to improved convergence rate for quadratic objective, which answers Question 1.

We then show that DANE can be readily accelerated via applying the heavy-ball acceleration technique~\citep{polyak1964some,qian1999momentum}. To this end, we modify the iteration of DANE by adding a small momentum term $\beta(w^{(t-1)}-w^{(t-2)})$ for some $\beta>0$ to the current iterate $w^{(t)}$. We call this alternative as \emph{DANE-HB}. For quadratic problems, we prove that such a simple momentum strategy boosts the communication complexity of DANE to match those of AIDE and MP-DANE but with more elementary analysis. As a perhaps more interesting contribution, DANE-HB can also be shown to exhibit the same sharp bound for strongly convex and twice differentiable objectives in a vicinity of the minimizer, which has not been covered by the previous analysis. Particularly, for the special case of learning with linear models, we further develop a variant of DANE-HB, namely \emph{DANE-HB-LM}, for which we can show that the sharp convergence bound holds globally.\\

\textbf{Highlight of results:} Table~\ref{tab:result_comparison} summarizes our main results on communication complexity of DANE-LS and DANE-HB and compares them against prior DANE-type methods. These results are divided into two groups respectively for quadratic (in stochastic setting) and non-quadratic (in deterministic setting) strongly convex problems. In stochastic setting, the big o notation $\mathcal{\tilde O}$ is used to hide the logarithm factors involving quantities other than $\epsilon$, $m$, $p$ and $\delta$, while in deterministic setting $\mathcal{O}$ is used to hide the logarithm factors involving quantities other than $\epsilon$. As highlighted in the colored cells of Table~\ref{tab:result_comparison}, we contribute several new theoretical insights into DANE, which are elaborated in details below.

\begin{table}
\centering
\begin{tabular}{|c|c|c|c|}
\hline
& Method &  Quadratic Problem & Non-quadratic Problem\\
\hline
\multirow{3}{*}{\tabincell{c}{Without \\ momentum \\ acceleration}} & DANE & $\mathcal{\tilde O}\left(\frac{\kappa^2}{n}\log\left(\frac{mp}{\delta}\right)\log \left(\frac{1}{\epsilon}\right)\right)$ & $\mathcal{O}\left(\kappa\log \left(\frac{1}{\epsilon}\right)\right)$  \\
&\InexactDane  & $\mathcal{\tilde O}\left(\frac{\kappa^2}{n}\log\left(\frac{mp}{\delta}\right)\log \left(\frac{1}{\epsilon}\right)\right)$ & $\mathcal{O}\left(\kappa\log \left(\frac{1}{\epsilon}\right)\right)$ \\
& DANE-LS (\textbf{ours}) & \cellcolor{red!30}$\mathcal{\tilde O}\left(\frac{\kappa}{\sqrt{n}}\log\left(\frac{p}{\delta}\right)\log \left(\frac{1}{\epsilon}\right)\right)$ & \cellcolor{lightskyblue} \tabincell{c}{Globally convergent with \\ local rate $\mathcal{O}\left(\frac{\gamma}{\mu}\log \left(\frac{1}{\epsilon}\right)\right)$} \\
\hline
\hline
\multirow{3}{*}{\tabincell{c}{With \\ momentum \\ acceleration}} & AIDE & $\mathcal{\tilde O}\left(\frac{\sqrt{\kappa}}{n^{1/4}}\log\left(\frac{mp}{\delta}\right)\log \left(\frac{1}{\epsilon}\right)\right)$ & $\mathcal{O}\left(\sqrt{\kappa}\log \left(\frac{1}{\epsilon}\right)\right)$  \\
&MP-DANE & $\mathcal{\tilde O}\left(\frac{\sqrt{\kappa}}{n^{1/4}}\log\left(\frac{mp}{\delta}\right)\log \left(\frac{1}{\epsilon}\right)\right)$  & \xmark  \\
&DANE-HB (\textbf{ours}) & $\mathcal{\tilde O}\left(\frac{\sqrt{\kappa}}{n^{1/4}}\log\left(\frac{p}{\delta}\right)\log \left(\frac{1}{\epsilon}\right)\right)$ & \cellcolor{brown!30} Local rate: $\mathcal {O}\left(\sqrt{\frac{\gamma}{\mu}}\log \left(\frac{1}{\epsilon}\right)\right)$ \\
\hline
\end{tabular}
\caption{Comparison of communication complexity bounds of different DANE-type methods without (top panel) or with (bottom panel) momentum acceleration. The x-mark ``\xmark'' indicates that the related result was not reported in the corresponding reference. For stochastic quadratic problems, the bounds hold in high probability with tail bound $\delta$. Best viewed in color. \label{tab:result_comparison}}
\end{table}

\begin{itemize}
\item The bound highlighted in \emph{light red} shade gives a positive answer to Question 1. That is, in the quadratic case, DANE-LS attains a tighter communication complexity bound $\mathcal{\tilde O}\left(\kappa/\sqrt{n}\log \left(p/\delta\right)\log \left(1/\epsilon\right)\right)$ than the already known $\mathcal{\tilde O}\left(\kappa^2/n\log \left(mp/\delta\right)\log \left(1/\epsilon\right)\right)$ bound for DANE. Such an improvement is achieved with only a minimal modification of algorithm (note that the line search option of DANE-LS is not activated for quadratic problems). This implies that even without any momentum acceleration, DANE actually can converge faster than already recognized in theory.

\item The result highlighted in \emph{light blue} shade answers Question 2 affirmatively. More specifically, blessed by the backtracking line search, DANE-LS with arbitrary values of $\gamma>0$ can be proved to converge globally to the unique minimizer when the objective function is strongly convex and twice differentiable. In Figure~{\ref{fig:example_global}} we illustrate the global convergence of DANE-LS when applied to a synthetic logistic regression task. The benefit of line search to DANE-type methods has also been numerically observed in~\citep{wang2018giant}, but without any theoretical justification being provided. In the late stage of iteration when the iterate is sufficiently close to the minimizer, provided that $\sup_{w} \|\nabla^2 F_1(w) - \nabla^2 F(w)\| \le \gamma$, the complexity of DANE-LS can be upper bounded by $\mathcal{O}\left(\gamma/\mu\log \left(1/\epsilon\right)\right)$ which matches the one for stochastic quadratic problems when $\gamma=\mathcal{O}(L/\sqrt{n})$.

\item From the third column of Table~\ref{tab:result_comparison} we can see that DANE-HB matches AIDE and MP-DANE in communication complexity for quadratic objective. For non-quadratic strongly convex functions, the bounds highlighted in \emph{light brown} shade shows that DANE-HB still possesses the nearly tight $\mathcal{O}\left(\sqrt{\gamma/\mu}\log \left(1/\epsilon\right)\right)$ communication complexity bound in a local area around the minimizer, hence answers Question 2 when algorithm acceleration is considered. Specially, for linear prediction models we can show that the bound actually holds globally for DANE-HB-LM as a modified version of DANE-HB. See Figure~{\ref{fig:example_global}} for an illustration of the global convergence of DANE-HB-LM and Table~\ref{tab:result_comparison_non_dane} for its theoretical properties. In contrast, the bound is $\mathcal{O}\left(\sqrt{\kappa}\log \left(1/\epsilon\right)\right)$ (which is global) for AIDE, while for MP-DANE the bound is not available.

\end{itemize}

\subsection{Other related work}

Driven by the ever-increasing demand on scaling up machine learning models in modern distributed computing environment, a vast body of distributed optimization algorithms has been developed in literature. A substantial number of these works, including the DANE-type algorithms we work on in this paper, focus on communication-efficient distributed learning which is preferable when the network has severely limited bandwidth and high latency~\citep{jaggi2014communication,jordan2018communication,richtarik2016distributed,lee2017distributed,wang2018giant}. For a special class of self-concordant empirical risk functions, \citep{zhang2015disco} proposed DiSCO as a distributed inexact damped Newton method attaining the nearly tight communication complexity bound $\mathcal{\tilde O}\left(\frac{\sqrt{\kappa}}{n^{1/4}}\log \left(\frac{mp}{\delta}\right) \log \left(\frac{1}{\epsilon}\right)\right)$  which was soon after matched by AIDE for quadratic problems. For large-scale convex linear models, CoCoA~\citep{jaggi2014communication} and CoCoA+~\citep{Ma-ICML2015,smith2018cocoa} were developed inside the framework of block coordinate descent/ascent to perform expensive local computations with the aim of reducing the overall communications across the network. In the same setting, DSCOVR~\citep{xiao2019dscovr} was proposed as a family of randomized primal-dual block coordinate algorithms for asynchronous distributed optimization with roughly $\mathcal{O}\left(m\log(\frac{1}{\epsilon})\right)$ communication complexity. With additional memory and preprocessing at each machine, \citet{lee2017distributed} showed that SVRG~\citep{johnson2013accelerating} can be adapted for distributed optimization to attain $\mathcal{O}(1)$ communication complexity, and nearly linear speedup in first-order oracle computation complexity can be achieved in the regime where sample size dominates condition number. Specifically for linear models, a more efficient implementation of distributed SVRG method was proposed and analyzed by~\citet{shamir2016without} under the without replacement sampling strategy. By combining DSVRG with minibatch passive-aggressive updates, the MP-DSVRG method was shown to have provable better tradeoff in communication-memory efficiency for quadratic loss function~\citep{wang2017memory}. The equivalence between a distributed implementation of SVRG and \InexactDane has been revealed in the framework of Federated SVRG~\citep{konevcny2016federated} for distributed machine learning with extremely large number of nodes. Recently, the GIANT method~\citep{wang2018giant} improves over DANE for linear prediction models under the assumption that sample size should be sufficiently larger than feature dimensionality. For sparse statistical estimation, EDSL~\citep{wang2017efficient} and DINPS~\citep{liu2019distributed} respectively extend DANE to solving $\ell_1$-regularized and $\ell_0$-constrained ERM problems, obtaining analogous improvement in communication bounds. Last but not least, the well designed distributed learning platforms such as MapReduce~\citep{DeanMapReduce2008}, Apache Spark~\citep{ZahariaAparchSpark2016}, Petuum~\citep{xing2015petuum} and Parameter Server~\citep{li2014communication} have significantly facilitated the system implementation of these algorithms.

\subsection{Organization and notation}

\noindent \textbf{Paper organization.} The rest of this paper is organized as follows: In \S\ref{sect:dane_ls_analysis} we introduce DANE-LS as a new alternative of DANE with backtracking line search and analyze its convergence rate for quadratic and non-quadratic convex functions. In~\S\ref{sect:dane_hb_analysis}, we propose DANE-HB to accelerate DANE using heavy ball approach, along with a variant specifically designed for linear prediction models. The numerical evaluation results are presented in \S\ref{sect:experiment}. Finally, we conclude this paper in \S\ref{sect:conclusion}. All the technical proofs of results are deferred to the appendix section.

\textbf{Notation.} The key quantities and notations that commonly used in our analysis are summarized in Table~\ref{tab:notation}. In deterministic setting, we use the big o notation $\mathcal{O}$ that hides the logarithm factors involving quantities other than $\epsilon$, while in stochastic setting, $\mathcal{\tilde O}$ is used with the logarithm factors involving quantities other than $\epsilon$, $m$, $p$ and $\delta$ hidden inside.

\begin{table}[h!]
\centering
\begin{tabular}{|c|c|}
\hline
Notation & Definition  \\
\hline
$m$ & number of worker machines \\
$n$ & number of training samples distributed on each individual worker machine \\
$N=mn$ & total number of training samples \\
$p$ & number of features \\
$F(w)$ & the global risk function \\
$F_1(w)$ & the local risk function on the master machine \\
$L$ & Lipschitz smoothness parameter of the gradient vector $\nabla F(w)$ \\
$\nu$ & Lipschitz smoothness parameter of the Hessian matrix $\nabla^2 F(w)$ \\
$\mu$ & The strong convexity parameter of $F(w)$ \\
$\kappa=L/\mu$ & the condition number of $F(w)$ \\
$\beta$ & momentum strength coefficient for heavy-ball acceleration \\
$\epsilon$ & sub-optimality of the global problem \\
$\varepsilon$ & sub-optimality of the local subproblem \\
$\gamma$ & the regularization strength parameter of the local subproblem \\
$\delta$ & the failure probability bound in stochastic setting \\
$[N]$ & the abbreviation of the index set $\{1,...,N\}$ \\
\hline
$\|x\|=\sqrt{x^\top x}$ & the Euclidean norm of a vector $x$ \\
$\lambda_{\max}(A)$ & the largest eigenvalue of a matrix $A$ \\
$\lambda_{\min}(A)$ & the smallest eigenvalue of a matrix $A$ \\
$A \succeq B$ & $A-B$ is symmetric, positive semi-definite \\
$A \succ B$ & $A-B$ is symmetric, positive definite \\
$\|A\|$ & the spectral norm of matrix $A$ \\
$\rho(A)$ & the spectral radius of $A$, i.e., its largest (in magnitude) eigenvalue\\
\hline
\end{tabular}
\caption{Table of notation. \label{tab:notation}}
\end{table}

\section{Globalization of DANE with Sharper Analysis}
\label{sect:dane_ls_analysis}

In this section, we provide a global and sharper analysis of the plain version of DANE method without applying any momentum acceleration. The analysis is actually conducted on a  modified version of DANE augmented with backtracking line search, while only a master machine is allocated to do local computation in an inexact manner. Such simple modifications turn out to be beneficial for the global asymptotic and local non-asymptotic analysis of DANE.

\begin{algorithm}[h]
\caption{DANE with backtracking Line Search: DANE-LS($\gamma, \rho, \nu$)}
\label{alg:dane_ls}
\SetKwInOut{Input}{Input}\SetKwInOut{Output}{Output}\SetKw{Initialization}{Initialization}
\Input{Parameters $\gamma, \nu>0$, $\rho\in(0,1/3)$.}
\Output{$w^{(t)}$.}

\Initialization{Set $w^{(0)}=0$ or $w^{(0)} \approx \argmin_{w} F_1(w)$.}

\For{$t=1, 2, ...$}{

\tcc{\textbf{Global computation on master machine associated with $F_1(w)$}}

Compute $\nabla F(w^{(t-1)})=\frac{1}{m}\sum_{j=1}^m \nabla F_j(w^{(t-1)})$;

Estimate $\tilde w^{(t)}$ such that $\|\nabla P^{(t-1)}(\tilde w^{(t)} )\|\le \varepsilon_t$, where
\begin{equation}\label{equat:P_t_w}
P^{(t-1)}(w):=\langle \nabla F(w^{(t-1)}) - \nabla F_1(w^{(t-1)}), w \rangle + \frac{\gamma}{2}\|w-w^{(t-1)}\|^2 + F_1(w);
\end{equation}

\If{The objective function $F$ is not quadratic}
{
\tcc{\textbf{Backtracking line search for non-quadratic objectives}}
\colorbox{lightskyblue}{\begin{minipage}{0.89\linewidth}
Update $w^{(t)} =(1-\eta_t)w^{(t-1)} + \eta_t \tilde w^{(t)}$ with proper $\eta_t\in(0, 1]$ which satisfies either of the following \emph{sufficient descent} condition for the provided $\rho$:

(\textbf{Option-I}) \tcc{Line-search with global value evaluation.}
\begin{equation}\label{eqn:global_line_search}
F(w^{(t)})\le F(w^{(t-1)}) - \psi(\tilde w^{(t)}, w^{(t-1)}),
\end{equation}
where
\[
\begin{aligned}
\psi(\tilde w^{(t)}, w^{(t-1)}):=& \eta_t \rho \langle \nabla F_1(\tilde w^{(t)}) -  \nabla F_1(w^{(t-1)}) + \gamma (\tilde w^{(t)} - w^{(t-1)}), \tilde w^{(t)} - w^{(t-1)}\rangle \\
&- \eta_t \varepsilon_t\|\tilde w^{(t)} - w^{(t-1)}\|;
\end{aligned}
\]

(\textbf{Option-II}) \tcc{Line-search without global value evaluation.}
\begin{equation}\label{eqn:local_line_search}
\begin{aligned}
& \langle\nabla F(w^{(t-1)}), w^{(t)} - w^{(t-1)} \rangle + (w^{(t)} - w^{(t-1)})^\top (\nabla^2 F_1(w^{(t-1)})+\gamma I) (w^{(t)} - w^{(t-1)}) \\
&+ \frac{\nu}{6} \|w^{(t)} - w^{(t-1)}\|^3\le - \psi(\tilde w^{(t)}, w^{(t-1)}).
\end{aligned}
\end{equation}

\end{minipage}
}
}
\Else{
$w^{(t)} = \tilde w^{(t)}$;
}
\tcc{\textbf{Local gradient evaluation on worker machines}}
\noindent
For each machine $j$, compute $\nabla F_j(w^{(t)})$ and broadcast to the master machine;
}
\end{algorithm}

\subsection{Leveraging backtracking line search}

Since DANE is essentially an approximated second-order method, it is a natural idea to leverage an additional line search operation to hopefully guarantee global convergence while preserving the appealing local non-asymptotic convergence rate. In practice, the numerical evidence in~\citep{wang2018giant} has already demonstrated, although without any theoretical support, that backtracking line search does help to improve the convergence performance of DANE-type methods. Inspired by these points, we propose the DANE-LS (DANE with Line Search) method which is outlined in Algorithm~\ref{alg:dane_ls}.  The notable differences between DANE-LS and DANE/\InexactDane at each iterate round are summarized in below:
\begin{itemize}
  \item For non-quadratic problems, two optional backtracking line search steps (as highlighted in light blue shade) are conducted on the master machine. The Option-I needs to evaluate the global objective value and hence requires additional communication cost. By only accessing the locally available information, the Option-II is free of evaluating the global objective value but at the price of introducing an additional hyper-parameter $\nu$ representing the smoothness of Hessian.
  \item As another notable difference, only a master machine is in charge of solving a local subproblem associated with $F_1(w)$ to obtain the next iterate, during which time the other worker machines stay idle. Such a master-slave architecture has been widely adopted and investigated in many distributed machine learning and statistical inference approaches~\citep{jordan2018communication,lee2017distributed,shamir2016without,wang2017efficient}. Allowing only master to do the heavy lifting is certainly more energy efficient and less sensitive to network latency.
\end{itemize}


As the consequence of these modifications, DANE-LS can be shown to improve over DANE not only for non-quadratic convex objectives (see Section~\ref{ssect:dnae_ls_global}) but also for the well studied quadratic case (see Section~\ref{ssect:dane_ls_quadratic}). Moreover, the master-slave computing architecture eases the generalization of analysis to the heavy-ball acceleration presented in the next section. It is noteworthy that the local subproblem is allowed to be solved inexactly with sub-optimality $\|\nabla P^{(t-1)}(\tilde w^{(t)} )\|\le \varepsilon_t$. Such a local sub-optimality condition is computationally more tractable for verification than those of \InexactDane and AIDE with unknown local minimizers involved, and hence is more practical from the perspective of algorithm implementation.

\subsection{Sharper bounds for quadratic function}
\label{ssect:dane_ls_quadratic}

We start by analyzing DANE-LS in a simple yet informative regime where the loss functions are quadratic. In this setting, the line search options will not be activated throughout algorithm execution.

\textbf{Preliminary.} Our analysis relies on the conditions of strong convexity and Lipschitz smoothness which are conventionally used in the previous analysis of distributed optimization methods.
\begin{definition}[Strong Convexity/Smoothness]\label{def:strong_smooth}
A differentiable function $g$ is $\mu$-strongly-convex and $L$-smooth if $\forall w, w'$,
\[
\frac{\mu}{2}\|w - w'\|^2 \le g(w) - g(w') - \langle \nabla g(w'), w - w'\rangle \le \frac{L}{2}\|w - w'\|^2.
\]
\end{definition}
The ratio value $\kappa = L/\mu$ is the condition number. We further introduce the concept of Lipschitz continuous Hessian which characterizes the continuity of the Hessian matrix.
\begin{definition}[Lipschitz Hessian]\label{def:RLH}
We say a twice continuously differentiable function $g$ has Lipschitz continuous Hessian with constant $\nu\ge 0$ ($\nu$-LH) if $\forall w, w'$,
\[
\left\|\nabla^2 g(w) - \nabla^2 g(w')\right\| \le \nu \|w - w'\|.
\]
\end{definition}
Let $w^* = \argmin_w F(w)$. The following is our main result on the convergence rate of DANE-LS for quadratic functions in terms of parameter estimation error.
\begin{restatable}[Convergence rate of DANE-LS for quadratic loss]{theorem}{DANELSQuadratic}\label{thrm:quadratic_dane}
Assume that the loss function is quadratic. Let $H$ and $H_1$ be the Hessian matrices of the global objective $F$ and local objective $F_1$, respectively. Assume that $\mu I \preceq H \preceq L I$. Given precision $\epsilon>0$, if $\|H_1 - H\| \le \gamma$ and $\varepsilon_t \le \frac{\mu^2\|\nabla F(w^{(t-1)})\|}{2(\mu +2\gamma)L}$, then Algorithm~\ref{alg:dane_ls} will output $w^{(t)}$ satisfying $\|w^{(t)} - w^*\| \le \epsilon$ after
\[
t \ge \frac{2(\mu+2\gamma)}{\mu} \log \left(\frac{\sqrt{\kappa}\|w^{(0)} - w^*\|}{\epsilon}\right)
\]
rounds of iteration.
\end{restatable}


As a comparison, the communication complexity bounds established for DANE~\citep[Lemma 1]{shamir2014communication} and \InexactDane~\citep[Corollary 1]{reddi2016aide} are both of the order $\mathcal{O}\left(\gamma^2/\mu^2\log \left(1/\epsilon\right)\right)$, which are clearly inferior to the $\mathcal{O}\left(\gamma/\mu\log \left(1/\epsilon\right)\right)$ bound obtained in Theorem~\ref{thrm:quadratic_dane}. After a careful inspection of the technical proofs in~\citep{reddi2016aide,shamir2014communication}, we note that the looseness of the former bounds essentially results from the reduce operation conducted by master machine for aggregating models from local workers, and such an issue is seemingly difficult to be remedied inside the original architecture of DANE. After applying the modifications as mentioned in the previous subsection, the tighter bound in Theorem~\ref{thrm:quadratic_dane} can be attained using a fairly elementary analysis. This answers Question 1 affirmatively.

To more clearly illustrate the improvement, we derive the following result which is an implication of Theorem~\ref{thrm:quadratic_dane} to the stochastic setting where the samples are uniformly randomly distributed over machines.

\begin{restatable}{corollary}{DANELSQuadraticCorol}\label{corol:quadratic_dane}
Assume the conditions in Theorem~\ref{thrm:quadratic_dane} hold and $\|\nabla^2 f(w; x_{i}, y_{i})\|\le L$ for all $i\in[N]$. Then for any $\delta\in(0,1)$,  with probability at least $1-\delta$  over the samples drawn to construct $F_1$, Algorithm~\ref{alg:dane_ls} with $\gamma=L\sqrt{\frac{32\log(p/\delta)}{n}}$ will output $w^{(t)}$ satisfying $\|w^{(t)} - w^*\| \le \epsilon$ after
\[
t \ge \left(1+ 2\kappa\sqrt{\frac{32\log(p/\delta)}{n}}\right) \log \left(\frac{2\sqrt{\kappa}\|w^{(0)} - w^*\|}{\epsilon}\right)
\]
rounds of iteration.
\end{restatable}
\begin{remark}
In statistical learning problems, the condition number $\kappa$ could scale as large as $\mathcal{O}(\sqrt{mn})$~\citep{shalev2009stochastic}. If this is the case, then Corollary~\ref{corol:quadratic_dane} implies an $\mathcal{\tilde O}\left(\sqrt{m}\log\left(p/\delta\right)\log \left(1/\epsilon\right)\right)$ communication complexity bound for stochastic quadratic problems, which contrasts itself from the $\mathcal{\tilde O}\left(m\log\left(mp/\delta\right)\log \left(1/\epsilon\right)\right)$ bound previously known for DANE and \InexactDane as well. Notice, such improvement is of particular interest in the regime of federated learning where the number of computing nodes $m$ could be huge~\citep{konevcny2016federated,mcmahan2017communication}.
\end{remark}

\subsection{Global analysis for strongly convex functions}
\label{ssect:dnae_ls_global}

We then move to consider the more general regime in which the objective function is strongly convex and twice differentiable with Lipschitz continuous Hessian. First, we show in the following lemma that the proposed global and local backtracking line search steps are always feasible under proper conditions.
\begin{restatable}[Feasibility of line search]{lemma}{DANELSFeasibility}\label{lemma:ls_key}
Assume that $F$ is $L$-smooth and $F_1$ is $\mu$-strongly convex. For any given $\rho \in (0,1)$,
\begin{itemize}
  \item[(a)] if
\[
0 < \eta_t \le \min \left\{1,\frac{2(\gamma+\mu)(1-\rho)}{L}\right\},
\]
then the global backtracking line search (Option-I) is feasible, i.e.,
\[
F(w^{(t)}) \le F(w^{(t-1)}) - \psi(\tilde w^{(t)}, w^{(t-1)}),
\]
where $\psi(\tilde w^{(t)}, w^{(t-1)}):= \eta_t \rho \langle \nabla F_1(\tilde w^{(t)}) -  \nabla F_1(w^{(t-1)}) + \gamma (\tilde w^{(t)} - w^{(t-1)}), \tilde w^{(t)} - w^{(t-1)}\rangle - \eta_t \varepsilon_t\|\tilde w^{(t)} - w^{(t-1)}\|$.
  \item[(b)] Moreover, assume that $F_1(w)$ has $\nu$-LH and $\exists D>0$ such that $\|\tilde w^{(t)} - w^{(t-1)}\|\le D$ for all $t\ge0$. If
\[
\eta_t \le \min \left\{1,\frac{-(3\nu D + 6(\gamma+\mu)) + \sqrt{(3\nu D + 6(\gamma+\mu)) ^2 + 96(1-\rho)\nu D (\gamma+\mu)}}{4\nu D}\right\},
\]
then the local backtracking line search (Option-II) is feasible, i.e.,
\[
\begin{aligned}
&\langle \nabla F(w^{(t-1)}) , w^{(t)} - w^{(t-1)}\rangle + \frac{1}{2}(w^{(t)} - w^{(t-1)})^\top(\nabla^2 F_1(w^{(t-1)})+\gamma I)(w^{(t)} - w^{(t-1)})\\
&+ \frac{\nu }{6}\|w^{(t)} - w^{(t-1)}\|^3\le - \psi(\tilde w^{(t)}, w^{(t-1)}).
\end{aligned}
\]
\end{itemize}
\end{restatable}
\begin{remark}
The bound $D$ in the part (b) of Lemma~\ref{lemma:ls_key} is reasonable if we focus on an $\ell_2$-norm bounded domain of interest $\Omega$ such that $D=\max_{w,w'\in \Omega}\|w - w'\|$. The result also implies that if the global line search of Option-I is used under Armijo rule, then the additional rounds of communication for global objective evaluation is roughly of the order $\mathcal{O}\left(\log\left(\frac{L}{(\gamma+\mu)(1-\rho)}\right)\right)$.
\end{remark}

The following theorem is our main result on the global convergence of DANE-LS.
\begin{restatable}[Global convergence of DANE-LS]{theorem}{DANELSGlobal}\label{thrm:global_general}
Assume that $F(w)$ and $F_1(w)$ are $L$-smooth, $\mu$-strongly-convex and have $\nu$-LH. Suppose that $\varepsilon_t\le \frac{\rho(\mu+\gamma)}{2(L+\gamma)+\rho(\mu+\gamma)}\|\nabla F(w^{(t-1)})\|$.
\begin{itemize}
  \item[(a)] Then the objective value sequence $\{F(w^{(t)})\}$ generated by Algorithm~\ref{alg:dane_ls} with the global line search step (Option-I) converges and the difference norm sequence $\{\|\tilde w^{(t)} - w^{(t-1)}\|\}$ converges to zero.
  \item[(b)] Assume in addition that $\sup_{w} \|\nabla^2 F_1(w) - \nabla^2 F(w)\| \le \gamma$ and $\|\tilde w^{(t)} - w^{(t-1)}\|$ is bounded from above for all $t\ge0$. Then the objective value sequence $\{F(w^{(t)})\}$ generated by Algorithm~\ref{alg:dane_ls} with local line search step (Option-II) converges and the difference norm sequence $\{\|\tilde w^{(t)} - w^{(t-1)}\|\}$ converges to zero.
\end{itemize}
\end{restatable}
\begin{remark}
Theorem~\ref{thrm:global_general} suggests a natural way of controlling the termination of Algorithm~\ref{alg:dane_ls} by monitoring either the difference of adjacent objective values $F(w^{(t)}) - F(w^{(t-1)})$ or the norm of vector difference $\|\tilde w^{(t)} - w^{(t-1)}\|$.
\end{remark}

\textbf{Local non-asymptotic convergence.} We further study the local convergence behavior of DANE-LS. The starting point is to show, via the following lemma, that the unit length eventually satisfies the sufficient descent condition in~\eqref{eqn:global_line_search}.
\begin{restatable}[Acceptability of unit length for line search]{lemma}{DANELSUnit}\label{lemma:unit_length}
Assume that the conditions in Theorem~\ref{thrm:global_general} hold. Then for sufficiently large $t$, the unit length satisfies the sufficient descent condition ~\eqref{eqn:global_line_search} with $\rho \in(0, 1/3)$.
\end{restatable}
The following lemma establishes the local convergence rate of Algorithm~\ref{alg:dane_ls} when $\eta_t\equiv1$, i.e., when the unit length is always accepted by the backtracking line search.

\begin{restatable}[Local convergence rate of DANE-LS]{lemma}{DANELSLocalrate}\label{lemma:local_unit}
Assume that $F$ and $F_1$ are $\mu$-strongly-convex, $L$-smooth and have $\nu$-LH. Assume that $\sup_w \|\nabla^2 F_1(w) - \nabla^2 F(w)\| \le \gamma$. Let $\tau=  \left\lceil \frac{\mu+2\gamma}{2\mu} \log\left(4\kappa\right)  \right\rceil$. Suppose that $\varepsilon_t \le \min\left\{(\gamma+\mu)^2, \frac{\|\nabla F(w^{(t-1)})\|^2}{L^2}\right\}$.  Given precision $\epsilon > 0$, if $\max_{0 \le i \le \tau-1}\|w^{(i)} - w^*\| \le \frac{(\gamma+\mu)}{4(6\nu+1)\sqrt{\kappa}\tau }$, then Algorithm~\ref{alg:dane_ls} with $\eta_t\equiv1$ will attain estimation error $\|w^{(t)} - w^*\| \le \epsilon$ after
\[
t \ge 4\tau \log \left(\frac{\gamma+\mu}{4(6\nu+1)\sqrt{\kappa}\tau\nu\epsilon}\right)
\]
rounds of iteration.
\end{restatable}
\begin{remark}
Lemma~\ref{lemma:local_unit} essentially shows that up to the logarithmic factors on $\kappa$ and $\tau$, the local communication complexity of DANE-LS is bounded as $\mathcal{O}(\gamma/\mu\log \left(1/\epsilon\right))$, which exactly matches the bound for the quadratic function.
\end{remark}

We are now ready to present our main result on the local non-asymptotic convergence of DANE-LS for strongly convex functions.
\begin{restatable}[Non-asymptotic convergence of DANE-LS]{theorem}{DANELSNonasymptotic}\label{thrm:local_general}
Assume that $F$ and $F_1$ are $\mu$-strongly-convex, $L$-smooth and have $\nu$-LH. Assume that $\sup_{w} \|\nabla^2 F_1(w) - \nabla^2 F(w)\| \le \gamma$. Suppose that $\rho \in (0,1/3)$ and
\[
\varepsilon_t \le \min\left\{(\gamma+\mu)^2, \frac{\|\nabla F(w^{(t-1)})\|^2}{L^2},\frac{\rho(\mu+\gamma)}{2(L+\gamma)+\rho(\mu+\gamma)}\|\nabla F(w^{(t-1)})\|\right\}.
\]
Then there exists a time stamp $t_0$ not relying on $\epsilon$ such that Algorithm~\ref{alg:dane_ls} will output solution $w^{(t)}$ satisfying $\|w^{(t)} - w^*\| \le \epsilon$ after
\[
t \ge t_0 + \mathcal{O} \left(\frac{\gamma}{\mu}\log \left(\frac{1}{\epsilon}\right)\right)
\]
rounds of iteration.
\end{restatable}
\begin{remark}\label{rmk:gamma_tight}
Theorem~\ref{thrm:local_general} reveals that DANE-LS converges globally towards $w^*$ and in a local area around $w^*$ it enjoys a linear rate of convergence with complexity $\mathcal{O}(\gamma/\mu \log(1/\epsilon))$. We comment on the choice of $\gamma$ in the theorem. For a large family of smooth loss functions, the uniform convergence theory from~\citep{mei2018landscape} suggests that $\gamma = \mathcal{\tilde O}\left(\sqrt{p/n}\right)$, which is expected to be small when the number of local samples is sufficiently larger than feature dimension. This result actually answers Question 2 raised in Section~\ref{ssect:review}. Based on that bound, we choose to set $\gamma=\mathcal{O}(1/\sqrt{n})$ in our numerical study to take better advantage of the statistical correlation of local problems for global optimization.
\end{remark}

\section{Heavy-Ball Acceleration of DANE}
\label{sect:dane_hb_analysis}

We further introduce a simple yet effective momentum acceleration method for DANE based on the classic heavy-ball approach~\citep{polyak1964some}, which has long been acknowledged to work favorably for accelerating first-order methods~\citep{ghadimi2015global,loizou2017linearly,wilson2016lyapunov,zhou2018efficient}.

\subsection{The DANE-HB Algorithm}
As outlined in Algorithm~\ref{alg:dane_hb}, the proposed DANE-HB method shares an almost identical centralized computing architecture to DANE-LS. The key difference is that for local subproblem optimization in the master machine, we first estimate $\tilde w^{(t)} \approx \argmin_w  P^{(t-1)}(w)$, and then compute $w^{(t)} = \tilde w^{(t)} + \beta (w^{(t-1)} - w^{(t-2)})$ as a linear combination of $\tilde w^{(t)}$ and the previous two iterates, where $\beta>0$ is the momentum strength coefficient. It is optional to implement the backtracking line search steps as proposed in Algorithm~\ref{alg:dane_ls} which work well in our numerical study to obtain global convergence, although there is no theoretical guarantee that the difference vector $w^{(t)} - w^{(t-1)}$ should point to a descent direction.
Concerning initialization, the simplest way is to set $w^{(-1)} = w^{(0)} = 0$, i.e., starting the iteration from scratch. Since $F_1(w)$ tends to be close to $F(w)$ in stochastic setting, another reasonable option of initialization is to set $w^{(-1)} = w^{(0)}\approx\argmin_{w} F_1(w)$ which is also expected to be close to the global solution $w^*$.

\begin{algorithm}[tb]\caption{DANE with Heavy-Ball acceleration: DANE-HB ($\gamma, \beta$)}
\label{alg:dane_hb}
\SetKwInOut{Input}{Input}\SetKwInOut{Output}{Output}\SetKw{Initialization}{Initialization}
\Input{Parameters  $\gamma, \beta>0$.}
\Output{$w^{(t)}$.}
\Initialization{Set $w^{(0)}=0$ or $w^{(0)} \approx \argmin_{w} F_1(w)$. Let $w^{(-1)} = w^{(0)}$.}

\For{$t=1, 2, ...$}{
\tcc{\textbf{Global computation on master}}

Compute $\nabla F(w^{(t-1)})=\frac{1}{m}\sum_{j=1}^m \nabla F_j(w^{(t-1)})$;

Estimate $\tilde w^{(t)} $ such that $\|\nabla P^{(t-1)}(\tilde w^{(t)} )\| \le \varepsilon_t$, where $P^{(t-1)}$ is defined by~\eqref{equat:P_t_w};

Compute $w^{(t)} = \tilde w^{(t)} + \beta (w^{(t-1)} - w^{(t-2)})$;

(Optionally) Conduct backtracking line search.

\tcc{\textbf{Local gradient evaluation on workers}}

For each machine $j$, compute $\nabla F_j(w^{(t)})$ and broadcast to the master machine;
}
\end{algorithm}

\subsection{Convergence results for quadratic function}

The following result confirms that the heavy-ball acceleration strategy can improve the communication efficiency of DANE for quadratic problems.
\begin{restatable}[Convergence rate of DANE-HB for quadratic function]{theorem}{DANEHBQuadratic}\label{thrm:quadratic_hb}
Assume that the loss function is quadratic. Let $H$ and $H_1$ be the Hessian matrices of the global objective $F$ and local objective $F_1$, respectively. Assume that $\mu I \preceq H \preceq L I$. Set $\beta = \left(1-\sqrt{\frac{\mu}{\mu+2\gamma}}\right)^2$. Given precision $\epsilon>0$, if $\|H_1 - H\| \le \gamma$ and $\varepsilon_t \le \frac{\sqrt{2}(\mu+\gamma)\|\nabla F(w^{(0)})\|}{2L(t+1)^2}\left(1-\frac{1}{2}\sqrt{\frac{\mu}{\mu+2\gamma}}\right)^{t+1}$,  then Algorithm~\ref{alg:dane_hb} will output $w^{(t)}$ satisfying $\|w^{(t)} - w^*\| \le \epsilon$ after
\[
t \ge 2\sqrt{\frac{\mu+2\gamma}{\mu}} \log \left(\frac{2\sqrt{2}c\|w^{(0)} - w^*\|}{\epsilon}\right)
\]
rounds of iteration, where $c$ is a constant relying on $\sqrt{\mu/(\mu + 2\gamma)}$.
\end{restatable}

The following corollary is the implication of Theorem~\ref{thrm:quadratic_hb} in stochastic setting.
\begin{restatable}{corollary}{DANEHBQuadraticCorol}\label{corol:quadratic_hb}
Assume the conditions in Theorem~\ref{thrm:quadratic_hb} hold and $\|\nabla^2 f(w; x_{i}, y_{i})\|\le L$ for all $i\in[N]$. Then for any $\delta\in(0,1)$, with probability at least $1-\delta$ over the samples drawn to construct $F_1$, Algorithm~\ref{alg:dane_hb} with $\gamma=L\sqrt{\frac{32\log(p/\delta)}{n}}$ will attain estimation error $\|w^{(t)} - w^*\| \le \epsilon$ after
\[
t \ge \mathcal{\tilde O} \left(\frac{\sqrt{\kappa}}{n^{1/4}}\log^{1/4}\left(\frac{p}{\delta}\right) \log \left(\frac{1}{\epsilon}\right)\right)
\]
rounds of iteration.
\end{restatable}
\begin{remark}
The result shows that  in the quadratic case, DANE-HB is able to match the communication complexity lower bounds (up to logarithmic factors) proved in~\citep{arjevani2015communication}. Similar guarantees for quadratic function have also been proved for AIDE and MP-DANE with acceleration achieved via applying the catalyst scheme~\citep{lin2015universal}.
\end{remark}

\subsection{Convergence results for strongly convex functions}

For a broad class of strongly convex functions with Lipschitz continuous Hessian, we show in the following theorem that in a vicinity of the global minimizer, DANE-HB enjoys the same appealing rate of convergence as established for the ridge regression problems.
\begin{restatable}[Local convergence rate of DANE-HB]{theorem}{DANEHBLocalrate}\label{thrm:local_general_hb}
Assume that $F$ and $F_1$ are $\mu$-strongly-convex and has $\nu$-LH. Assume that $\sup_{w} \|\nabla^2 F_1(w) - \nabla^2 F(w)\| \le \gamma$. Choose $\beta = \left(1-\sqrt{\mu/(\mu+2\gamma)}\right)^2$. Let $\tau= \left\lceil 2\sqrt{(\mu+2\gamma)/\mu} \log(2c)  \right\rceil$ in which $c$ is a constant dependent on $\sqrt{\mu/(\mu+2\gamma)}$. Assume that $\varepsilon_t \le \min\left\{(\gamma+\mu)^2, \|\nabla F(w^{(t-1)})\|^2/L^2\right\}$. Given precision $\epsilon>0$, if $\max_{-1\le i \le \tau-1}\|w^{(i)} - w^*\| \le \frac{\gamma+\mu}{4(6\nu+1)\sqrt{2}c\tau}$, then Algorithm~\ref{alg:dane_hb} will output $w^{(t)}$ satisfying $\|w^{(t)} - w^*\| \le \epsilon$ after
\[
t \ge 4\tau \log \left(\frac{\gamma+\mu}{4(6\nu+1)c\tau } \left(\frac{1}{\epsilon}\right)\right)
\]
rounds of iteration.
\end{restatable}
We comment on the related bounds of DANE-HB and AIDE for non-quadratic convex problems. It was proved in~\citep[Theorem 6]{reddi2016aide} that AIDE converges at the rate of $\mathcal{O}\left(\sqrt{\kappa}\log\left(1/\epsilon\right)\right)$ for non-quadratic strongly convex functions with $\gamma=\mathcal{O}(L)$, and that result is global. In contrast, we obtain the $\mathcal{O} \left(\sqrt{\gamma/\mu}\log \left(1/\epsilon\right)\right)$ bound in Theorem~\ref{thrm:local_general_hb} for arbitrary $\gamma$ as long as the $\gamma$-related condition $\sup_{w} \|\nabla^2 F_1(w) - \nabla^2 F(w)\| \le \gamma$ holds, and hence is tighter when $\gamma\ll L$ (see Remark~\ref{rmk:gamma_tight}). However, this bound is only provable in a local area around the global minimizer.

\subsection{Extension for learning with linear models}
\label{sect:dane_hb_lm_analysis}

So far, DANE-HB has been shown to converge globally for the quadratic objective, whilst for non-quadratic problems it can merely be shown to converge in a vicinity of the global minimizer. In this section, we move to study a special class of learning problems with linear regression or prediction models. More specifically, we consider the loss function of the form
\[
f(w; x_i, y_i) = l(w^\top x_i, y_i) + \frac{\mu}{2}\|w\|^2,
\]
where $l(w^\top x_i;y_i)$ is a convex function that measures the linear regression/prediction loss of $w$ at data point $(x_i,y_i)$ and $\mu>0$ controls the strength of $\ell_2$-regularization. For example, the least square loss $l(w^\top x_i,y_i)=\frac{1}{2}(y_i - w^\top x_i)^2$ is used in linear regression and the logistic loss $l(w^\top x_i, y_i)=\log\left(1+\exp(-y_{i}w^\top x_{i})\right)$ in logistic binary classification. Then we can reexpress problem~\eqref{eqn:general}
\[
\min_{w\in \mathbb{R}^p} F(w)= \tilde F(w) + \frac{\mu}{2}\|w\|^2,
\]
where $\tilde F(w):=\frac{1}{N}\sum_{i=1}^N l(w^\top x_i, y_i)$.
For such a problem of learning with linear models, we will be able to show that with proper modification, DANE-HB actually converges globally at a rate similar to that of the quadratic problem.

\textbf{The DANE-HB-LM algorithm.} The method of DANE-HB-LM (DANE-HB for linear models) is formally stated in Algorithm~\ref{alg:dane_hb_lm}. The idea behind the method is fairly straightforward: at each iterate $w^{(t-1)}$, we first construct a quadratic approximation $Q^{(t-1)}(w)$ to the original problem around $w^{(t-1)}$ as in~\eqref{eqn:quadratic_approx} and then apply DANE-HB to optimize $Q^{(t-1)}(w)$ in a distribute fashion. Specially, when $\beta=1$, DANE-HB-LM reduces to a variant of plain DANE for learning with linear models.

\begin{algorithm}[tb]\caption{DANE-HB for Linear Models: DANE-HB-LM($\gamma, \beta, \ell$)}
\label{alg:dane_hb_lm}
\SetKwInOut{Input}{Input}\SetKwInOut{Output}{Output}\SetKw{Initialization}{Initialization}
\Input{Parameters $\gamma, \beta, \ell>0$. Typically $\gamma=\mathcal{O}(1/\sqrt{n})$.}
\Output{$w^{(t)}$.}
\Initialization{Set $w^{(0)}=0$.}

\For{$t=1, 2, ...$}{

(S1) Construct a quadratic approximation to $F$ at $w^{(t-1)}$ as:
\begin{equation}\label{eqn:quadratic_approx}
\begin{aligned}
&Q^{(t-1)}(w)\\
:=& \tilde F(w^{(t-1)}) + \langle \nabla \tilde F(w^{(t-1)}), w - w^{(t-1)}\rangle + \frac{1}{2} (w-w^{(t-1)})^\top H (w - w^{(t-1)}) + \frac{\mu}{2}\|w\|^2,
\end{aligned}
\end{equation}
where $H = \frac{\ell X X^\top}{N}$.

(S2) Estimate $w^{(t)}=\text{DANE-HB}(\gamma,\beta)$ by applying DANE-HB to $Q^{(t-1)}(w)$ with a warm-start initialization $w^{(t-1)}$ such that
\[
Q^{(t-1)}(w^{(t)}) \le \min_w Q^{(t-1)}(w) + \varepsilon_t.
\]
}
\end{algorithm}

 \textbf{Convergence analysis.} Let us denote $X_1$ the subset of data samples associated with $F_1$ that stored on the master machine. The following is our main result on the convergence rate of DANE-HB-LM.

\begin{restatable}[Convergence of DANE-HB-LM]{theorem}{DANEHBLMConvergence}\label{thrm:global_dane_hb_lm}
Assume that the univariate functions $l_i$ are $\ell$-smooth and $\sigma$-strongly convex, and $F$ is $L$-smooth. Let $H=\frac{\ell}{N}XX^\top + \mu I$ and $H_1 = \frac{\ell}{n}X_1 X_1^\top + \mu I$. Choose $\beta = \left(1-\sqrt{\frac{\mu}{\mu+2\gamma}}\right)^2$. If $\|H_1 - H\| \le \gamma$ and $\varepsilon_t \le \frac{\sigma\mu}{4\ell L^2}\|\nabla F(w^{(t-1)})\|^2$ , then Algorithm~\ref{alg:dane_hb_lm} will output solution $w^{(t)}$ with sub-optimality $F(w^{(t)}) - F(w^*) \le \epsilon$ after
\[
t \ge \frac{\ell}{\sigma} \log \left(\frac{2 (F(w^{(0)}) - F(w^*))}{\epsilon}\right)
\]
rounds of outer-loop iteration and
\[
\mathcal{O}\left(\frac{\ell }{\sigma}\sqrt{\frac{\gamma}{\mu}} \log^2 \left(\frac{1}{\epsilon}\right)\right)
\]
rounds of inner-loop iteration of DANE-HB.
\end{restatable}
\begin{remark}
When the univariate function $l_i$ is second-order differentiable, the condition of $l_i$ being $\ell$-smooth and $\sigma$-strongly convex is identical to $\sigma\le l''_i(\cdot)\le\ell$. For the quadratic loss function $l(w^\top x_i,y_i)=\frac{1}{2}(y_i - w^\top x_i)^2$, we have $\ell=\sigma=1$. For the binary logistic loss $l(w^\top x_i, y_i)=\log\left(1+\exp(-y_{i}w^\top x_{i})\right)$, let us consider without loss of generality that $\forall i$, $\|x_i\|\le 1$ and the domain of interest is bounded, i.e., $\|w\|\le B$ for some $B>0$. Then we can verify that $\ell = 1/4$ and $\sigma=\frac{\exp(B)}{(1+\exp(B))^2}$ which do not scale with sample size.
\end{remark}

We also have the following stochastic variant of Theorem~\ref{thrm:global_dane_hb_lm} which is a direct consequence of applying Lemma~\ref{lemma:hessian_close} to the theorem.
\begin{corollary}\label{corol:local_general_hb_lm}
Assume the conditions in Theorem~\ref{thrm:global_dane_hb_lm} hold and $\|\nabla^2 f(w; x_{i}, y_{i})\|\le L$ for all $i\in[N]$. Then for any $\delta\in(0,1)$, with probability at least $1-\delta$ over the samples drawn to construct $F_1(w)$, Algorithm~\ref{alg:dane_hb_lm} with $\gamma=L\sqrt{\frac{32\log(p/\delta)}{n}}$ will attain estimation error $\|w^{(t)} - w^*\| \le \epsilon$ after
\[
t = \mathcal{\tilde O} \left(\frac{\ell\sqrt{\kappa}}{\sigma n^{1/4}}\log^{1/4}\left(\frac{p}{\delta}\right) \log^2 \left(\frac{1}{\epsilon}\right)\right).
\]
rounds of iteration.
\end{corollary}

\begin{remark}
To our best knowledge, this is the first provable nearly-optimal non-asymptotic bound for DANE-type methods for non-quadratic convex functions. In contrast to the DiSCO method~\citep{zhang2015disco} which has similar communication bound but for self-concordant functions, DANE-HB-LM does not need to access the Hessian matrix of the model which could be huge in high dimensional learning problems.
\end{remark}

\subsection{Comparison against prior methods}

\begin{table}
\centering
\begin{tabular}{|c|c|c|}
\hline
 Method &  Ridge regression & Logistic regression \\
\hline
GIANT & $\mathcal{O}\left(\log \left(\frac{\kappa}{\epsilon}\right)\right)$ & \xmark \\
DSVRG & $\mathcal{O}\left(\frac{\kappa}{n} \log \left(\frac{1}{\epsilon}\right)+ \frac{\kappa^2}{mn} \log^2 \left(\frac{1}{\epsilon}\right)\right)$ & $\mathcal{O}\left(\frac{\kappa}{n} \log \left(\frac{1}{\epsilon}\right)+ \frac{\kappa^2}{mn} \log^2 \left(\frac{1}{\epsilon}\right)\right)$ \\
DiSCO & $\mathcal{O}\left(\frac{\sqrt{\kappa}}{n^{1/4}}\log \left(\frac{1}{\epsilon}\right)\right)$  & $\mathcal{O}\left(p^{1/4}\left(\frac{\sqrt{\kappa}}{n^{1/4}}\log \left(\frac{1}{\epsilon}\right) + \frac{\kappa^{3/2}}{n^{3/4}}\right)\right)$ \\
DANE-HB (\textbf{ours}) & $\mathcal{O}\left(\frac{\sqrt{\kappa}}{n^{1/4}}\log \left(\frac{1}{\epsilon}\right)\right)$ & Local rate: $\mathcal {O}\left(\sqrt{\frac{\gamma}{\mu}}\log \left(\frac{1}{\epsilon}\right)\right)$  \\
DANE-HB-LM (\textbf{ours}) & $\mathcal{O}\left(\frac{\sqrt{\kappa}}{n^{1/4}}\log \left(\frac{1}{\epsilon}\right)\right)$ & $\mathcal{O}\left(\frac{\sqrt{\kappa}}{n^{1/4}}\log^2 \left(\frac{1}{\epsilon}\right)\right)$ \\
\hline
\end{tabular}
\caption{Comparison of communication complexity for different distributed learning methods. The x-mark ``\xmark'' indicates that the related result was not explicitly reported in the corresponding reference. \label{tab:result_comparison_non_dane}}
\end{table}

In Table~\ref{tab:result_comparison}, we have listed the communication complexity bounds of DANE-LS and DANE-HB and highlighted their advantages over prior DANE-type methods.
To further compare our methods against other distributed learning algorithms beyond DANE, we list in Table~\ref{tab:result_comparison_non_dane} the amount of communication required by DANE-HB/DANE-HB-LM and several representative sample-distributed learning algorithms for solving ridge regression and logistic regression problems. The amount of communication is measured by the number of vectors of size $p$ transmitted among the networked machines. Here we do not count the communication cost spent for distributing data to machines which is required virtually by all the sample-distributed methods. The only exception is DSVRG which, in addition to data allocation, also requires to distribute a random subset of data in order to guarantee unbiased estimation of batch gradient for local optimization. In the following elaboration, we highlight the key observations that can be made from Table~\ref{tab:result_comparison_non_dane}.
\begin{itemize}
  \item \textbf{Results for ridge regression problem.} In this quadratic loss setting, GIANT~\citep{wang2018giant} has logarithmic dependence on the condition number $\kappa$ and hence is superior to the other methods that have polynomial bounds on $\kappa$. However, such an improvement of GIANT is only valid in the well-conditioned regime where the sample size $N$ should be sufficiently larger than feature dimension $p$. In contrast, without assuming $N\gg p$, DiSCO~\citep{zhang2015disco} and our DANE-HB/DANE-HB-LM require $\mathcal{O}\left(\frac{\sqrt{\kappa}}{n^{1/4}}\log \left(\frac{1}{\epsilon}\right)\right)$ rounds of communications with $\mathcal{O}(p)$ bits communicated per round. The amount required by DSVRG~\citep{lee2017distributed,shamir2016without} is $\mathcal{O}\left(\frac{\kappa}{n} \log \left(\frac{1}{\epsilon}\right)+ \frac{\kappa^2}{mn} \log^2 \left(\frac{1}{\epsilon}\right)\right)$ in which the additional term $\frac{\kappa^2}{mn} \log^2 \left(\frac{1}{\epsilon}\right)$ arises from distributing a multi-set sampled with replacement from the entire data, and it certainly dominates the bound when $\kappa=\Omega(m)$. If this is the case, then DSVRG will be comparable or superior to DiSCO/DANE-HB/DANE-HB-LM when $\kappa = \mathcal{O}(n^{1/2}m^{2/3})$, and otherwise the former will be inferior to the latter in communication efficiency.
  \item \textbf{Results for logistic regression problem.} For general smooth loss functions such as logistic loss, GIANT exhibits linear-quadratic local convergence behavior but without any communication complexity bound explicitly provided. The amount of communication required by DSVRG is still $\mathcal{O}\left(\frac{\kappa}{n} \log \left(\frac{1}{\epsilon}\right)+ \frac{\kappa^2}{mn} \log^2 \left(\frac{1}{\epsilon}\right)\right)$. For DiSCO, the communication complexity becomes  $\mathcal{O}\left(p^{1/4}\left(\frac{\sqrt{\kappa}}{n^{1/4}}\log \left(\frac{1}{\epsilon}\right) + \frac{\kappa^{3/2}}{n^{3/4}}\right)\right)$ which tends to be inferior to DSVRG especially in high dimensional settings due to its polynomial dependence on $p$. DANE-HB has the $\mathcal{O}\left(\sqrt{\gamma/\mu}\log \left(1/\epsilon\right)\right)$ bound in a local area around the minimizer, provided that the local and global objectives are $\gamma$-related, i.e., $\sup_{w} \|\nabla^2 F_1(w) - \nabla^2 F(w)\| \le \gamma$. For DANE-HB-LM, the required amount of communications is bounded by $\mathcal{O}\left(\frac{\sqrt{\kappa}}{n^{1/4}}\log^2 \left(\frac{1}{\epsilon}\right)\right)$. In view of the discussions in the previous quadratic case, given that $\kappa=\Omega(m)$, DSVRG will be comparable or superior to DANE-HB-LM when $\kappa = \mathcal{O}(n^{1/2}m^{2/3})$, and otherwise DANE-HB-LM should be more cheaper in communication.
\end{itemize}
To summarize the above discussions, DANE-HB/DANE-HB-LM is able to offer comparable or superior communication efficiency to the considered distributed learning algorithms in high-dimensional and ill-conditioned (e.g., $\kappa = \Omega(n^{1/2}m^{2/3})$) regimes.

\section{Experiments}
\label{sect:experiment}

In this section, we present a numerical study for theory verification and algorithm evaluation. In the theory verification part, we conduct simulations on linear regression and binary logistic regression problems to verify the strong convergence guarantees established for DANE-LS, DANE-HB and DANE-HB-LM. Then in the algorithm evaluation part, we run experiments on synthetic and real data binary logistic regression tasks to evaluate the numerical performance of these alternatives with comparison to several state-of-the-art distributed learning methods. We simulate the distributed environment on a single server powered by dual Intel(R) Xeon(R) E5-2630V4@2.2GHz CPU with multiple logic processors simulating multiple machines. All the considered methods are implemented in Matlab R2018b on Microsoft Windows 10. The local subproblems in DANE are solved by an SVRG solver from SGDLibrary~\citep{kasai2017sgdlibrary}, and the momentum coefficient $\beta$ in DANE-HB is set according to Theorem~\ref{thrm:quadratic_hb}. We replicate each experiment $10$ times over random split of data and report the results in mean-value along with error bar. We initialize $w^{(0)} = 0$ throughout our numerical study.

\subsection{Theory verification}

The following experimental protocol is considered for theory verification study.
 \begin{itemize}
   \item To verify the bounds established in Theorem~\ref{thrm:quadratic_dane} for DANE-LS and in Theorem~\ref{thrm:quadratic_hb} for DANE-HB for quadratic problems, we consider the ridge regression model with loss function $f(w; x_i, y_i)=\frac{1}{2}(y_i-w^{\top}x_i)^2 + \frac{\mu}{2}\|w\|^2$. The feature points $\{x_i\}_{i=1}^N$ are sampled from standard multivariate normal distribution. The responses $\{y_i\}_{i=1}^N$ are generated according to a linear model $y_i=\bar{w}^{\top}x_i+e_i$ with a random Gaussian vector $\bar{w}\in \mathbb{R}^p$ and random Gaussian noise $e_i\sim \mathcal{N}(0,\sigma^2)$.
   \item For DANE-HB-LM, we verify its communication complexity bounds as presented in Theorem~\ref{thrm:global_dane_hb_lm} by applying it to the  binary logistic regression model with loss function $f(w; x_i, y_i)=\log\left(1+\exp(-y_{i}w^\top x_{i})\right)+ \frac{\mu}{2}\|w\|^2$. We consider a simulation task in which each data feature $x_i$ is sampled from standard multivariate normal distribution and its binary label $y_i \in \{-1,+1\}$ is determined by the conditional probability $\mathbb{P}(y_i|x_i; \bar w) = \exp (2y_i \bar w^\top x)/(1+\exp (2y_i \bar w^\top x_i))$ with a Gaussian vector $\bar{w}$.
 \end{itemize}

For our simulation study, we test with feature dimensions $p\in\{200, 500\}$.   We fix $N=10 p$, $\mu = 1/\sqrt{N}$, and study the impact of varying number of machines $m$ and regularization $\gamma$ on the needed rounds of communication to reach sub-optimality $\epsilon=10^{-6}$. We replicate the experiment $10$ times over random split of data.\\

\begin{figure}[h!]
\begin{center}
\mbox{
\subfigure[$p=200$]{
\includegraphics[width=2.1in]{fig/dane_sharp_analysis/p_200/dane_communication_numMachines_root.eps}\hspace{-0.13in}
\includegraphics[width=2.1in]{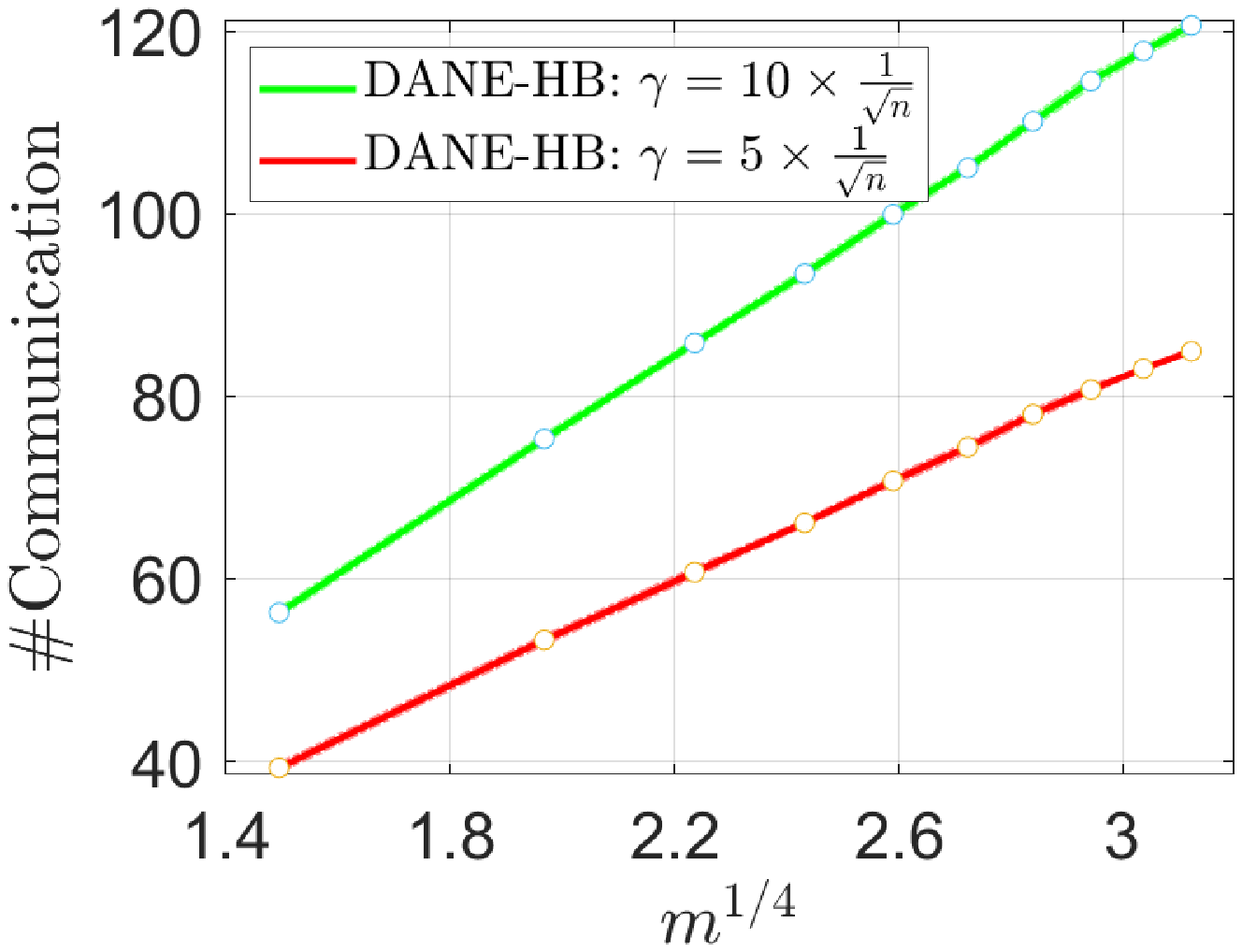}\hspace{-0.13in}
\includegraphics[width=2.1in]{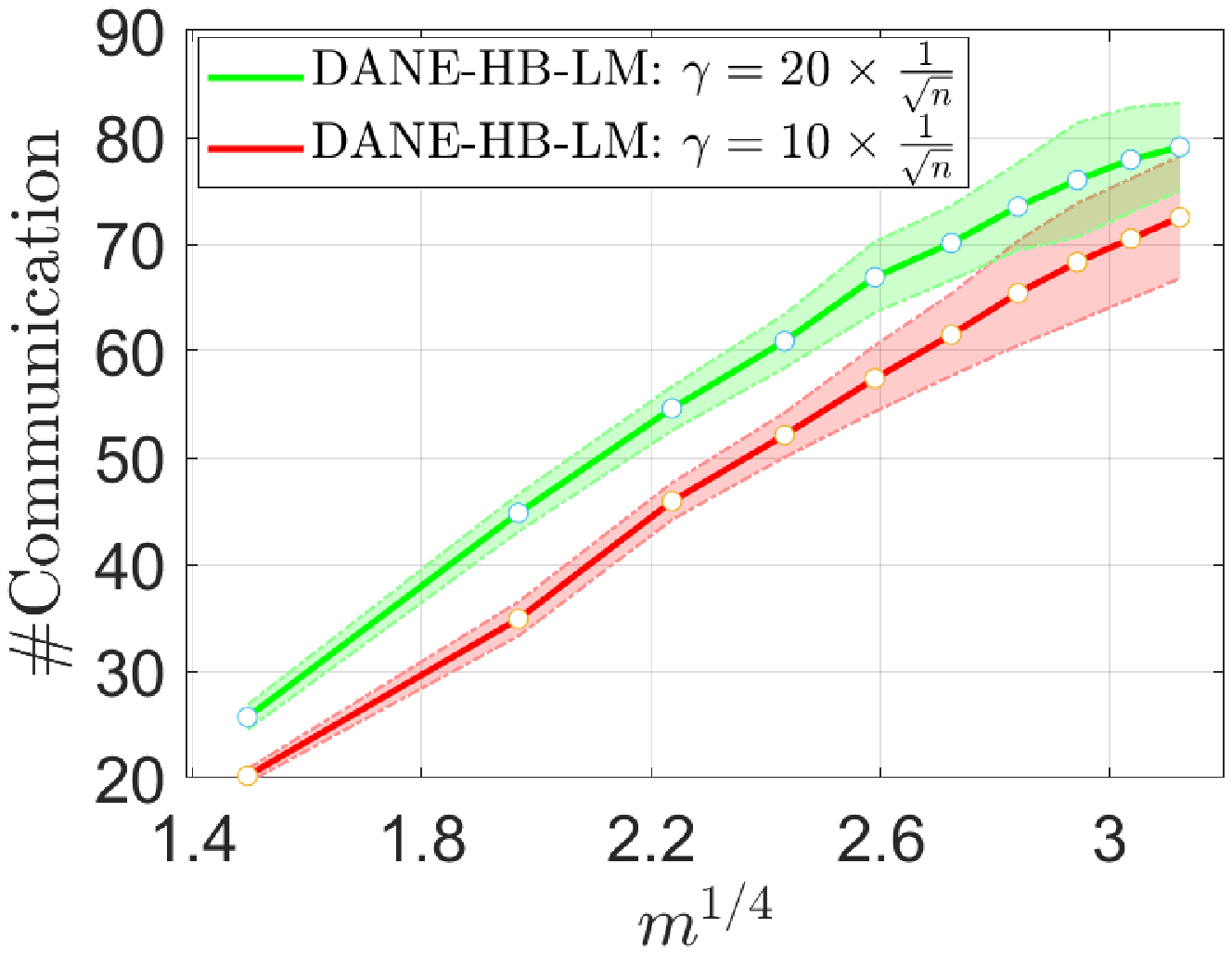}
}}

\mbox{
\subfigure[$p=500$]{
\includegraphics[width=2.1in]{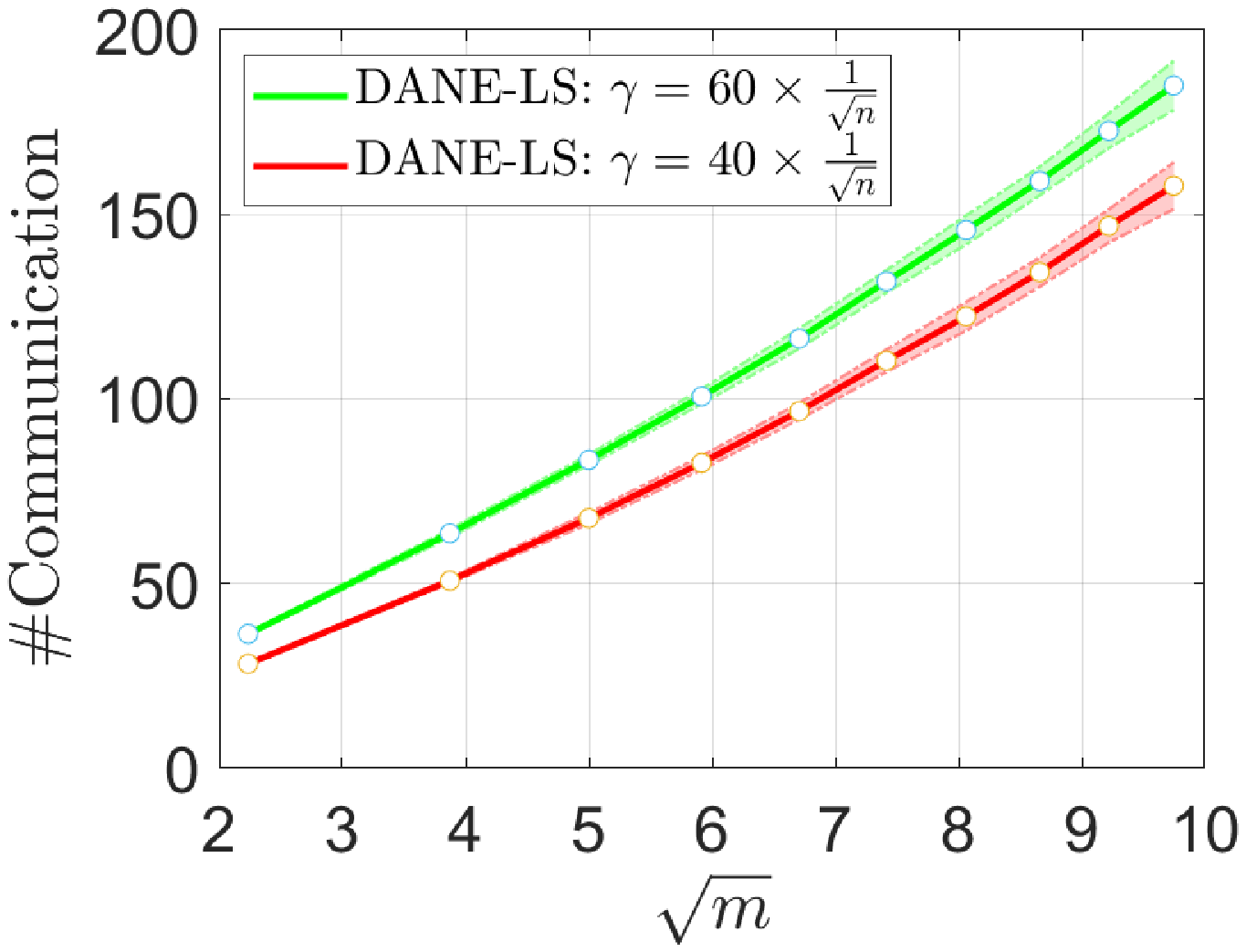}\hspace{-0.13in}
\includegraphics[width=2.1in]{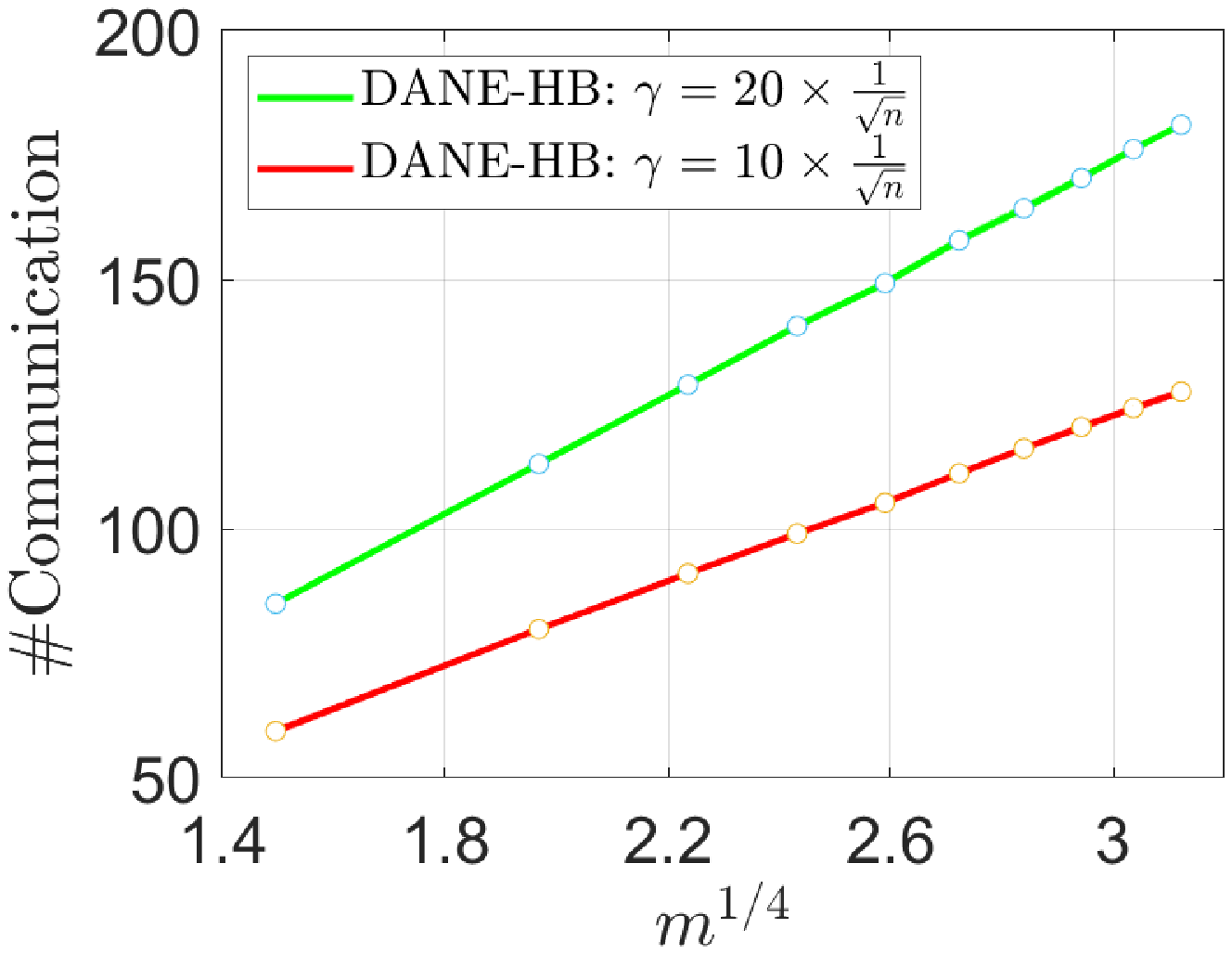}\hspace{-0.13in}
\includegraphics[width=2.1in]{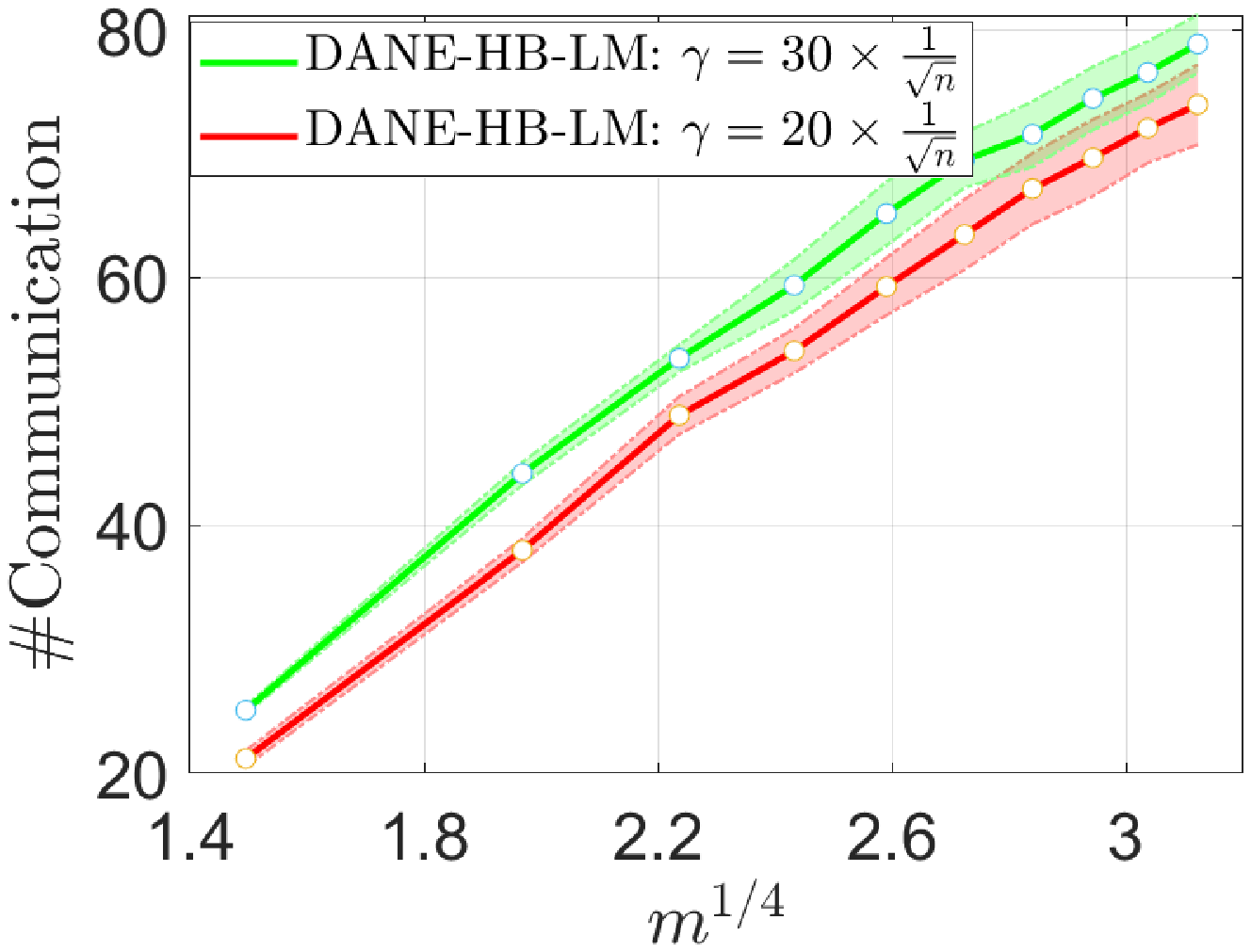}
}}

\end{center}
\vspace{-0.2in}
\caption{Theory verification: the number of communication rounds (y-axis) versus number of machines (x-axis) curves of DANE-LS (left panels) and DANE-HB (middle panels) on a synthetic ridge regression task, and of DANE-HB-LM (right panels) on a synthetic logistic regression task.}
\label{fig:theory_sharpbound_add}
\end{figure}

\textbf{Results.} Figure~\ref{fig:theory_sharpbound_add} shows the evolving curves (error bar shaded in color) of the needed communication rounds as functions of number of machines achieved by DANE-LS (left panel), DANE-HB (middle panel) and DANE-HB-LM (right panel)in the considered setting. Visually speaking, the number of communication rounds scales roughly linearly with respect to $\sqrt{m}$ for DANE-LS and to $m^{1/4}$ for DANE-HB and DANE-HB-LM, under varying values of $\gamma$. We can also observe that smaller $\gamma$ always leads to fewer rounds of communication. These results confirm the theoretical predictions in Theorem~\ref{thrm:quadratic_dane}, Theorem~\ref{thrm:quadratic_hb} and Theorem~\ref{thrm:global_dane_hb_lm}.

\begin{figure}[h]
\begin{center}
\mbox{
\subfigure[Synthetic]{
\includegraphics[width=2.1in]{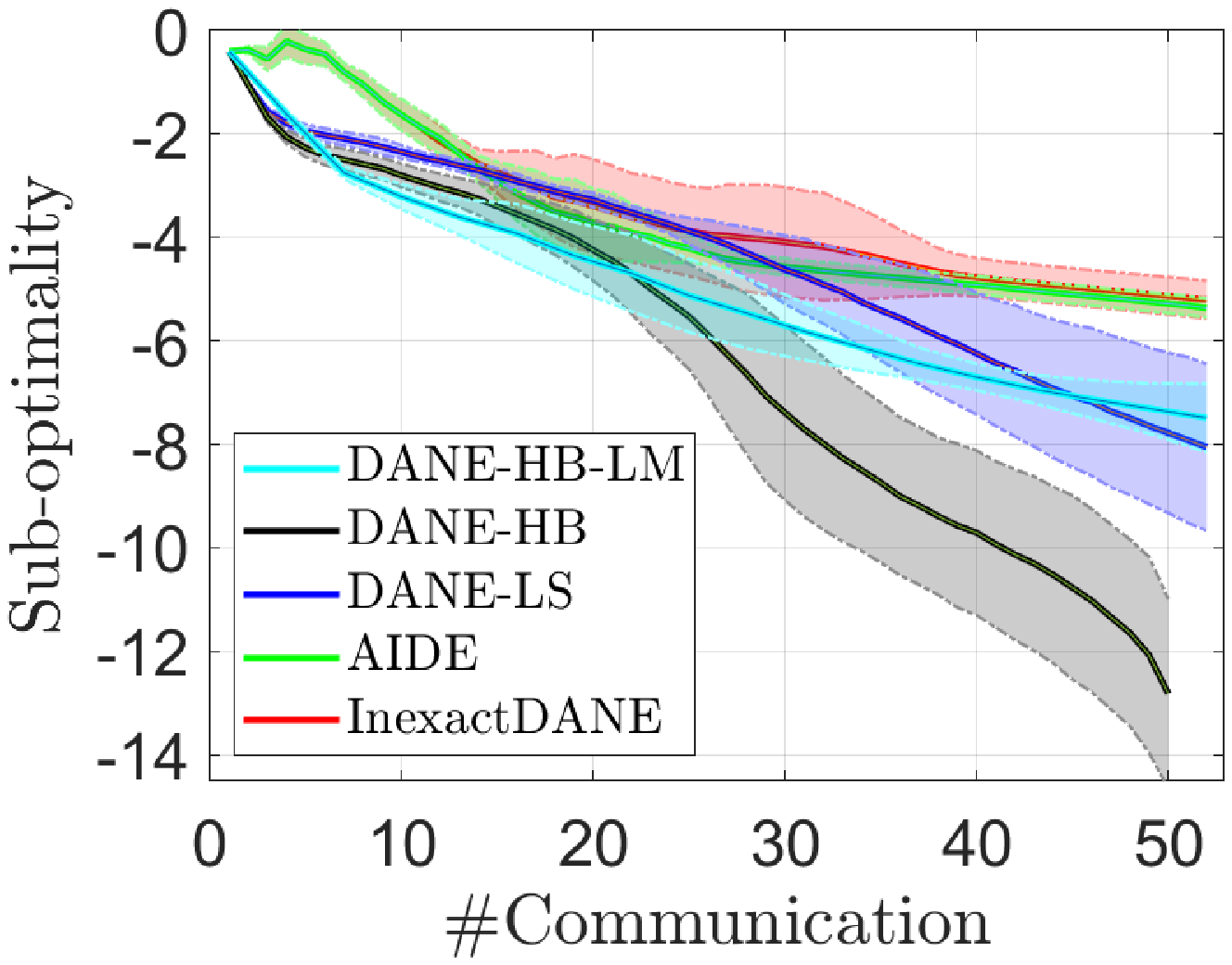}\hspace{-0.13in}
\includegraphics[width=2.1in]{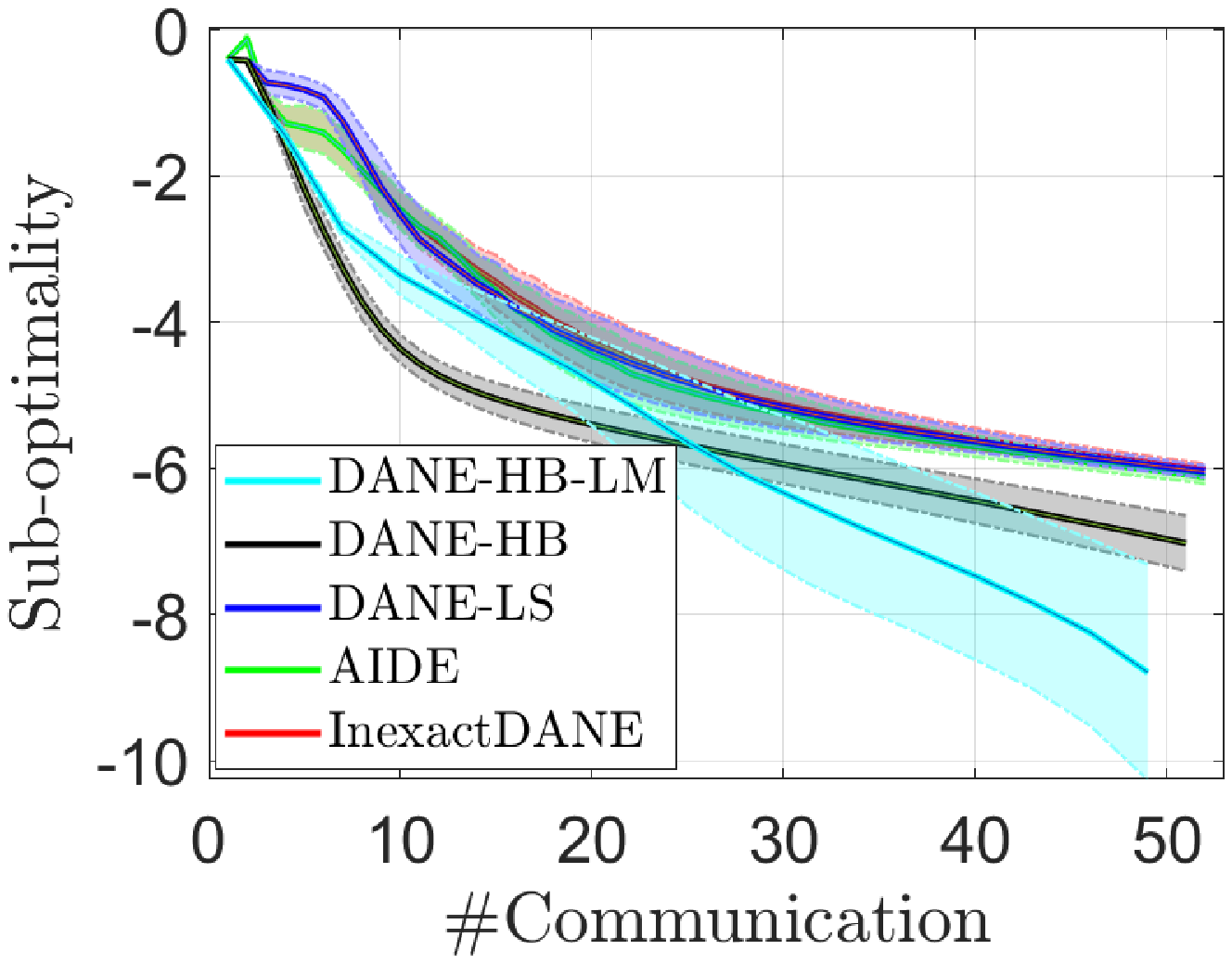}\hspace{-0.13in}
\includegraphics[width=2.1in]{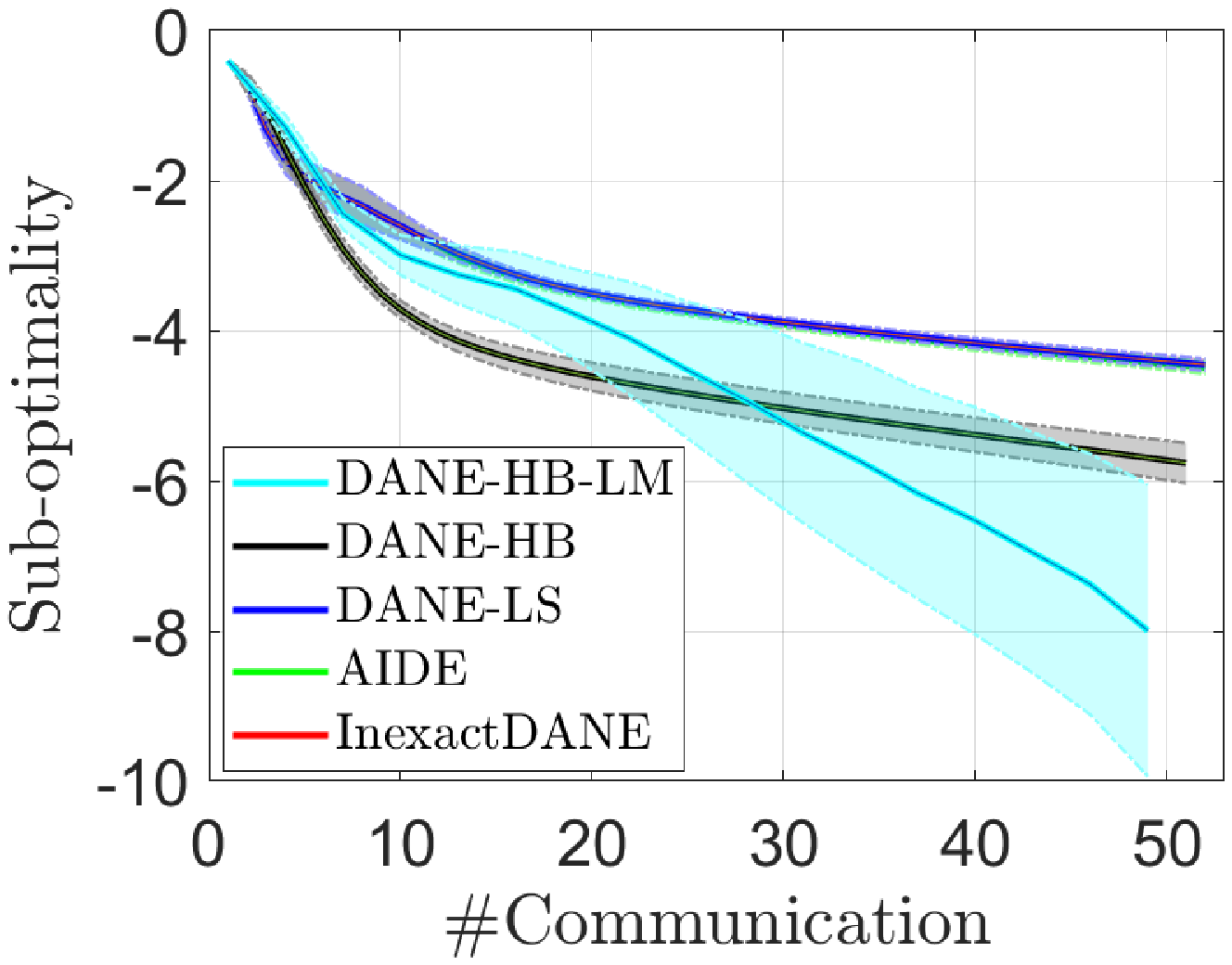}
\label{fig:algorithm_synthetic_add}}
}

\mbox{
\subfigure[\texttt{gisette}]{
\includegraphics[width=2.1in]{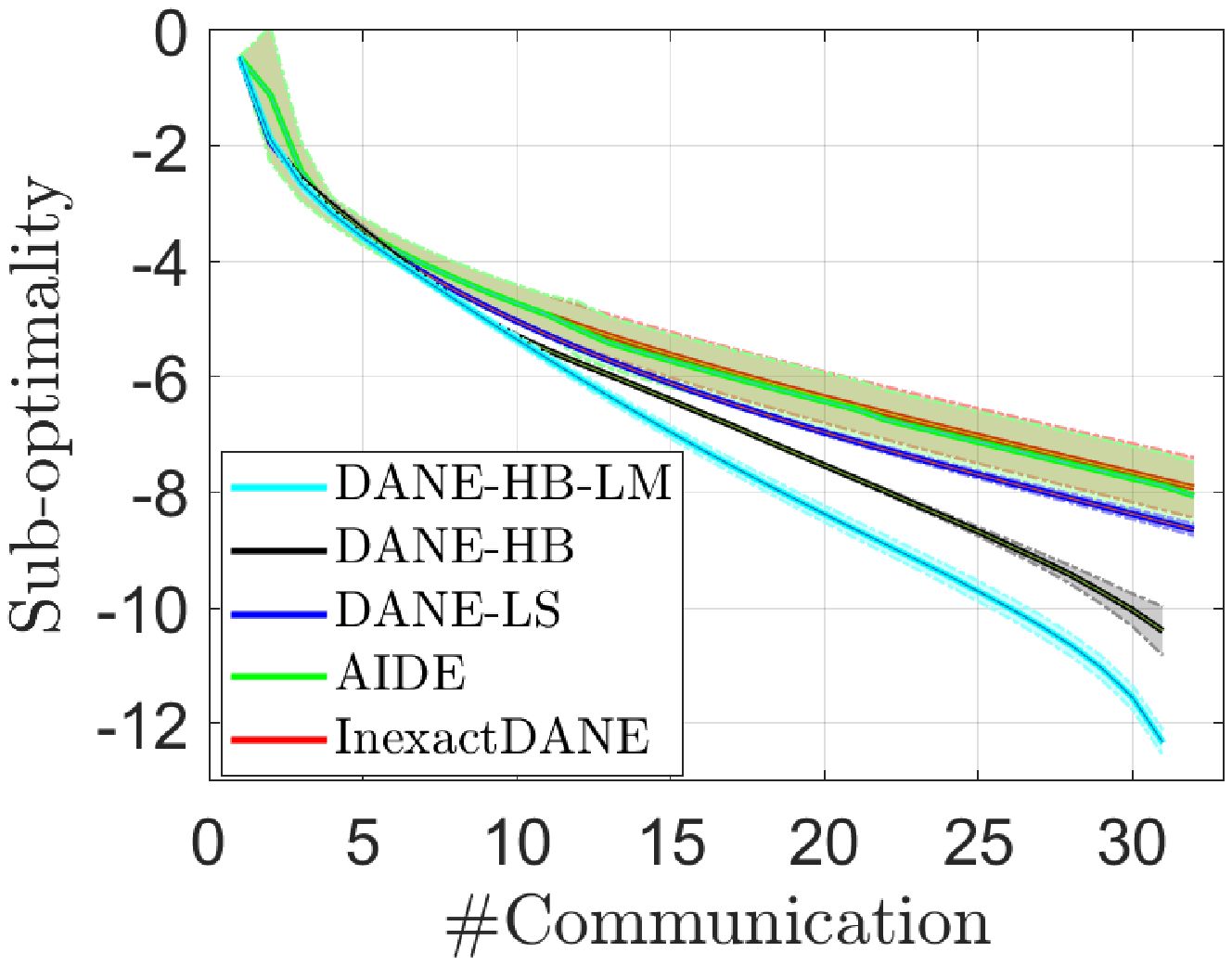}\hspace{-0.13in}
\includegraphics[width=2.1in]{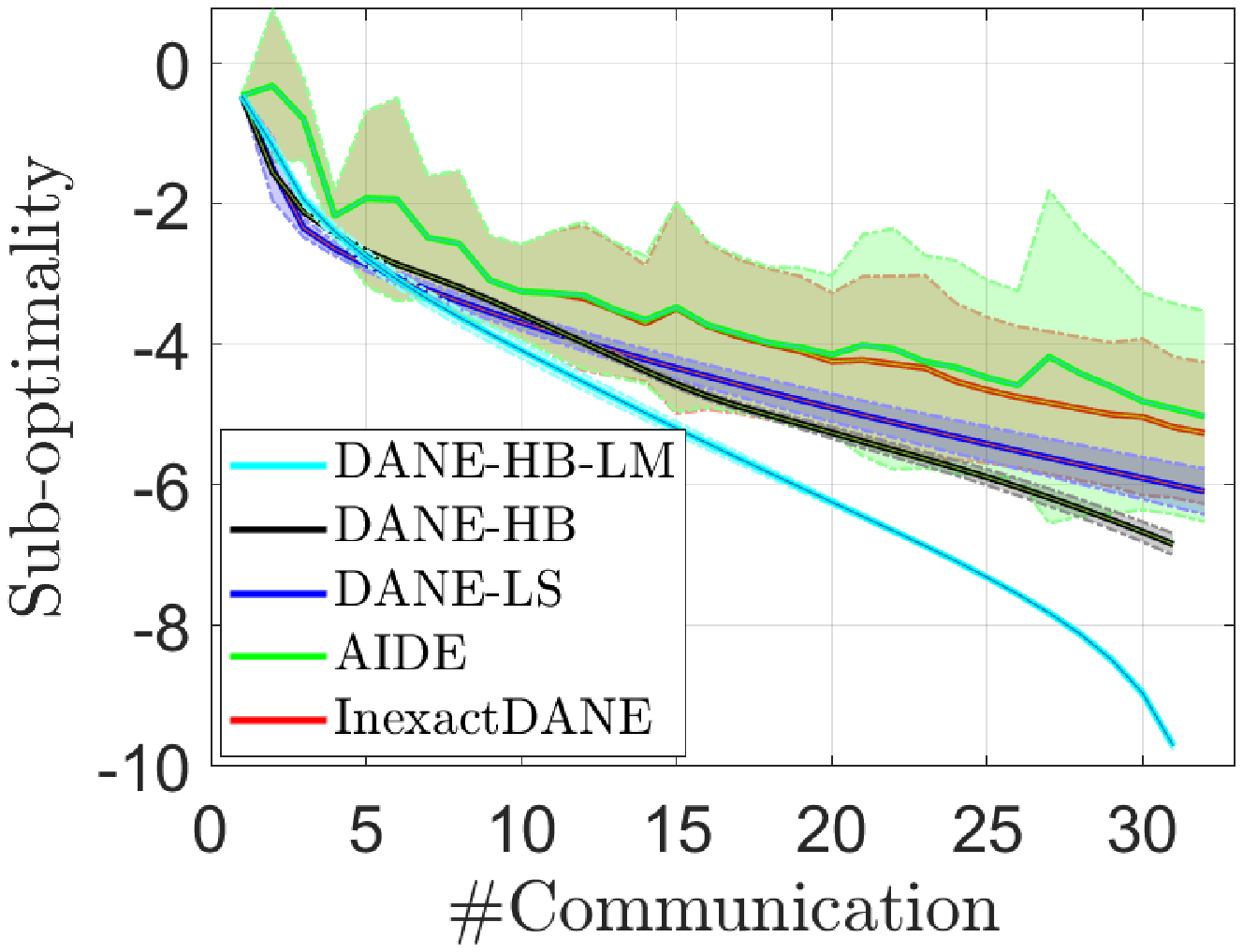}\hspace{-0.13in}
\includegraphics[width=2.1in]{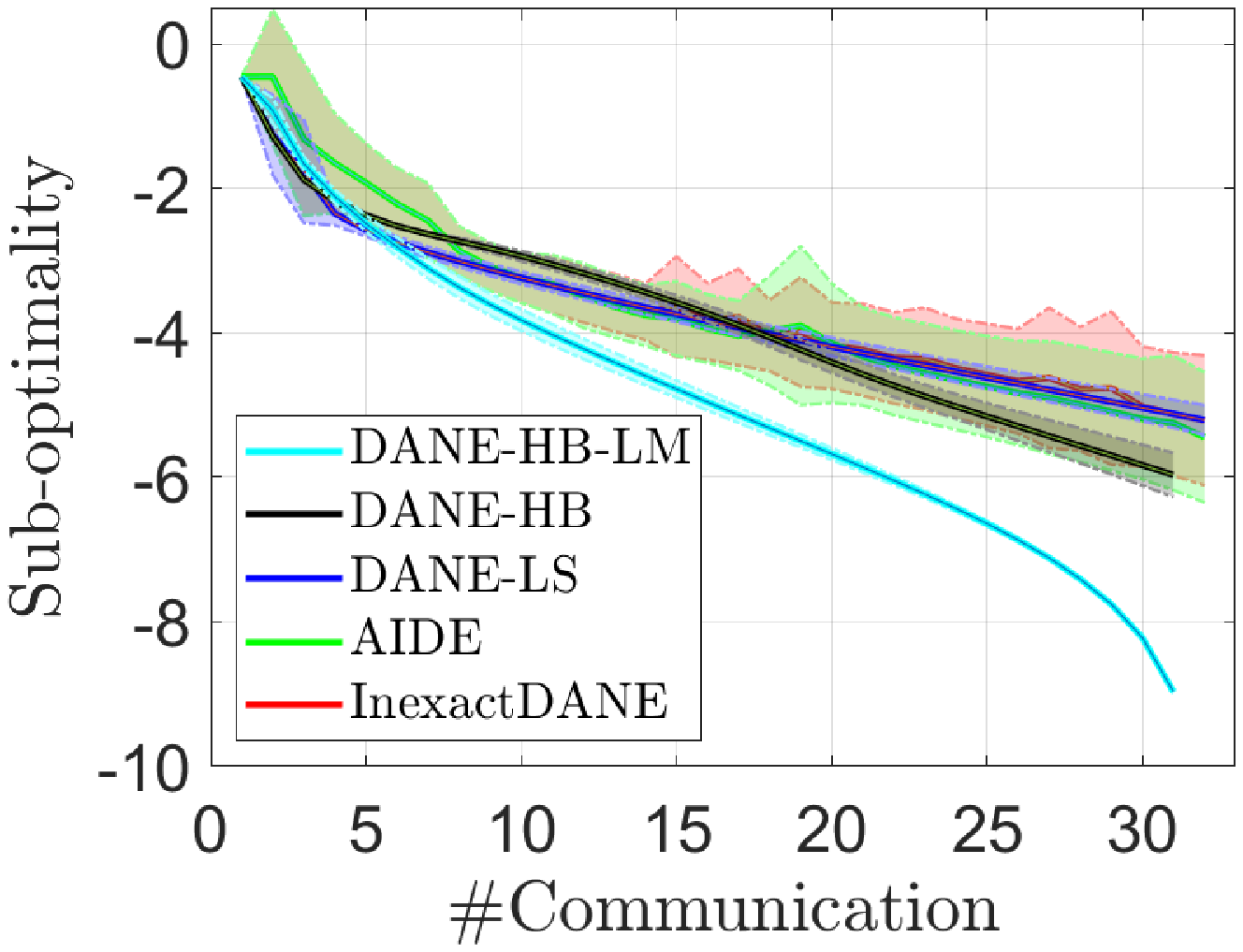}
\label{fig:algorithm_real_gistte_add}}
}

\mbox{
\subfigure[\texttt{rcv1.binary}]{
\includegraphics[width=2.1in]{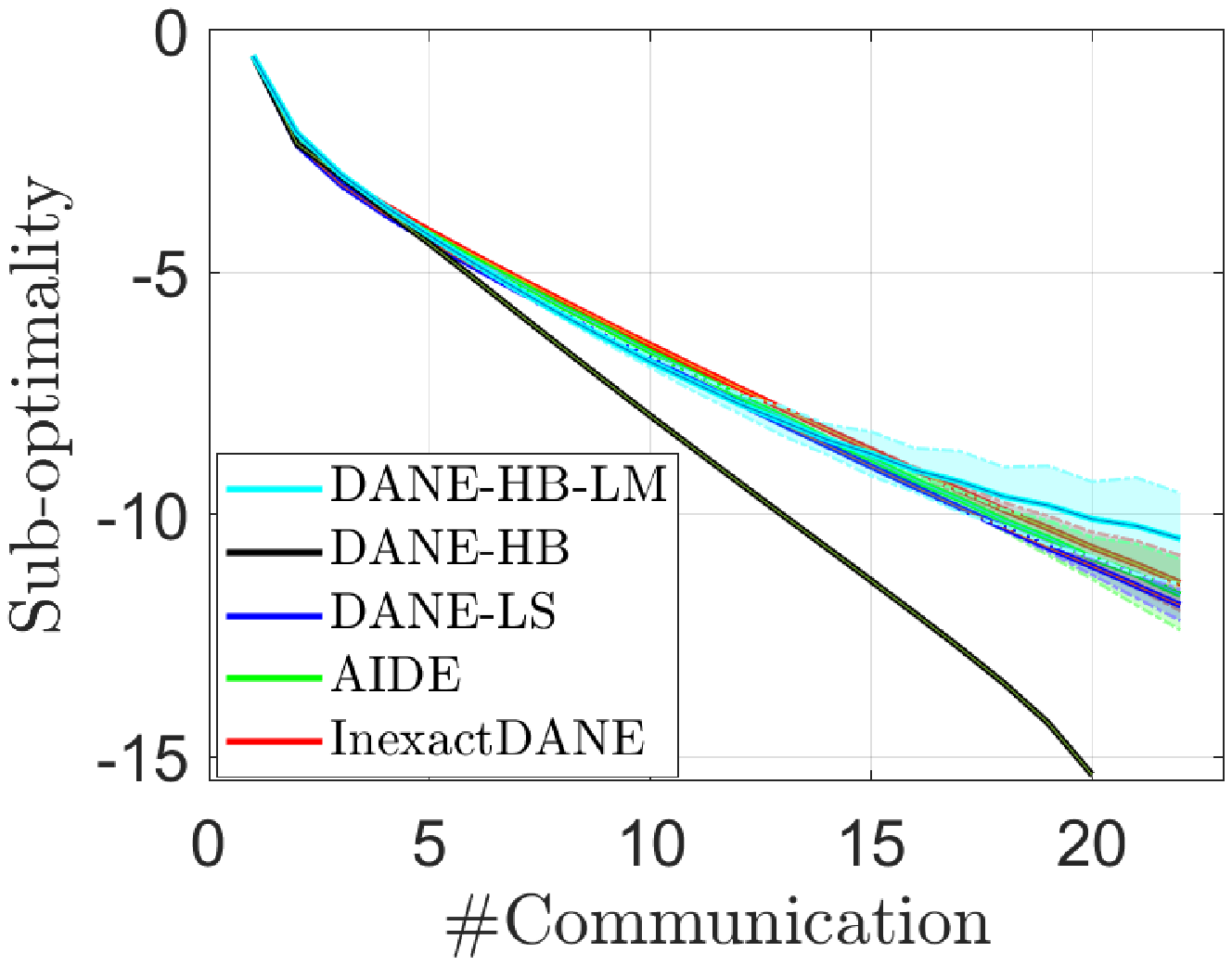}\hspace{-0.13in}
\includegraphics[width=2.1in]{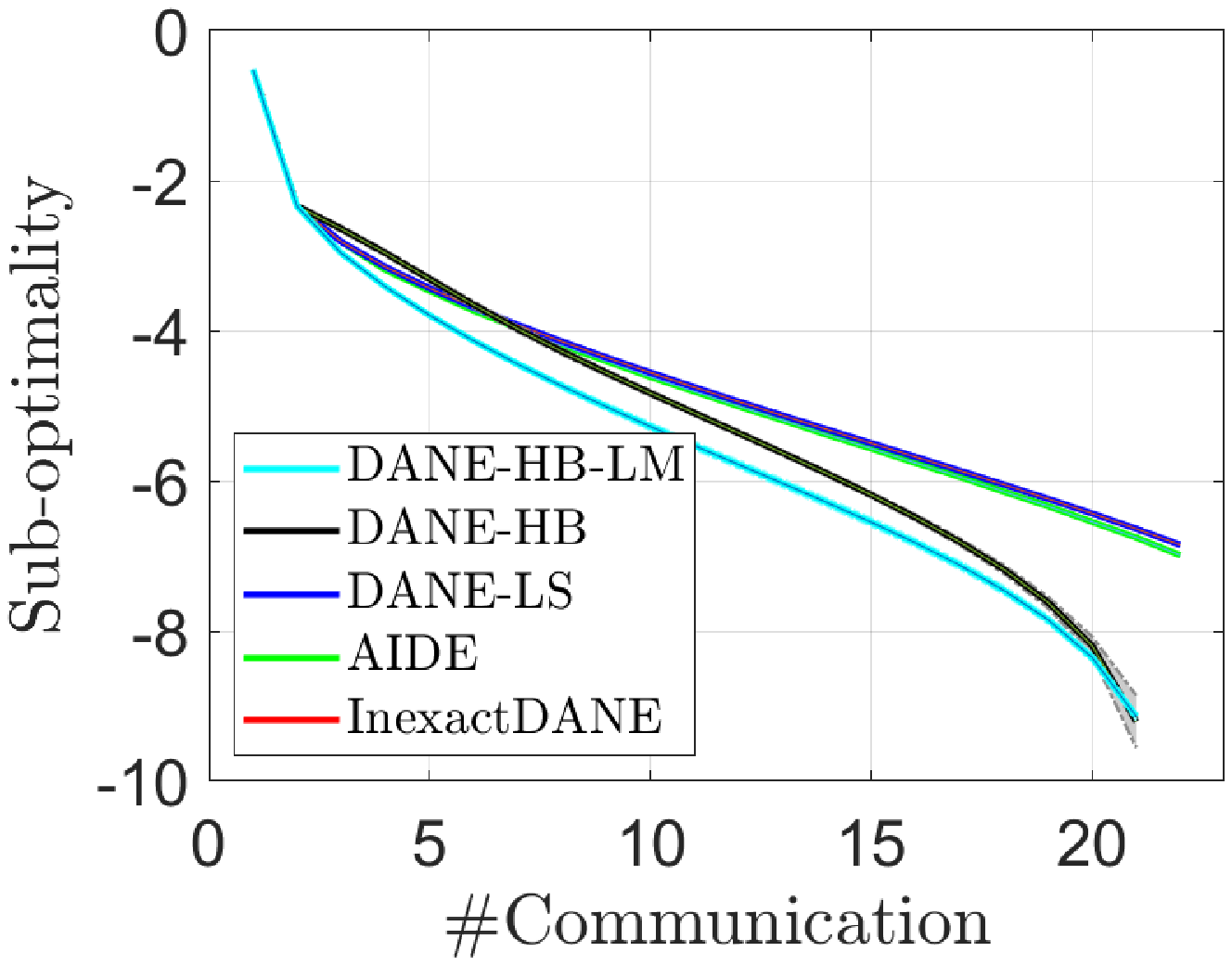}\hspace{-0.13in}
\includegraphics[width=2.1in]{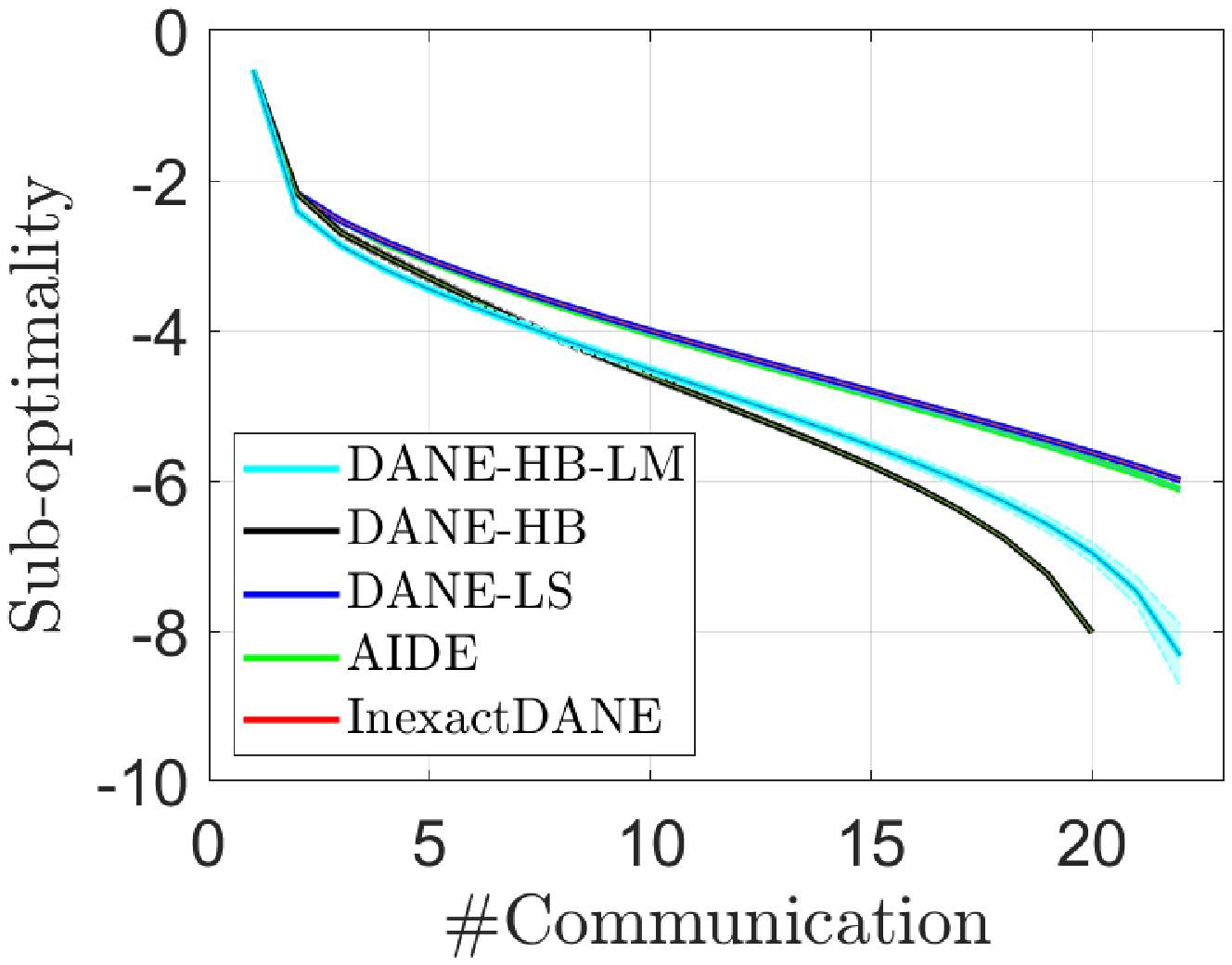}
\label{fig:algorithm_real_rcv1_add}}}
\end{center}
\vspace{-0.2in}
\caption{Algorithm evaluation with comparison to DANE-type methods: the objective value evolving curves on synthetic and real logistic regression tasks with $m=4$ (left panels), $m=16$ (middle panels) and $m=32$ (right panels).  Best viewed in color.}
\label{fig:algorithm_add}
\end{figure}

\subsection{Algorithm evaluation}

We further compare the convergence performance of DANE-LS and DANE-HB/DANE-HB-LM with several representative communication-efficient distributed learning methods. For the sake of presentation clarity, we divide the numerical study into two categories using the DANE-type methods and other type of methods as baselines respectively.

\subsubsection{Comparison against DANE-type methods}

In this part, we carry out experiments to compare our methods with \InexactDane and AIDE, both are developed by~\citet{reddi2016aide}, for binary logistic regression problems. We begin with a simulation study using the same data generation protocol as in the previous theory verification study. We test with $p=200$, $N=10p$, $\gamma=40/\sqrt{n}$, $\mu = 1/\sqrt{N}$ and $m\in \{4, 16, 32\}$. Figure~\ref{fig:algorithm_synthetic_add} shows the objective value convergence curves (w.r.t. communication rounds) of the considered algorithms. From these curves we can see that DANE-LS and DANE-HB/DANE-HB-LM are stable in convergence while \InexactDane and AIDE exhibit strong zigzag effect in the early stage of iteration when $m=4, 16$. The convergence instability of the plain DANE method has also been observed in~\citep{shamir2014communication}. The stability of our proposed methods shows the benefit of line search for improving the convergence behavior of DANE-type methods. In terms of communication efficiency, it can be seen that: i) DANE-LS is superior or comparable to \InexactDane and AIDE in decreasing the global objective value after the same rounds of communication; and ii) DANE-HB and DANE-HB-LM converge considerably faster than the other methods. These observations confirm the effectiveness of heavy-ball approach for accelerating the communication efficiency of DANE.

Next, we evaluate the convergence performance of the considered algorithms on two real data sets \texttt{gisette}~\citep{guyon2005result} ($p=5000$, $N=6000$) and \texttt{rcv1.binary}~\citep{lewis2004rcv1} ($p=47236$, $N=20242$). For each data set, we fix the regularization parameter $\mu = 10^{-5}$ and test with $m\in\{4,16, 32\}$. The results are shown in Figure~\ref{fig:algorithm_add} from which we have the following observations:
\begin{itemize}
  \item For \texttt{gisette}, it can be observed from Figure~\ref{fig:algorithm_real_gistte_add} that DANE-LS and DANE-HB/DANE-HB-LM converge much more stably than \InexactDane and AIDE, which again demonstrates the effectiveness of backtracking line search adopted by our methods. In terms of communication efficiency, DANE-HB-LM outperforms the other considered methods with a clear margin and DANE-HB is the runner-up. DANE-LS converges slightly faster than \InexactDane and AIDE when $m=4,16$, while the former is comparable to the latter ones when $m=32$.
  \item For \texttt{rcv1.binary}, Figure~\ref{fig:algorithm_real_rcv1_add} shows that all the considered algorithms converge smoothly, and thus line search does not help much to improve performance. In most cases, DANE-HB and DANE-HB-LM are superior to DANE-LS, \InexactDane and AIDE which exhibit very close performance on this data.
\end{itemize}

To summarize this group of experiments, our proposed algorithms are stabler than the prior DANE-type methods which matches the global convergence theory established for our algorithms. Particularly, DANE-HB and DANE-HB-LM tend to substantially outperform the other methods in communication efficiency.

\begin{figure}[h]
\begin{center}
\mbox{
\subfigure[Synthetic]{
\includegraphics[width=2.1in]{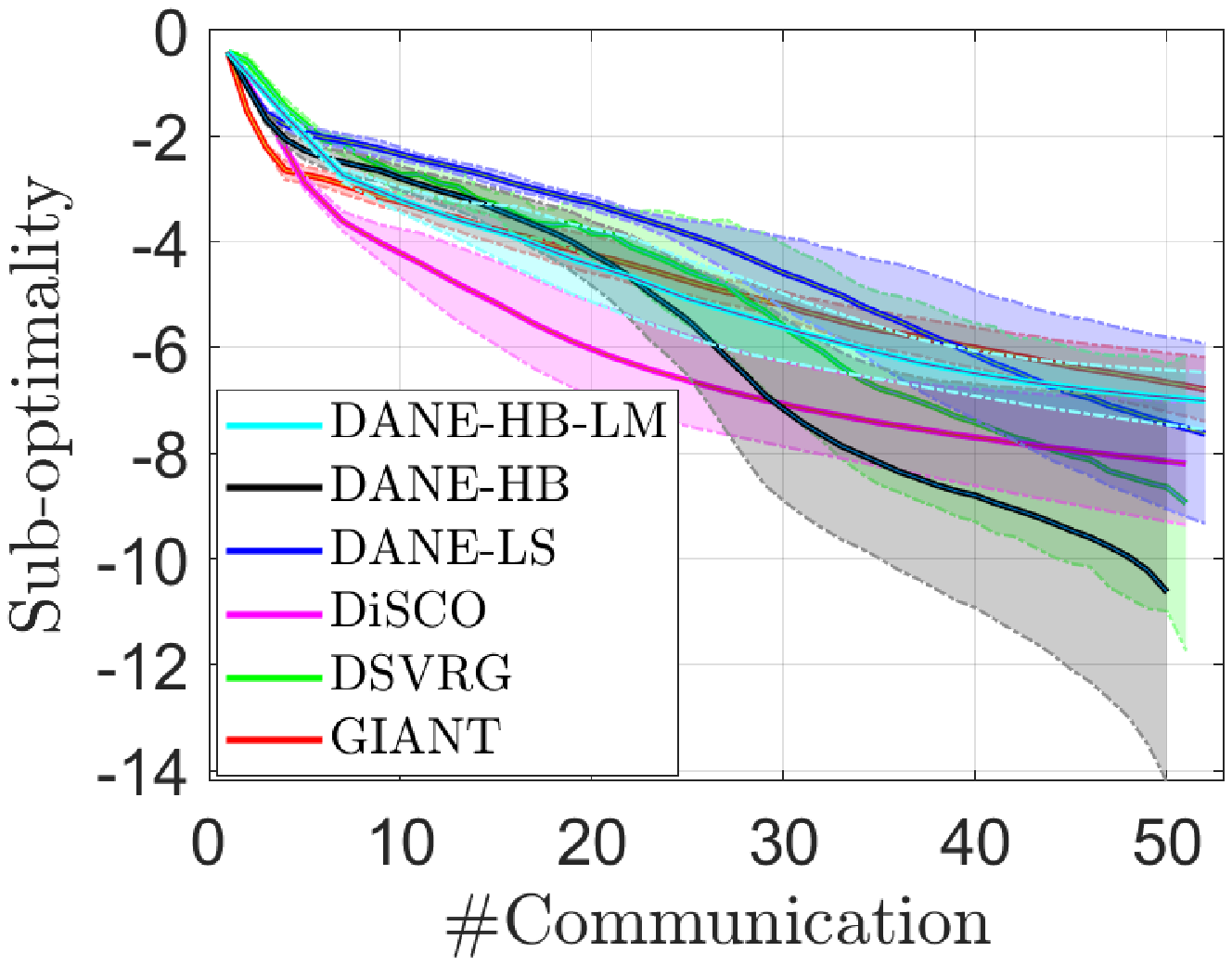}\hspace{-0.13in}
\includegraphics[width=2.1in]{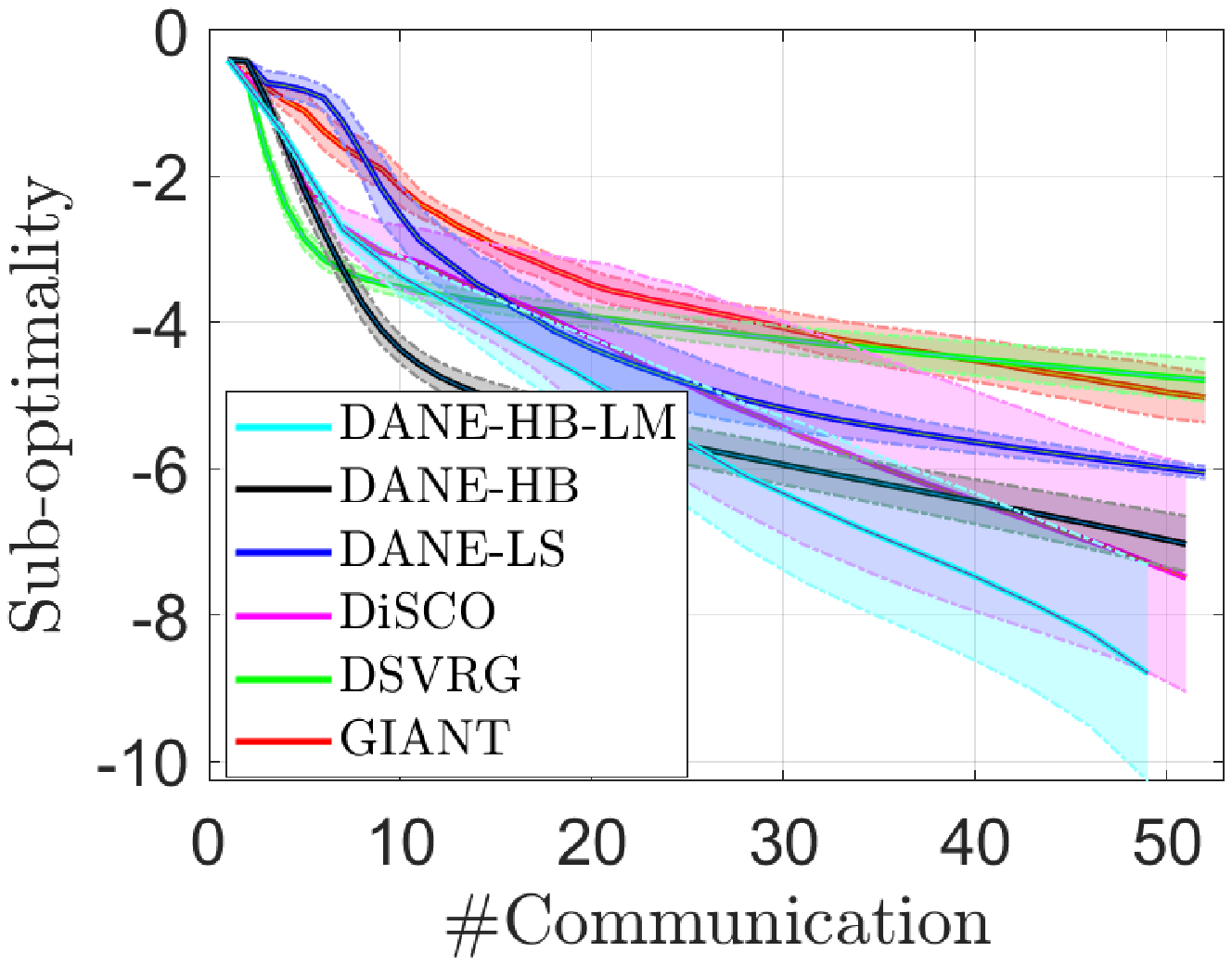}\hspace{-0.13in}
\includegraphics[width=2.1in]{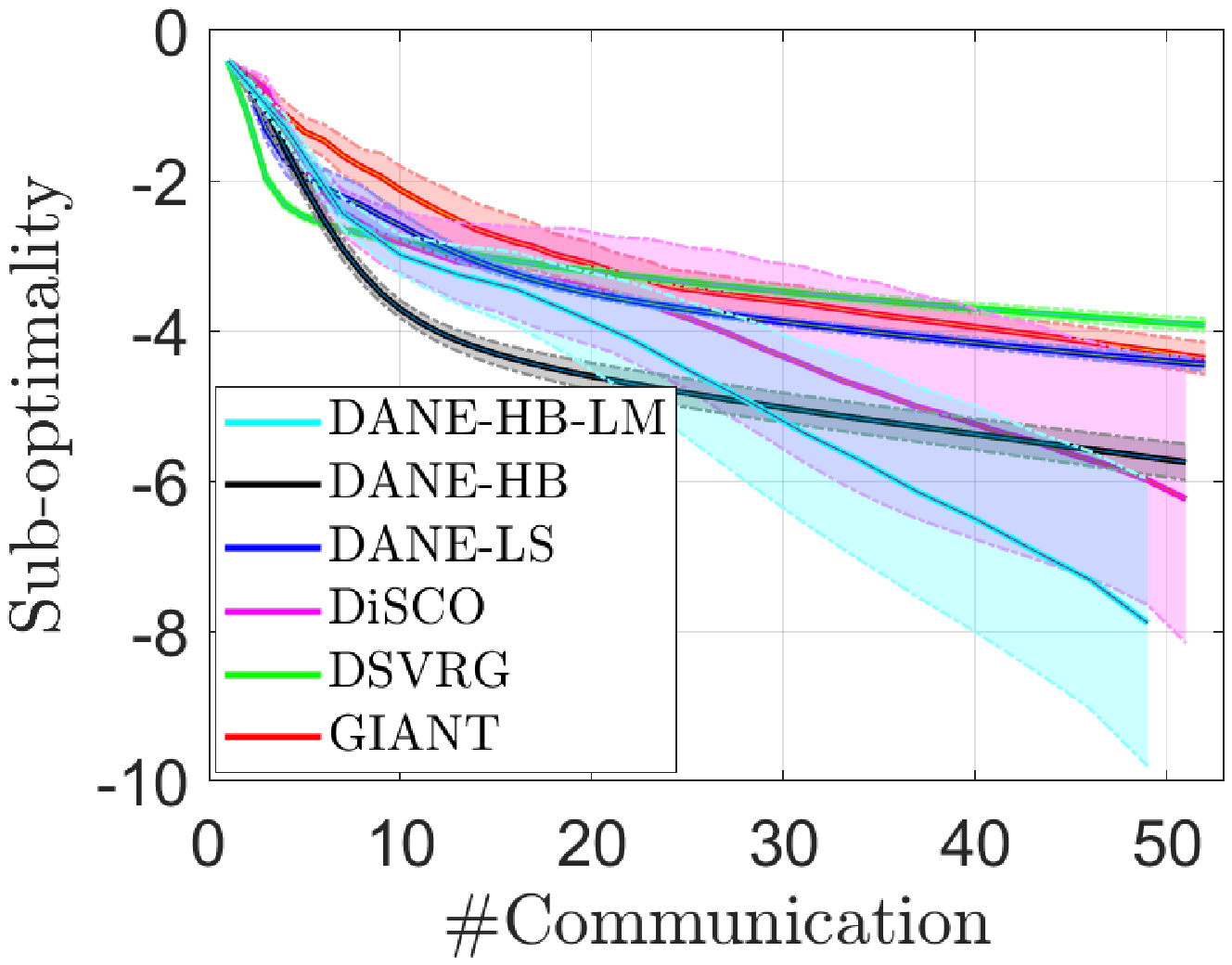}
\label{fig:algorithm_synthetic_add_other}}
}

\mbox{
\subfigure[\texttt{gisette}]{
\includegraphics[width=2.1in]{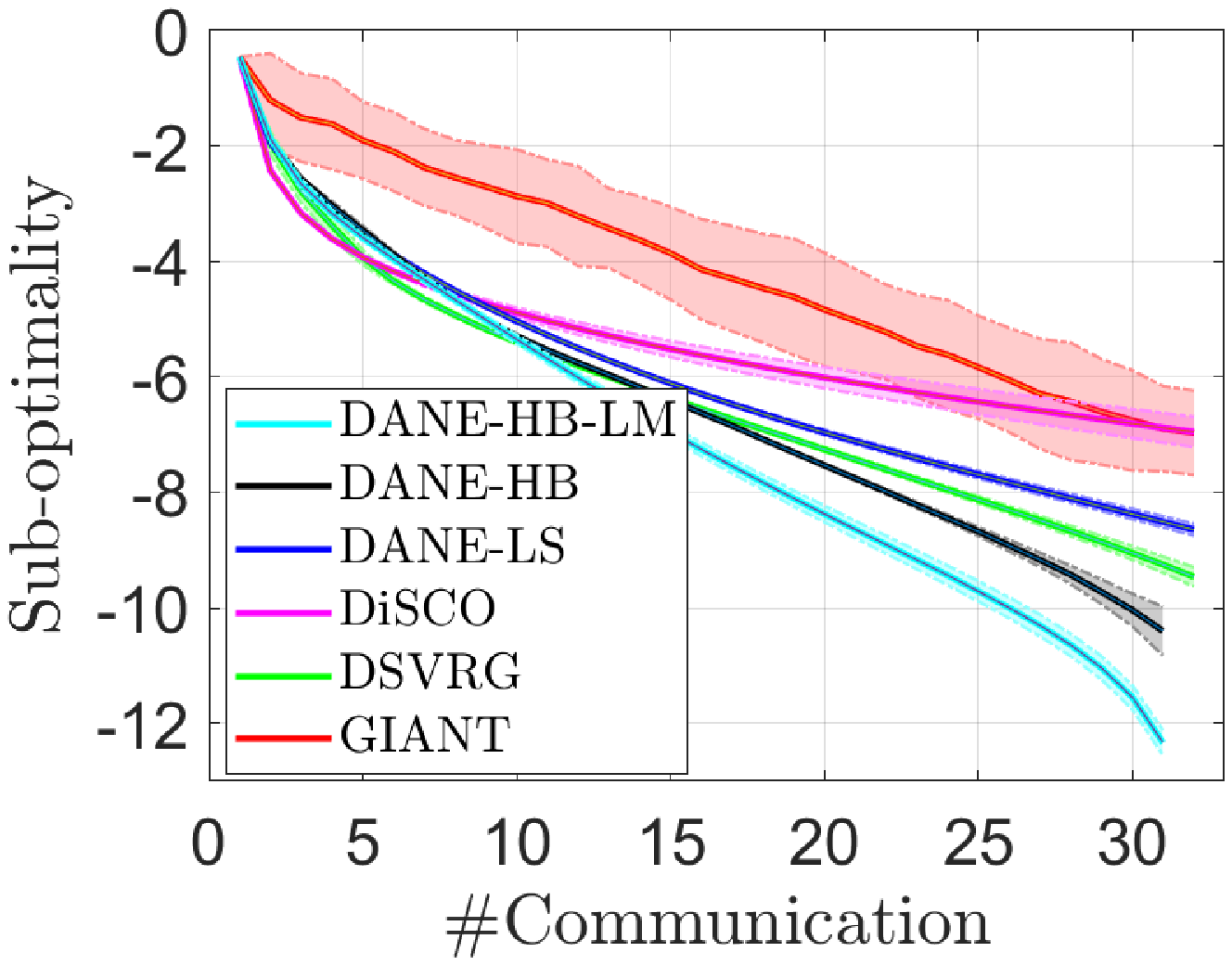}\hspace{-0.13in}
\includegraphics[width=2.1in]{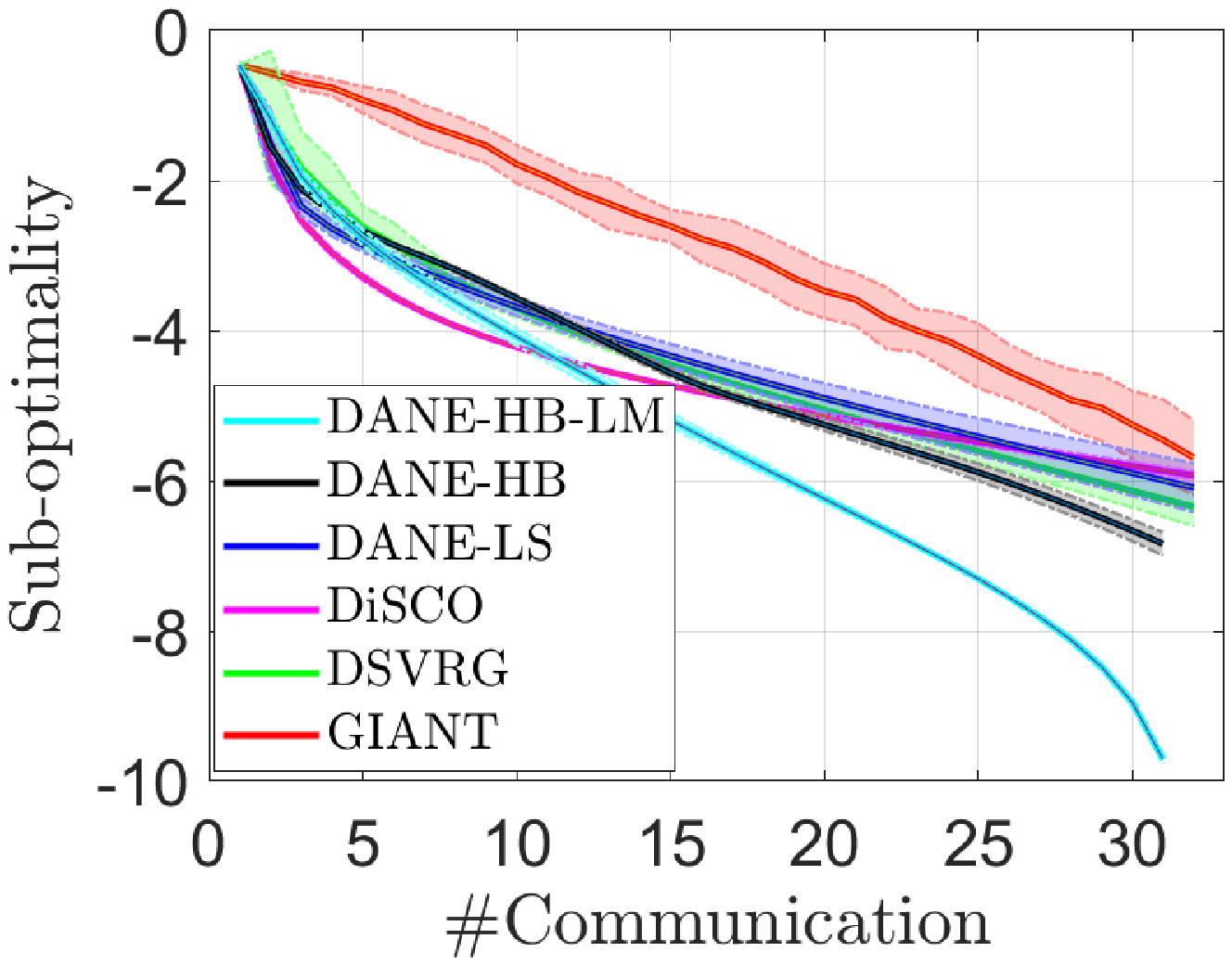}\hspace{-0.13in}\hspace{-0.13in}
\includegraphics[width=2.1in]{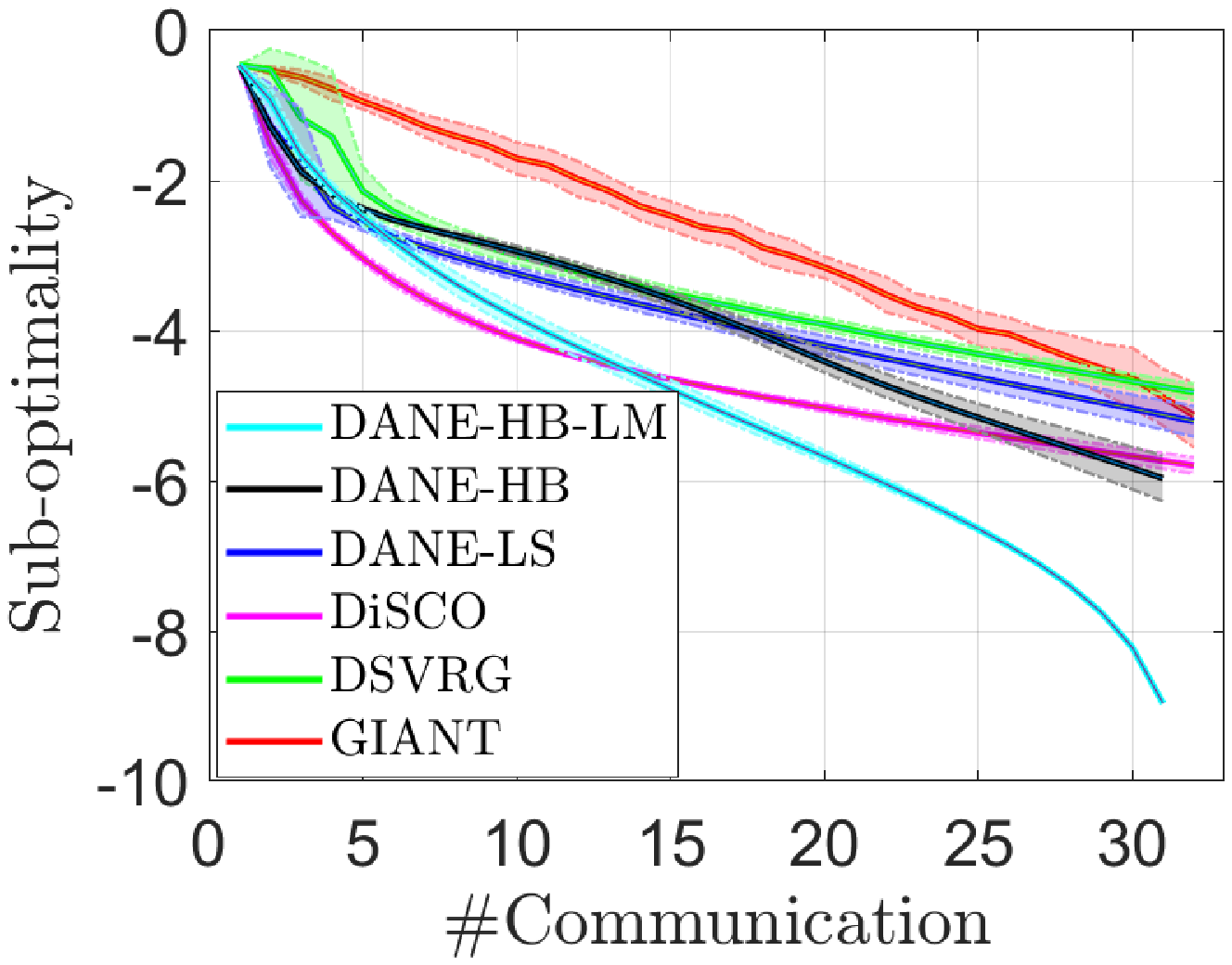}
\label{fig:algorithm_real_gistte_add_other}}
}

\mbox{
\subfigure[\texttt{rcv1.binary}]{
\includegraphics[width=2.1in]{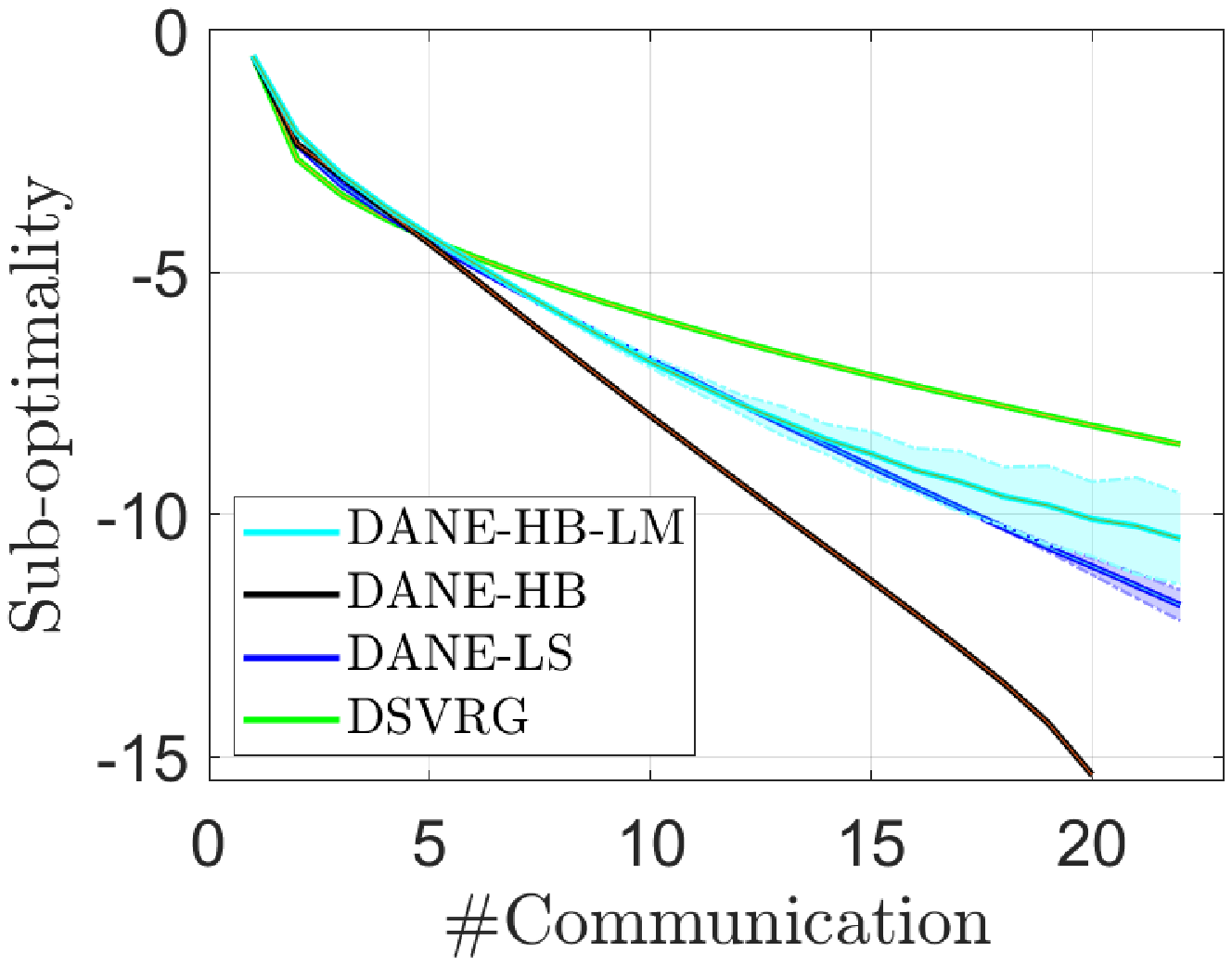}\hspace{-0.13in}
\includegraphics[width=2.1in]{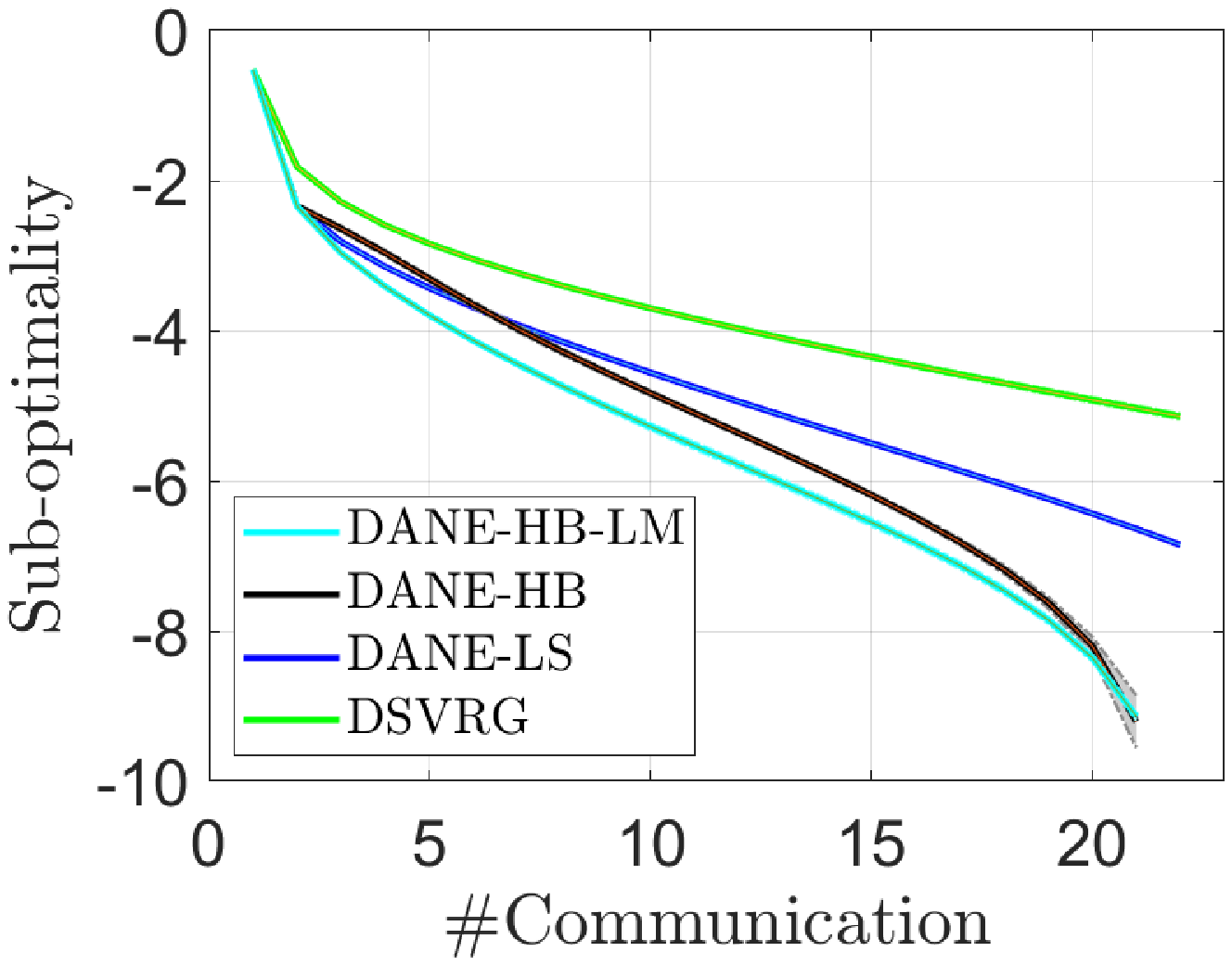}\hspace{-0.13in}
\includegraphics[width=2.1in]{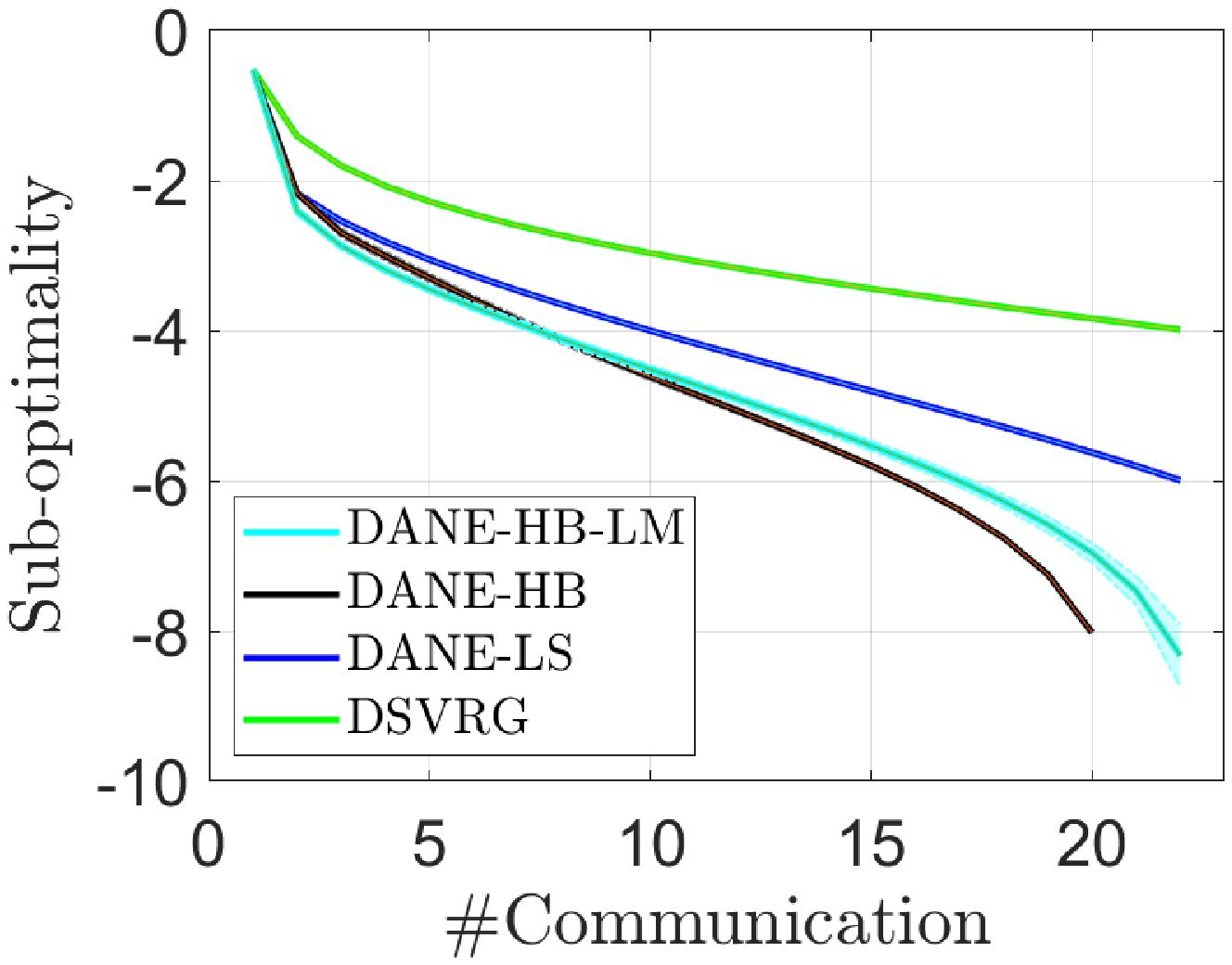}
\label{fig:algorithm_real_rcv1_add_other}}
}
\end{center}
\vspace{-0.2in}
\caption{Algorithm evaluation with comparison to other methods: the objective value evolving curves on synthetic and real logistic regression tasks with $m=4$ (left panels), $m=16$ (middle panels) and $m=32$ (right panels).  Best viewed in color.}
\label{fig:algorithm_non_dane}
\end{figure}

\subsubsection{Comparison against other methods beyond DANE}

In this group of evaluation, we compare the performance of DANE-LS and DANE-HB/DANE-HB-LM with DSVRG~\citep{lee2017distributed}, DiSCO~\citep{zhang2015disco} and GIANT~\citep{wang2018giant} which are, among others, three representative first-order and second-order algorithms for communication-efficient distributed learning. For the re-implementation of GIANT, we follow~\citep{wang2018giant} to add a backtracking line search step to ensure global convergence, although the theoretical guarantee of GIANT does not apply to such a practical implementation. The evaluation is conducted on the same data sets as used in the previous experiment, and the results are shown in Figure~\ref{fig:algorithm_non_dane}. Note that the results of DiSCO and GIANT on \texttt{rcv1.binary} are not available due to its failure of loading the local Hessian matrix ($\sim 16.6$ G) to the 16G SDRM of our evaluation system. Below we summarize the main observations that can be made from these results:
\begin{itemize}
  \item Results on synthetic data: DANE-HB-LM $\ge$ DANE-HB $\ge$ DiSCO $\ge$ DANE-LS $\ge$ DSVRG $\ge$ GIANT. As shown in Figure~\ref{fig:algorithm_synthetic_add_other}, DANE-HB and DiSCO outperform the other considered algorithms when relatively small $m=4$ number of machines is used. For relatively large $m=16, 32$, DANE-HB-LM, DANE-HB and DiSCO converge faster than the other methods. In most cases, DANE-LS plays moderately among all in communication efficiency.
  \item Results on \texttt{gisette}: DANE-HB-LM $>$ DiSCO $\ge$ DANE-HB $\approx$ DANE-LS $\approx$ DSVRG $>$ GIANT. From the curves in Figure~\ref{fig:algorithm_real_gistte_add_other} we can see that DSVRG is comparable to DANE-LS and DANE-HB and they are slightly inferior to DiSCO and DANE-HB-LM. Equipped with line-search, GIANT converges smoothly but at the slowest rate among the considered algorithms.
  \item Results on \texttt{rcv1.binary}: DANE-HB $\ge$ DANE-HB-LM $\ge$ DANE-LS $>$ DSVRG. Figure~\ref{fig:algorithm_real_rcv1_add_other} shows that our proposed DANE-type methods outperform DSVRG with a clear margin on this data set.
\end{itemize}

Overall, DANE-HB and DANE-HB-LM are top two solvers among all the considered algorithms. When applicable, DiSCO is found to be competitive to DANE-HB but inferior to DANE-HB-LM. In many cases, DANE-LS and DSVRG are comparable and they tend to outperform GIANT on real data.

\section{Conclusions}
\label{sect:conclusion}

In this paper, we made progress towards deeply understanding the mysterious convergence behavior of DANE for both quadratic and non-quadratic convex functions. To this end, we proposed two new alternatives, DANE-LS and DANE-HB, which are more suitable for global asymptotic and local non-asymptotic analysis, and yet effective for momentum acceleration.
The core messages conveyed by our study are:
\begin{itemize}
  \item[(1)] \textbf{The plain DANE method can actually converge faster than already known.} For quadratic problems, even without any momentum acceleration, DANE-LS attains a tighter communication complexity bound than the already discovered for plain DANE;
  \item[(2)] \textbf{Line search is beneficial to DANE.} For non-quadratic strongly convex functions, with the blessing of backtracking line search under Armijo rule, DANE-LS converges globally under a wider spectrum of $\gamma$ than DANE, with an appealing local non-asymptotic convergence rate;
  \item[(3)] \textbf{Heavy-ball acceleration is effective for DANE.} DANE-HB possesses a nearly tight communication complexity bound for quadratic objective functions. Whilst for non-quadratic convex functions, DANE-HB exhibits the same bound in the vicinity of minimizer. For learning with linear models, DANE-HB-LM can be shown to have global convergence with favorable communication complexity bounds.
\end{itemize}
Numerical results support our theoretical findings and confirm that DANE-LS and DANE-HB (DANE-HB-LM) are safe and in many cases more attractive alternatives to the prior DANE-type methods for communication-efficient distributed machine learning. We expect that the theory and algorithms developed in this article will fuel future investigation on non-convex distributed optimization problems such as distributed training of deep neural nets. Also, we hope our improved DANE-type methods will have practical implications in large-scale federated optimization for privacy-preserving collaborative machine learning.

\section*{Acknowledgements}

Xiao-Tong Yuan is partially supported by Natural Science Foundation of China (NSFC) under Grant 61876090.

\appendix

\makeatletter
\renewcommand\theequation{A.\@arabic\c@equation }
\makeatother \setcounter{equation}{0}

\section{Some Auxiliary Lemmas}
\label{append:notation_auxiliary_lemmas}

Here we introduce auxiliary lemmas which will be used for proving the results in the manuscript. For the sake of readability, we defer the proofs of some lemmas into Appendix~\ref{ProofforAuxiliaryLemmas}. The following elementary lemma will be used frequently throughout our analysis.
\begin{lemma}\label{lemma:precondition_bound}
Let $A$ and $B$ be two symmetric and positive definite matrices and $B\succeq \mu I$ for some $\mu>0$. If $\|A-B\|\le \gamma$, then $(A+\gamma I)^{-1}B$ is diagonalizable and
\[
\lambda_{\max} (A+\gamma I)^{-1}B \le 1, \quad \lambda_{\min}((A+\gamma I)^{-1}B) \ge \frac{\mu}{\mu + 2\gamma}.
\]
Moreover, the following spectral norm bound holds:
\[
\|I - B^{1/2}(A+\gamma I)^{-1} B^{1/2}\| \le  \frac{2\gamma}{\mu + 2\gamma}.
\]
\end{lemma}

Let us denote $\rho(A)$ the spectral radius of $A$, i.e., the largest (in magnitude) eigenvalue of a square matrix $A$.
\begin{lemma}\label{lemma:spectral_radius_bound}
Let $A\in \mathbb{R}^{d\times d}$ be a square matrix with positive real eigenvalues such that $0<\mu\le\lambda_{\min}(A)\le \lambda_{\max}(A)\le L$. Assume that $A$ is diagonalizable. Then
\[
\rho \left(\left[\begin{array}{*{20}{c}}
   (1+\beta) I -\eta A & -\beta I \\
   I & 0 \\
\end{array} \right]\right) \le \max\{|1-\sqrt{\eta \mu}|, |1-\sqrt{\eta L}|\},
\]
where $\beta = \max\{|1-\sqrt{\eta \mu}|, |1-\sqrt{\eta L}|\}^2$.
\end{lemma}

An important relationship between the spectral norm $\|A\|$ and spectral radius $\rho(A)$ is given by the equality $\rho(A) = \lim_{t\rightarrow \infty}\|A^t\|^{1/t}$, which implies the following classic lemma.
\begin{lemma}\label{lemma:special_radius_norm}
For $\lim_{t\rightarrow \infty}A^t=0$ it is necessary and sufficient that $\rho(A)<1$ and for every $\delta>0$ there exists a constant $c=c(\delta)$ such that
\[
\|A^t\|\le c(\rho(A)+\delta)^t
\]
for all integers $t$.
\end{lemma}

The following lemma is standard and will be used in many places of analysis.
\begin{lemma}\label{lemma:LH}
Assume that function $g$ has $\nu$-LH. Then
\[
\left\|\Delta g(w,w')\right\| \le \frac{\nu}{2} \|w - w'\|^2,
\]
where $\Delta g(w,w'):= \nabla g(w) - \nabla g(w') - \nabla^2 g(w')(w-w')$.
\end{lemma}

The following lemma is useful in our analysis.
\begin{lemma}\label{lemma:grad_bound}
Assume that $F$ and $F_1$ have Lipschitz continuous Hessian. If $\sup_{w}\|\nabla^2 F_1(w) - \nabla^2 F(w)\|\le \gamma$, then at any time instant $t$ it is true that
\[
\|\nabla F(\tilde w^{(t)})\| \le 2\gamma \|\tilde w^{(t)} - w^{(t-1)}\|+ \varepsilon_t, \quad \|\tilde w^{(t)} - w^*\|\le \frac{2\gamma}{\mu} \|\tilde w^{(t)} - w^{(t-1)}\|+ \frac{\varepsilon_t}{\mu}.
\]
\end{lemma}

The next lemma, which is based on a matrix concentration bound~\citep{tropp2012user}, shows that the Hessian of $F_1(w)$ is close to that of $F(w)$ when the sample size is sufficiently large. The same result appears in~\citep{shamir2014communication}.
\begin{lemma}\label{lemma:hessian_close}
Assume that $\|\nabla^2 f(w^\top x_{i}, y_{i})\|\le L$ holds for all $i\in [N]$. Let $H(w) = \nabla^2 F(w)$ and $H_1(w)=\nabla^2 F_1(w)$. Then for each fixed $w$, with probability at least $1-\delta$ over the samples drawn to construct $F_1(w)$, the following bound holds:
\[
\|H_1(w) - H(w)\| \le \sqrt{\frac{32L^2\log(p/\delta)}{n}}.
\]
\end{lemma}

\section{Proofs for Section~\ref{sect:dane_ls_analysis}}
\label{append:dane_ls_analysis}

We collect in this appendix section the technical proofs of the results in Section~\ref{sect:dane_ls_analysis} of the main paper, including Theorems~\ref{thrm:quadratic_dane}, Theorem~\ref{thrm:global_general} and Theorem~\ref{thrm:local_general}, and their corollaries.

\subsection{Proof of Theorem~\ref{thrm:quadratic_dane}}
\label{append:proof_of_dane_ls_quadratic}

In this appendix subsection, we prove Theorem~\ref{thrm:quadratic_dane} as restated in below.
\DANELSQuadratic*

\begin{proof}[of Theorem~\ref{thrm:quadratic_dane}]
Since the objective is quadratic, for any $w^{(t-1)}$ the optimal solution $w^* = \argmin_w F(w)$ can always be expressed as
\[
w^* = w^{(t-1)} - H^{-1} \nabla F(w^{(t-1)}).
\]
Since $H_1^{(t)}\equiv H_1$ holds in the quadratic case, from the definition of $w^{(t)}$ and the gradient equation of $P^{(t-1)}$ we have
\[
w^{(t)} = w^{(t-1)} - (H_1 + \gamma I)^{-1} \nabla F(w^{(t-1)}) + (H_1 + \gamma I)^{-1} \nabla P^{(t-1)}(w^{(t)}).
\]
By combining the above two inequalities we obtain
\begin{equation}\label{equat:proof_quadratic_thrm_key_3}
w^{(t)} - w^* = (I - \eta(H_1 + \gamma I)^{-1} H ) (w^{(t-1)} - w^*) + (H_1 + \gamma I)^{-1} \nabla P^{(t-1)}(w^{(t)}).
\end{equation}
By multiplying $H^{1/2}$ on both sides of the above recurrent form we have
\[
H^{1/2}(w^{(t)} - w^*) = (I - H^{1/2}(H_1 + \gamma I)^{-1} H^{1/2})H^{1/2}(w^{(t-1)} - w^*) + H^{1/2}(H_1 + \gamma I)^{-1} \nabla P^{(t-1)}(w^{(t)})
\]
Let $u^{(t)}=H^{1/2}(w^{(t)} - w^*)$. Based on the basic inequality $\|Tx\|\le \|T\|\|x\|$ we obtain
\[
\begin{aligned}
&\|u^{(t)}\| \\
\le& \|I - H^{1/2}(H_1 + \gamma I)^{-1} H^{1/2}\| \|u^{(t-1)}\| + \|H^{1/2}(H_1 + \gamma I)^{-1}H^{1/2}\|\|H^{-1/2}\nabla P^{(t-1)}(w^{(t)})\| \\
\overset{\zeta_1}{\le}& \frac{2\gamma}{\mu + 2\gamma} \|u^{(t-1)}\| + \frac{\varepsilon_t}{\sqrt{\mu}}\\
\overset{\zeta_2}{\le}& \left( 1- \frac{\mu}{\mu + 2\gamma}\right) \|u^{(t-1)}\| + \frac{\mu}{2(\mu + 2\gamma)}\|u^{(t-1)}\| = \left( 1- \frac{\mu}{2(\mu + 2\gamma)}\right) \|u^{(t-1)}\|,
\end{aligned}
\]
where in the inequality ``$\zeta_1$'' we have used Lemma~\ref{lemma:precondition_bound} and $\|H^{1/2}(H_1 + \gamma I)^{-1}H^{1/2}\|\le 1$ which are valid in view of $\|H_1 - H\|\le\gamma$ and $H\succeq \mu I$, ``$\zeta_2$'' follows from the condition $\varepsilon_t \le \frac{\mu^2\|\nabla F(w^{(t-1)})\|}{2(\mu +2\gamma)L}$ which implies $\frac{\varepsilon_t}{\sqrt{\mu}} \le \frac{\mu\sqrt{\mu} \|w^{(t-1)} - w^*\|}{2(\mu+2\gamma)}\le \frac{\mu\|u^{(t-1)}\|}{2(\mu+2\gamma)}$. The above inequality then leads to
\[
\begin{aligned}
\|w^{(t)} - w^*\| \le& \frac{1}{\sqrt{\mu}}\|u^{(t)}\| \le \frac{1}{\sqrt{\mu}}\left(1-\frac{\mu}{2(\mu+2\gamma)}\right)^t \|u^{(0)}\| \le \sqrt{\frac{L}{\mu}} \left(1-\frac{\mu}{2(\mu+2\gamma)}\right)^t \|w^{(0)} - w^*\|.
\end{aligned}
\]
By applying the basic fact $(1-x)^t \le \exp{\{-xt\}}$ we can show that $\|w^{(t)} - w^*\| \le \epsilon$ is valid when
\[
t \ge \frac{2(\mu+2\gamma)}{\mu} \log \left(\frac{\sqrt{L}\|w^{(0)} - w^*\|}{\sqrt{\mu}\epsilon}\right).
\]
This concludes the proof.
\end{proof}

We further prove Corollary~\ref{corol:quadratic_dane} as restated in below.
\DANELSQuadraticCorol*
\begin{proof}
Since $H(w)\equiv H$ and $H_1(w)\equiv H_1$ in the quadratic case, we know from Lemma~\ref{lemma:hessian_close} that $\|H_1 - H\|\le \gamma=L\sqrt{\frac{32\log(p/\delta)}{n}}$ holds with probability at least $1-\delta$. By invoking Theorem~\ref{thrm:quadratic_dane} we obtain the desired bound.
\end{proof}

\subsection{Proof of Theorem~\ref{thrm:global_general}}
\label{append:proof_of_dane_ls_global}

We provide in this appendix subsection a detailed proof of Theorem~\ref{thrm:global_general} as restated below.
\DANELSGlobal*

As a key step, we first need to prove the following restated Lemma~\ref{lemma:ls_key}.
\DANELSFeasibility*
\begin{proof}
Let us define
\begin{equation}\label{inequat:ls_proof_key2}
r^{(t)} =\nabla P^{(t-1)}(\tilde w^{(t)} )= \nabla F_1(\tilde w^{(t)}) + \nabla F(w^{(t-1)}) -  \nabla F_1(w^{(t-1)}) + \gamma (\tilde w^{(t)} - w^{(t-1)}).
\end{equation}
From the definition of $\tilde w^{(t)}$ we have that $\|r^{(t)} \| \le \varepsilon_t$. Since $F(w)$ is $L$-smooth, we have
\[
\begin{aligned}
&F(w^{(t)}) \\
\le& F(w^{(t-1)}) + \langle\nabla F(w^{(t-1)}), w^{(t)} - w^{(t-1)} \rangle + \frac{L}{2}\|w^{(t)} - w^{(t-1)}\|^2 \\
=& F(w^{(t-1)}) + \eta_t \langle\nabla F(w^{(t-1)}), \tilde w^{(t)} - w^{(t-1)} \rangle + \frac{L\eta_t^2}{2}\|\tilde w^{(t)} - w^{(t-1)}\|^2 \\
\overset{\zeta_1}{\le} & F(w^{(t-1)}) - \eta_t \langle \nabla F_1(\tilde w^{(t)}) -  \nabla F_1(w^{(t-1)}) + \gamma (\tilde w^{(t)} - w^{(t-1)}), \tilde w^{(t)} - w^{(t-1)}\rangle \\
& + \eta_t \langle r^{(t)},\tilde w^{(t)} - w^{(t-1)}\rangle + \frac{L\eta_t^2}{2}\|\tilde w^{(t)} - w^{(t-1)}\|^2 \\
\overset{\zeta_2}{\le} &  F(w^{(t-1)}) - \left( \eta_t - \frac{L\eta_t^2}{2(\gamma+\mu)}\right)\langle \nabla F_1(\tilde w^{(t)}) -  \nabla F_1(w^{(t-1)}) + \gamma (\tilde w^{(t)} - w^{(t-1)}), \tilde w^{(t)} - w^{(t-1)}\rangle \\
 & + \eta_t \varepsilon_t\|\tilde w^{(t)} - w^{(t-1)}\|,
\end{aligned}
\]
where ``$\zeta_1$'' follows from~\eqref{inequat:ls_proof_key2} and ``$\zeta_2$'' is due to the $\mu$-strong-convexity of $F_1$ which implies $\langle \nabla F_1(\tilde w^{(t)}) -  \nabla F_1(w^{(t-1)}) + \gamma (\tilde w^{(t)} - w^{(t-1)}), \tilde w^{(t)} - w^{(t-1)}\rangle \ge (\mu+\gamma)\|\tilde w^{(t)} - w^{(t-1)}\|^2$. To make a successful global line search, we simply require $-\left(\eta_t-\frac{L\eta_t^2}{2(\gamma + \mu)}\right)  \le - \eta_t\rho$, which obviously can be guaranteed by setting
\[
0 < \eta_t \le \min \left\{1,\frac{2(\gamma+\mu)(1-\rho)}{L}\right\}.
\]
This prove the result in Part(a).

To prove the result in Part(b), we note that the equality~\eqref{inequat:ls_proof_key2} is identical to
\begin{equation}\label{inequat:ls_proof_key1}
\nabla F(w^{(t-1)}) =  - (\nabla^2 F_1(w^{(t-1)}) + \gamma I) (\tilde w^{(t)} - w^{(t-1)}) - \Delta F_1(\tilde w^{(t)}, w^{(t-1)}) + r^{(t)}.
\end{equation}
Then based on the definition of $w^{(t)}$ we can derive that
\[
\begin{aligned}
&\langle\nabla F(w^{(t-1)}), w^{(t)} - w^{(t-1)} \rangle + \frac{1}{2} (w^{(t)} - w^{(t-1)})^\top (\nabla^2 F_1(w^{(t-1)})+\gamma I) (w^{(t)} - w^{(t-1)}) \\
&+ \frac{\nu}{6} \|w^{(t)} - w^{(t-1)}\|^3 \\
=& \eta_t \langle \nabla F(w^{(t-1)}) , \tilde w^{(t)} - w^{(t-1)}\rangle + \frac{\eta_t^2}{2}(\tilde w^{(t)} - w^{(t-1)})^\top(\nabla^2 F_1(w^{(t-1)})+\gamma I)(\tilde w^{(t)} - w^{(t-1)})\\
& + \frac{\nu\eta_t^3}{6} \|\tilde w^{(t)} - w^{(t-1)}\|^3\\
\overset{\zeta_1}{=}& \eta_t \langle \nabla F(w^{(t-1)}) , \tilde w^{(t)} - w^{(t-1)}\rangle - \frac{\eta_t^2}{2}\langle \nabla F(w^{(t-1)}) , \tilde w^{(t)} - w^{(t-1)}\rangle + \frac{\nu\eta_t^3}{6} \|\tilde w^{(t)} - w^{(t-1)}\|^3\\
& - \frac{\eta_t^2}{2} \langle \Delta F_1(\tilde w^{(t)}, w^{(t-1)}) , \tilde w^{(t)} - w^{(t-1)}\rangle + \frac{\eta_t^2}{2}\langle r^{(t)},\tilde w^{(t)} - w^{(t-1)}\rangle \\
\overset{\zeta_2}{\le}&  (\eta_t -\frac{\eta_t^2}{2})\langle \nabla F(w^{(t-1)}) , \tilde w^{(t)} - w^{(t-1)}\rangle  +  \left(\frac{\nu\eta_t^2}{4} +\frac{\nu \eta_t^3}{6}\right) \|\tilde w^{(t)} - w^{(t-1)}\|^3 \\
&+ \frac{\eta_t^2}{2}\langle r^{(t)},\tilde w^{(t)} - w^{(t-1)}\rangle \\
\overset{\zeta_3}{=}& -(\eta_t-\frac{\eta_t^2}{2}) \langle \nabla F_1(\tilde w^{(t)}) -  \nabla F_1(w^{(t-1)}) + \gamma (\tilde w^{(t)} - w^{(t-1)}), \tilde w^{(t)} - w^{(t-1)}\rangle \\
&+ \left(\frac{\nu\eta_t^2}{4} +\frac{\nu \eta_t^3}{6}\right) \|\tilde w^{(t)} - w^{(t-1)}\|^3 + \eta_t \langle r^{(t)},\tilde w^{(t)} - w^{(t-1)}\rangle\\
\le& -(\eta_t-\frac{\eta_t^2}{2}) \langle \nabla F_1(\tilde w^{(t)}) -  \nabla F_1(w^{(t-1)}) + \gamma (\tilde w^{(t)} - w^{(t-1)}), \tilde w^{(t)} - w^{(t-1)}\rangle \\
&+ \left(\frac{\nu\eta_t^2}{4} +\frac{\nu \eta_t^3}{6}\right)D \|\tilde w^{(t)} - w^{(t-1)}\|^2 + \eta_t \varepsilon_t\|\tilde w^{(t)} - w^{(t-1)}\|\\
\overset{\zeta_4}{\le}& \left(-\left(\eta_t-\frac{\eta_t^2}{2}\right) + \left(\frac{\nu\eta_t^2}{4} +\frac{\nu \eta_t^3}{6}\right)\frac{D}{\gamma+\mu} \right)\times \\
&\langle \nabla F_1(\tilde w^{(t)}) -  \nabla F_1(w^{(t-1)}) + \gamma (\tilde w^{(t)} - w^{(t-1)}), \tilde w^{(t)} - w^{(t-1)}\rangle + \eta_t \varepsilon_t\|\tilde w^{(t)} - w^{(t-1)}\|,
\end{aligned}
\]
where ``$\zeta_1$'' follows from~\eqref{inequat:ls_proof_key1}, ``$\zeta_2$'' uses $\| \Delta \tilde F(w^{(t-1)},\tilde w^{(t)})\|\le \frac{\nu}{2} \|\tilde w^{(t)} - w^{(t-1)}\|^2$, ``$\zeta_3$'' follows from~\eqref{inequat:ls_proof_key2} and ``$\zeta_4$'' is due to the $\mu$-strong-convexity of $F_1$ which implies $\langle \nabla F_1(\tilde w^{(t)}) -  \nabla F_1(w^{(t-1)}) + \gamma (\tilde w^{(t)} - w^{(t-1)}), \tilde w^{(t)} - w^{(t-1)}\rangle \ge (\mu+\gamma)\|\tilde w^{(t)} - w^{(t-1)}\|^2$. To make a successful line search, we simply require the following bound to hold:
\[
-\left(\eta_t-\frac{\eta_t^2}{2}\right) + \left(\frac{\nu\eta_t^2}{4} +\frac{\nu \eta_t^3}{6}\right)\frac{D}{\gamma+\mu} \le - \eta_t\rho
\]
which indeed can be guaranteed by setting
\[
0 < \eta_t \le \min \left\{1,\frac{-(3\nu D + 6(\gamma+\mu)) + \sqrt{(3\nu D + 6(\gamma+\mu)) ^2 + 96(1-\rho)\nu D (\gamma+\mu)}}{4\nu D}\right\}.
\]
This completes the proof of the result in Part(b).
\end{proof}

Now we are in the position to prove the main result in Theorem~\ref{thrm:global_general}.
\begin{proof}[of Theorem~\ref{thrm:global_general}]
Part (a): We first prove the convergence of the objective value sequence. Based on~\eqref{inequat:ls_proof_key2}, the smoothness of $F_1$ and the condition $\varepsilon_t\le \frac{\rho(\mu+\gamma)}{2(L+\gamma)+\rho(\mu+\gamma)}\|\nabla F(w^{(t-1)})\|$ we can show that
\[
\begin{aligned}
\varepsilon_t \ge& \|r_t\| \ge \|\nabla F(w^{(t-1)})\| - \|\nabla F_1(\tilde w^{(t)}) -  \nabla F_1(w^{(t-1)}) + \gamma (\tilde w^{(t)} - w^{(t-1)})\| \\
\ge& \left(\frac{2(L+\gamma)}{\rho(\mu+\gamma)}+1\right) \varepsilon_t - (L+\gamma)\|\tilde w^{(t)} - w^{(t-1)}\|,
\end{aligned}
\]
which then implies the following bound
\begin{equation}\label{inequat:eps_t_bound_global}
\varepsilon_t \le \frac{\rho(\gamma+\mu)}{2} \|\tilde w^{(t)} - w^{(t-1)}\|.
\end{equation}
Since $F(w)$ is $L$-smooth and $F_1(w)$ is $\mu$-strongly convex, from the first part of Lemma~\ref{lemma:ls_key} we know that the global line search is feasible at each step of iteration and thus
\[
\begin{aligned}
F(w^{(t)})\le & F(w^{(t-1)})-\eta_t \rho \langle \nabla F_1(\tilde w^{(t)}) -  \nabla F_1(w^{(t-1)}) + \gamma (\tilde w^{(t)} - w^{(t-1)}), \tilde w^{(t)} - w^{(t-1)}\rangle \\
&+ \eta_t \varepsilon_t\|\tilde w^{(t)} - w^{(t-1)}\| \\
\overset{\zeta_1}{\le}& F(w^{(t-1)}) - \eta_t\rho(\gamma+\mu)\|\tilde w^{(t)} - w^{(t-1)}\|^2 + \frac{\eta_t \rho(\gamma+\mu)}{2}\|\tilde w^{(t)} - w^{(t-1)}\|^2  \\
=& F(w^{(t-1)}) - \frac{\eta_t \rho(\gamma + \mu)}{2}\|\tilde w^{(t)} - w^{(t-1)}\|^2,
\end{aligned}
\]
where in ``$\zeta_1$'' we have used the bound~\eqref{inequat:eps_t_bound_global}. From Lemma~\ref{lemma:grad_bound} we know that $\|\tilde w^{(t)} - w^{(t-1)}\|\neq 0$ uncles $\tilde w^{(t)}$ admits a global minimizer of $F$. Then based on the above inequality the sequence $\{F(w^{(t)})\}$ is decreasing. Since $F(w^{(t)})\ge F(w^*)>-\infty$, it must hold that $\{F(w^{(t)})\}$ converges. Also from the above inequality we have
\[
\eta_t\rho(\gamma+\mu)\|\tilde w^{(t)} - w^{(t-1)}\|^2\le 2(F(w^{(t-1)}) - F(w^{(t)})),
\]
which implies $\|\tilde w^{(t)} - w^{(t-1)}\|\rightarrow 0$ as $t\rightarrow \infty$.

Proof of part(b): Since $F(w)$ has $\nu$-smooth, we have
\[
\begin{aligned}
&F(w^{(t)})\\
\le& F(w^{(t-1)}) + \langle \nabla F(w^{(t-1)}) ,  w^{(t)} - w^{(t-1)}\rangle + \frac{1}{2}( w^{(t)} - w^{(t-1)})^\top\nabla^2 F(w^{(t-1)})( w^{(t)} - w^{(t-1)}) \\
&+ \frac{\nu}{6} \| w^{(t)} - w^{(t-1)}\|^3\\
\overset{\zeta_1}{\le}& F(w^{(t-1)}) + \langle\nabla F(w^{(t-1)}), w^{(t)} - w^{(t-1)} \rangle \\
 &+ \frac{1}{2} (w^{(t)} - w^{(t-1)})^\top (\nabla^2 F_1(w^{(t-1)})+\gamma I) (w^{(t)} - w^{(t-1)}) + \frac{\nu}{6} \|w^{(t)} - w^{(t-1)}\|^3 \\
\overset{\zeta_2}{\le} & F(w^{(t-1)})-\eta_t \rho \langle \nabla F_1(\tilde w^{(t)}) -  \nabla F_1(w^{(t-1)}) + \gamma (\tilde w^{(t)} - w^{(t-1)}), \tilde w^{(t)} - w^{(t-1)}\rangle\\
&+ \eta_t \varepsilon_t\|\tilde w^{(t)} - w^{(t-1)}\| \\
\overset{\zeta_3}{\le} & F(w^{(t-1)}) - \eta_t\rho(\gamma+\mu)\|\tilde w^{(t)} - w^{(t-1)}\|^2+ \frac{\eta_t \rho(\gamma+\mu)}{2}\|\tilde w^{(t)} - w^{(t-1)}\|^2 \\
=& F(w^{(t-1)}) - \frac{\eta_t \rho(\gamma + \mu)}{2}\|\tilde w^{(t)} - w^{(t-1)}\|^2,
\end{aligned}
\]
where ``$\zeta_1$'' follows from $\|\nabla^2 F_1(w^{(t-1)}) -\nabla^2 F(w^{(t-1)})\| \le \gamma$ such that $\nabla^2 F_1(w^{(t-1)}) - \nabla^2 F(w^{(t-1)})  + \gamma I \succeq 0$, in ``$\zeta_2$'' we have used the second part of Lemma~\ref{lemma:ls_key}, and ``$\zeta_3$'' is due to the bound~\eqref{inequat:eps_t_bound_global}. By using the same argument as in the part(a) we can show that the sequence $\{F(w^{(t)})\}$ converges and $\|\tilde w^{(t)} - w^{(t-1)}\|\rightarrow 0$ as $t\rightarrow \infty$. This completes the proof.
\end{proof}

\subsection{Proof of Theorem~\ref{thrm:local_general}}
\label{append:proof_of_dane_ls_local}

This appendix subsection is devoted to providing a detailed proof of Theorem~\ref{thrm:local_general} as restated in below.
\DANELSNonasymptotic*

To prove the theorem, we first need to prove the following restated Lemma~\ref{lemma:unit_length}.
\DANELSUnit*
\begin{proof}
Since $F(w)$ has $\nu$-LH, it holds that
\[
\begin{aligned}
&F(w^{(t)})\\
\le& F(w^{(t-1)}) + \langle \nabla F(w^{(t-1)}) ,  w^{(t)} - w^{(t-1)}\rangle + \frac{1}{2} ( w^{(t)} - w^{(t-1)})^\top\nabla^2 F(w^{(t-1)})( w^{(t)} - w^{(t-1)}) \\
&+ \frac{\nu}{6} \| w^{(t)} - w^{(t-1)}\|^3\\
\le& F(w^{(t-1)}) + \langle\nabla F(w^{(t-1)}), w^{(t)} - w^{(t-1)} \rangle\\
& + \frac{1}{2} (w^{(t)} - w^{(t-1)})^\top (\nabla^2 F_1(w^{(t-1)})+\gamma I) (w^{(t)} - w^{(t-1)}) + \frac{\nu}{6} \|w^{(t)} - w^{(t-1)}\|^3,
\end{aligned}
\]
where in the last inequality we have used $\|\nabla^2 F_1(w^{(t-1)}) -\nabla^2 F(w^{(t-1)})\| \le \gamma$. Based on the above inequality, it is sufficient to prove
\[
\begin{aligned}
&\langle\nabla F(w^{(t-1)}), w^{(t)} - w^{(t-1)} \rangle + \frac{1}{2} (w^{(t)} - w^{(t-1)})^\top (\nabla^2 F_1(w^{(t-1)})+\gamma I) (w^{(t)} - w^{(t-1)}) \\
& + \frac{\nu}{6} \|w^{(t)} - w^{(t-1)}\|^3 \\
\le& -\frac{1}{3}  \langle \nabla F_1(\tilde w^{(t)}) -  \nabla F_1(w^{(t-1)}) + \gamma (\tilde w^{(t)} - w^{(t-1)}), \tilde w^{(t)} - w^{(t-1)}\rangle,
\end{aligned}
\]
To this end, by mimicking the arguments in the proof of Lemma~\ref{lemma:ls_key} we can show that
\[
\begin{aligned}
&\langle\nabla F(w^{(t-1)}), w^{(t)} - w^{(t-1)} \rangle + \frac{1}{2} (w^{(t)} - w^{(t-1)})^\top (\nabla^2 F_1(w^{(t-1)})+\gamma I) (w^{(t)} - w^{(t-1)}) \\
 &+ \frac{\nu}{6} \|w^{(t)} - w^{(t-1)}\|^3 \\
\le& -\left(\eta_t-\frac{\eta_t^2}{2}\right) \langle \nabla F_1(\tilde w^{(t)}) -  \nabla F_1(w^{(t-1)}) + \gamma (\tilde w^{(t)} - w^{(t-1)}), \tilde w^{(t)} - w^{(t-1)}\rangle \\
&+ \left(\frac{\nu\eta_t^2}{4} +\frac{\nu \eta_t^3}{6}\right) \|\tilde w^{(t)} - w^{(t-1)}\|^3 + \eta_t \varepsilon_t\|\tilde w^{(t)} - w^{(t-1)}\| \\
\overset{\zeta_1}{\le}& -\left(\eta_t-\frac{\eta_t^2}{2}\right) \langle \nabla F_1(\tilde w^{(t)}) -  \nabla F_1(w^{(t-1)}) + \gamma (\tilde w^{(t)} - w^{(t-1)}), \tilde w^{(t)} - w^{(t-1)}\rangle \\
& + \frac{1}{\mu+\gamma}\left(\frac{\nu\eta_t^2}{4} +\frac{\nu \eta_t^3}{6}\right)\|\tilde w^{(t)} - w^{(t-1)}\|  \langle \nabla F_1(\tilde w^{(t)}) -  \nabla F_1(w^{(t-1)}) \\
&+ \gamma (\tilde w^{(t)} - w^{(t-1)}), \tilde w^{(t)} - w^{(t-1)}\rangle + \eta_t \varepsilon_t\|\tilde w^{(t)} - w^{(t-1)}\| \\
& \left(-\left(\eta_t-\frac{\eta_t^2}{2}\right) + \frac{5\nu\|\tilde w^{(t)} - w^{(t-1)}\|}{12(\gamma+\mu)} \right)\langle \nabla F_1(\tilde w^{(t)}) -  \nabla F_1(w^{(t-1)}) \\
 &+ \gamma (\tilde w^{(t)} - w^{(t-1)}), \tilde w^{(t)} - w^{(t-1)}\rangle + \eta_t \varepsilon_t\|\tilde w^{(t)} - w^{(t-1)}\|,
\end{aligned}
\]
where ``$\zeta_1$'' is due to $\langle \nabla F_1(\tilde w^{(t)}) -  \nabla F_1(w^{(t-1)}) + \gamma (\tilde w^{(t)} - w^{(t-1)}), \tilde w^{(t)} - w^{(t-1)}\rangle \ge (\mu+\gamma)\|\tilde w^{(t)} - w^{(t-1)}\|^2$ and in the last inequality we have used $\eta_t\le 1$. When $t$ is sufficiently large, from Theorem~\ref{thrm:global_general} we know that $\|\tilde w^{(t)} - w^{(t-1)}\|$ will be sufficiently close to zero so that $\frac{5\nu\|\tilde w^{(t)} - w^{(t-1)}\|}{12(\gamma+\mu)} \le \frac{1}{6}$. Consider $\eta_t=1$ in the above inequality. Then
\[
\begin{aligned}
&\langle\nabla F(w^{(t-1)}), w^{(t)} - w^{(t-1)} \rangle + \frac{1}{2} (w^{(t)} - w^{(t-1)})^\top (\nabla^2 F_1(w^{(t-1)})+\gamma I) (w^{(t)} - w^{(t-1)}) \\
 &+ \frac{\nu}{6} \|w^{(t)} - w^{(t-1)}\|^3 \\
\le& -\frac{1}{3}  \langle \nabla F_1(\tilde w^{(t)}) -  \nabla F_1(w^{(t-1)}) + \gamma (\tilde w^{(t)} - w^{(t-1)}), \tilde w^{(t)} - w^{(t-1)}\rangle + \varepsilon_t\|\tilde w^{(t)} - w^{(t-1)}\|,
\end{aligned}
\]
which implies that unit length is acceptable for any $\rho\in (0,1/3)$.
\end{proof}

We also need the following restated Lemma~\ref{lemma:local_unit} which establishes the local convergence rate of Algorithm~\ref{alg:dane_ls} when $\eta_t\equiv1$, i.e., the unit length is always accepted by the backtracking line search.
\DANELSLocalrate*
\begin{proof}
Since $\eta_t=1$, we have $w^{(t)} = \tilde w^{(t)}$. By using the first-order optimality condition $\nabla F(w^*)=0$. We can show that
\[
\begin{aligned}
&\nabla P^{(t-1)}(w^{(t)}) \\
=&\nabla F_1(w^{(t)}) -  \nabla F_1(w^{(t-1)}) + \nabla F(w^{(t-1)}) + \gamma (w^{(t)} - w^{(t-1)})\\
=&\nabla F_1(w^{(t)}) - \nabla F_1(w^*) + \nabla  F_1(w^*) - \nabla F_1(w^{(t-1)}) + \nabla F(w^{(t-1)}) - \nabla F(w^*) + \gamma (w^{(t)} - w^{(t-1)})\\
=& \Delta F_1(w^{(t)}, w^*) + \nabla^2 F_1(w^*)(w^{(t)}-w^*) - \Delta F_1(w^{(t-1)}, w^*) - \nabla^2 F_1(w^*)(w^{(t-1)}-w^*)\\
&+ \Delta F(w^{(t-1)}, w^*) + \nabla^2 F(w^*)(w^{(t-1)}-w^*) +  \gamma (w^{(t)} - w^{(t-1)})\\
=&  \Delta F_1(w^{(t)}, w^*) + (\nabla^2 F_1(w^*)+\gamma I)(w^{(t)}-w^*) - \Delta F_1(w^{(t-1)}, w^*) - (\nabla^2 F_1(w^*)+\gamma I)(w^{(t-1)}-w^*)\\
&+ \Delta F(w^{(t-1)}, w^*) + \nabla^2 F(w^*)(w^{(t-1)}-w^*).
\end{aligned}
\]
By multiplying $(\nabla^2 F_1 (w^*)+\gamma I)^{-1}$ on both sides of the above and after proper rearrangement we obtain
\[
\begin{aligned}
&w^{(t)} - w^* \\
=& \left(I - (\nabla^2 F_1(w^*)+\gamma I)^{-1}\nabla^2 F(w^*)\right)(w^{(t-1)} - w^*) \\
&+  (\nabla^2 F_1(w^*)+\gamma I)^{-1} \left( \nabla P^{(t-1)}(w^{(t)}) -  \Delta F_1(w^{(t)}, w^*) +  \Delta F_1(w^{(t-1)}, w^*) - \Delta F(w^{(t-1)}, w^*)\right)\\
=&  \left(I - (\nabla^2 F_1(w^*)+\gamma I)^{-1}\nabla^2 F(w^*)\right)(w^{(t-1)} - w^*) \\
&+  (\nabla^2 F_1(w^*)+\gamma I)^{-1} \left(\Delta F_1(w^{(t-1)}, w^*) - \Delta F(w^{(t-1)}, w^*)-  \Delta F_1(w^{(t)}, w^*)+ \nabla P^{(t-1)}(w^{(t)})\right).
\end{aligned}
\]
Let $H^* = \nabla^2 F(w^*)$ and $H_1^* = \nabla^2 F_1(w^*)$. Similar to the previous analysis, we work on the three term recurrence in matrix form
\begin{equation}\label{equat:iteraiton_form_cov}
u^{(t)} = A u^{(t-1)} + r^{(t-1)}
\end{equation}
where $u^{(t)}:= w^{(t)} - w^*$, $A:= I - (H_1^* + \gamma I)^{-1} H^*$ and
\[
r^{(t-1)}:=(H_1^* + \gamma I)^{-1}\left(\Delta F_1(w^{(t-1)}, w^*) - \Delta F(w^{(t-1)}, w^*)-  \Delta F_1(w^{(t)}, w^*) + \nabla P^{(t-1)}(w^{(t)}) \right).
\]
We next bound $\|r^{(t-1)}\|$ with respect to $\|u^{(t-1)}\|$ and the local optimization precision $\varepsilon_t$.
\begin{equation}\label{inequat:proof_general_cov_key1}
\begin{aligned}
&\|r^{(t-1)}\| \\
\le& \left\|(H_1^* +\gamma I)^{-1}\right\| \left\| \Delta F_1(w^{(t-1)}, w^*) - \Delta F(w^{(t-1)}, w^*) -  \Delta F_1(\tilde w^{(t)}, w^*) \right\|\\
& + \left\|(H_1^* +\gamma I)^{-1}\right\| \|\nabla P^{(t-1)}(w^{(t)})\| \\
\le& \frac{\nu}{2(\gamma+\mu)}\|w^{(t)} - w^*\|^2 + \frac{\nu}{\gamma+\mu}\|w^{(t-1)} - w^*\|^2 + \frac{\varepsilon_t}{\gamma+\mu},
\end{aligned}
\end{equation}
where we have used $H_1^* = \nabla^2 F_1(w^*)\succeq \mu I$ and the Lipschitz Hessian assumption such that $\|\Delta F_1(w^{(t)}, w^*)\|\le \frac{\nu}{2}\|w^{(t)}-w^*\|^2$, $\|\Delta F_1(w^{(t-1)}, w^*)\|\le \frac{\nu}{2}\|w^{(t-1)}-w^*\|^2$ and $\|\Delta F(w^{(t-1)}, w^*)\|\le \frac{\nu}{2}\|w^{(t-1)}-w^*\|^2$, and also $\|\nabla P^{(t-1)}(w^{(t)})\|\le\varepsilon_t$. In the following step we bound $\|w^{(t)}-w^*\|$ with respect to $\|w^{(t-1)}-w^*\|$. Since $\tilde F(w)$ is $\mu$-strongly-convex, $P^{(t-1)}(w)$ is obviously $(\gamma + \mu)$-strongly-convex. Therefore
\begin{equation}\label{inequat:proof_general_cov_key2}
\begin{aligned}
&\|w^{(t)} - w^*\| \\
\le& \frac{1}{\gamma+\mu} \|\nabla P^{(t-1)}(w^{(t)}) - \nabla P^{(t-1)}(w^*)\| \overset{\zeta_1}{=} \frac{1}{\gamma+\mu}\| \nabla P^{(t-1)}(w^*)\| + \frac{\varepsilon_t}{\gamma+\mu}\\
=& \frac{1}{\gamma + \mu} \|\nabla F(w^{(t-1)}) - \nabla F_1(w^{(t-1)}) + \gamma (w^* - w^{(t-1)}) + \nabla F_1(w^*)\| + \frac{\varepsilon_t}{\gamma+\mu} \\
= & \frac{1}{\gamma + \mu}\left\|\left(\nabla F(w^{(t-1)}) - \nabla F_1(w^{(t-1)}) \right) - \left(\nabla F(w^*) - \nabla F_1(w^*) \right) + \gamma(w^* - w^{(t-1)})\right\| \\
&+ \frac{\varepsilon_t}{\gamma+\mu}\\
\le& \frac{2\gamma}{\gamma+\mu}\|w^{(t-1)} - w^*\| + \frac{\varepsilon}{\gamma+\mu}\le 2\|w^{(t-1)} - w^*\| + \frac{\varepsilon_t}{\gamma+\mu},
\end{aligned}
\end{equation}
where ``$\zeta_1$'' follows from the optimality of $w^{(t)} = \tilde w^{(t)}$ with respect to $P^{(t-1)}$ and the last inequality is implied by the assumption $\|\nabla^2 F(w) - \nabla^2 F_1(w)\|\le \gamma$ for all $w$. By combining~\eqref{inequat:proof_general_cov_key1} and~\eqref{inequat:proof_general_cov_key2}, and using the basic inequality $(a+b)^2\le 2a^2+2b^2$ we arrive at
\begin{equation}\label{inequat:proof_general_cov_cov_key4}
\begin{aligned}
\|r^{(t-1)}\| \le& \frac{4\nu}{\gamma+\mu}\|w^{(t-1)} - w^*\|^2 + \frac{\nu \varepsilon_t^2}{(\gamma + \mu)^3} +\frac{\nu}{\gamma+\mu}\|w^{(t-1)} - w^*\|^2 + \frac{\varepsilon_t}{\gamma+\mu}\\
\overset{\zeta_1}{\le}& \frac{5\nu}{\gamma+\mu}\|u^{(t-1)}\|^2+ \frac{(\nu+1)\varepsilon_t}{\gamma+\mu}\overset{\zeta_2}{\le} \frac{6\nu+1}{\gamma+\mu}\|u^{(t-1)}\|^2 ,
\end{aligned}
\end{equation}
where in the inequality ``$\zeta_1$'' we have used the assumption on $\varepsilon_t$ which implies $\varepsilon_t \le (\gamma+\mu)^2$, and $\zeta_2$ follows from $\varepsilon_t \le \frac{\|\nabla F(w^{(t-1)})\|^2}{L^2}\le \|w^{(t-1)} - w^*\|^2=\|u^{(t-1)}\|^2$. Since $\|H_1^* - H^*\|\le\gamma$ and $H^*\succeq \mu I$, by applying Lemma~\ref{lemma:precondition_bound} we obtain that
\begin{equation}\label{inequat:proof_general_thrm_cov_key3}
\begin{aligned}
\|A^t\| =& \|(I - (H_1^* + \gamma I)^{-1} H^*)^t\| \\
 =& \left\|\left((H^{*})^{-1/2}(I - (H^{*})^{1/2}(H_1^* + \gamma I)^{-1} (H^{*})^{1/2})(H^{*})^{1/2}\right)^t\right\| \\
 =& \left\|(H^{*})^{-1/2}(I - (H^{*})^{1/2}(H_1^* + \gamma I)^{-1} (H^{*})^{1/2})^t(H^{*})^{1/2}\right\| \\
 \le& \sqrt{\frac{L}{\mu}}\|I - (H^{*})^{1/2}(H_1^* + \gamma I)^{-1} (H^{*})^{1/2}\|^t \le \sqrt{\frac{L}{\mu}} \left(1-\frac{\mu}{\mu+2\gamma}\right)^t.
\end{aligned}
\end{equation}
In the following argument, to simplify notation, we abbreviate
\[
c=\sqrt{\frac{L}{\mu}}, \quad \vartheta = \frac{6\nu+1}{\gamma+\mu}, \quad \rho =1-\frac{\mu}{\mu+2\gamma}
\]
such that $\|A^t\|\le c \rho^t$ and $\|r^{(t)}\| \le \vartheta \|u^{(t)}\|^2 $. Let us consider the following defined integer
\[
\tau= \left\lceil \frac{\mu+2\gamma}{2\mu} \log\left(\frac{4L}{\mu}\right)  \right\rceil
\]
such that $\|A^\tau\|\le \frac{1}{2}$.  We now prove by induction that for any integer $k\ge 0$, $\max_{0\le i\le \tau-1}\|u^{(k\tau+i)}\|\le \frac{1}{4c\tau\vartheta}\left(\frac{3}{4}\right)^k$.  The assumption $\max_{0\le i \le \tau-1}\|w^{(i)} - w^*\| \le \frac{1}{4c\tau\vartheta }$ guarantees that the bound is valid for the case $k=0$, i.e., $\max_{0\le i \le \tau-1}\|u^{(i)}\| \le \frac{1}{4c\tau\vartheta}$. Now assume that $\max_{0\le i\le \tau-1}\|u^{(k\tau+i)}\|\le \left(\frac{3}{4}\right)^k \frac{1}{4c\tau\vartheta}$ for some $k\ge 0$.
By recursively applying~\eqref{equat:iteraiton_form_cov} we obtain
\[
\begin{aligned}
\|u^{((k+1)\tau)}\| &= \left\|A^{\tau} u^{(k\tau)} + \sum_{i=0}^{\tau-1} A^i r^{(k\tau+\tau-1-i)}\right\| \le \|A^\tau\|\|u^{(k\tau)}\| + \sum_{i=0}^{\tau-1} \|A^i\| \|r^{(k\tau+\tau-1-i)}\| \\
&\overset{\zeta_1}{\le}  \frac{1}{2} \|u^{(k\tau)}\|+ \vartheta c\sum_{i=0}^{\tau-1} \|u^{(k\tau + \tau-1-i)}\|^2 \overset{\zeta_2}{\le} \frac{1}{2} \left(\frac{3}{4}\right)^k \frac{1}{4c\tau\vartheta} + \frac{1}{4\tau}\sum_{i=0}^{\tau-1}\|u^{k\tau+i}\| \\
&\overset{\zeta_3}{\le} \frac{1}{2} \left(\frac{3}{4}\right)^k \frac{1}{4c\tau\vartheta} + \frac{1}{4} \left(\frac{3}{4}\right)^k \frac{1}{4c\tau\vartheta} = \left(\frac{3}{4}\right)^{k+1} \frac{1}{4c\tau\vartheta},
\end{aligned}
\]
where ``$\zeta_1$'' is due to~\eqref{inequat:proof_general_thrm_cov_key3} which implies $\|A^i\|\le c$ for all $i\ge 1$ and it also has used $\|r^{(t)}\| \le \vartheta \|u^{(t)}\|^2$, ``$\zeta_2$'' and ``$\zeta_3$'' are based on the induction step and $\|u^{k\tau+i}\|\le \frac{1}{4c\tau\vartheta}$ for all $0\le i \le \tau-1$. By using the same argument as the above, we can show that $\|u^{((k+1)\tau+i)}\|\le \frac{1}{4c\tau\vartheta}\left(\frac{3}{4}\right)^{k+1}$ for all $0\le i \le \tau-1$. This proves that $\max_{0\le i\le \tau-1} \|u^{(k\tau+i)}\|\le \frac{1}{4c\tau\vartheta}\left(\frac{3}{4}\right)^k$ holds for all $k\ge0$. In particularly,
\[
\|w^{(k\tau)} - w^*\|=\|u^{k\tau}\| \le\frac{1}{4c\tau\vartheta}\left(\frac{3}{4}\right)^k.
\]
Therefore, we need $t \ge 4\tau\log \left(\frac{1}{4c\tau\vartheta \epsilon}\right)$ to guarantee the estimation bound $\|w^{(t)} - w^*\|\le\epsilon$. This completes the proof.
\end{proof}

We are now ready to prove the main theorem.

\begin{proof}[of Theorem~\ref{thrm:local_general}]
Under the given conditions, from Theorem~\ref{thrm:global_general} and Lemma~\ref{lemma:unit_length} we know that there exists a sufficiently large $t_0$ such that for all $t\ge t_0$, the unit length $\eta_t=1$ is acceptable with $\rho \in(0,1/3)$ and the following holds:
\begin{equation}\label{inequat:proof_global_local_key_1}
\|\tilde w^{(t)} - w^{(t-1)}\| \le  \frac{6\mu}{13\gamma + \mu}\left(\frac{\gamma+\mu}{4(6\nu+1)\sqrt{\kappa}\tau}\right).
\end{equation}
Since $\varepsilon_t\le \frac{\rho(\mu+\gamma)}{2(L+\gamma)+\rho(\mu+\gamma)}\|\nabla F(w^{(t-1)})\|$, we have that the bound~\eqref{inequat:eps_t_bound_global} holds and thus
\[
\varepsilon_t \le \frac{\rho(\gamma+\mu)}{2} \|\tilde w^{(t)} - w^{(t-1)}\| \le \frac{\gamma+\mu}{6}\|\tilde w^{(t)} - w^{(t-1)}\|,
\]
where we have used $\rho\le 1/3$.
Then based on Lemma~\ref{lemma:grad_bound} and~\eqref{inequat:proof_global_local_key_1}, the following holds for all $t\ge t_0$,
\[
\begin{aligned}
\|w^{(t)} - w^*\|=&\|\tilde w^{(t)} - w^*\| \le \frac{2\gamma}{\mu} \|\tilde w^{(t)} - w^{(t-1)}\|+ \frac{\varepsilon_t}{\mu} \\
\le& \left(\frac{2\gamma}{\mu} + \frac{\gamma+\mu}{6\mu}\right) \|\tilde w^{(t)} - w^{(t-1)}\| \le \frac{\gamma+\mu}{4(6\nu+1)\sqrt{\kappa}\tau}.
\end{aligned}
\]
Given the condition on $\varepsilon_t$, by invoking Lemma~\ref{lemma:local_unit} we obtain $\|w^{(t_0+t_1)} - w^*\| \le \epsilon$ after
\[
t_1 \ge 4\tau \log \left(\frac{\gamma+\mu}{4(6\nu+1)\sqrt{\kappa}\tau}\left(\frac{1}{\epsilon}\right)\right),
\]
where $\tau= \left\lceil \frac{\mu+2\gamma}{2\mu} \log\left(4\kappa\right)  \right\rceil$. This proves the desired bound.
\end{proof}

\section{Proofs for Section~\ref{sect:dane_hb_analysis}}
\label{append:dane_hb_analysis}

We collect in this appendix section the technical proofs of the results in Section~\ref{sect:dane_ls_analysis} of the main paper, including Theorems~\ref{thrm:quadratic_hb}, Theorem~\ref{thrm:local_general_hb}, Theorem~\ref{thrm:local_general} and their corollaries.

\subsection{Proof of Theorem~\ref{thrm:quadratic_hb}}\label{append:proof_of_dane_hb_quadratic}

We now prove Theorem~\ref{thrm:quadratic_hb} which is restated as  follows.
\DANEHBQuadratic*
\begin{proof}[of Theorem~\ref{thrm:quadratic_hb}]
Since the objective is quadratic, for any $w^{(t-1)}$ the optimal solution $w^* = \argmin_w F(w)$ can always be expressed as
\[
w^* = w^{(t-1)} - H^{-1} \nabla F(w^{(t-1)}).
\]
Since $H_1^{(t)}\equiv H_1$ holds in the quadratic case, from the definition of $w^{(t)}$ and the gradient equation of $P^{(t-1)}$ we have
\[
w^{(t)} = w^{(t-1)} - \eta (H_1 + \gamma I)^{-1} \nabla F(w^{(t-1)}) + \beta (w^{(t-1)} - w^{(t-2)}) + r^{(t-1)},
\]
where the residual term $r^{(t-1)}$ is given by
\[
r^{(t-1)} = (H_1 + \gamma I)^{-1} \nabla P^{(t-1)}(\tilde w^{(t)}).
\]
By combining the above two inequalities we obtain
\begin{equation}\label{equat:proof_quadratic_thrm_key_3}
w^{(t)} - w^* = ((1+\beta)I - (H_1 + \gamma I)^{-1} H ) (w^{(t-1)} - w^*) - \beta (w^{(t-2)} - w^*) + r^{(t-1)}.
\end{equation}
Now let us study the three term recurrence in matrix form
\[
\begin{aligned}
&\left[\begin{array}{*{20}{c}}
   w^{(t)} - w^*  \\
   w^{(t-1)} - w^*  \\
\end{array}
\right]\\
=&\left[\begin{array}{*{20}{c}}
   (1+\beta) I - (H_1 + \gamma I)^{-1} H & -\beta I \\
   I & 0 \\
\end{array} \right]\left[\begin{array}{*{20}{c}}
   w^{(t-1)} - w^*  \\
   w^{(t-2)} - w^*  \\
\end{array}
\right] + r^{(t-1)}\\
=& \left[\begin{array}{*{20}{c}}
   (1+\beta) I -(H_1 + \gamma I)^{-1} H & -\beta I \\
   I & 0 \\
\end{array} \right]^t\left[\begin{array}{*{20}{c}}
   w^{(0)} - w^*  \\
   w^{(-1)} - w^*  \\
\end{array}
\right] \\
&+ \sum_{\tau=0}^{t-1}\left[\begin{array}{*{20}{c}}
   (1+\beta) I -(H_1 + \gamma I)^{-1} H & -\beta I \\
   I & 0 \\
\end{array} \right]^{\tau} r^{(t-1-\tau)}.
\end{aligned}
\]
Let us abbreviate $u^{(t)}:=\left[\begin{array}{*{20}{c}}
   w^{(t)} - w^*  \\
   w^{(t-1)} - w^*  \\
\end{array}
\right]$ and $A:=\left[\begin{array}{*{20}{c}}
   (1+\beta) I - (H_1 + \gamma I)^{-1} H & -\beta I \\
   I & 0 \\
\end{array} \right]$. Based on the basic fact $\|Tx\|\le \|T\|\|x\|$ we obtain
\begin{equation}\label{inequat:proof_quadratic_thrm_key_1}
\|u^{(t)}\| \le \|A^t\| \|u^{(0)}\| + \sum_{\tau=0}^{t-1} \|A^{\tau}\|\|r^{(t-1-\tau)}\|.
\end{equation}
Let us now temporarily assume that $\rho(A)<1$ and consider $\delta = \frac{1-\rho(A)}{2}$. From Lemma~\ref{lemma:special_radius_norm} we know that there exists a constant $c=c(\delta)$ such that for all $t\ge 0$:
\begin{equation}\label{inequat:proof_quadratic_thrm_key_2}
\|A^t\| \le c(\rho(A)+\delta)^t = c\left(\frac{1+\rho(A)}{2}\right)^t.
\end{equation}
Next we show that $\rho(A)<1$ is indeed the case under the conditions of the theorem. Since $\|H_1 - H\|\le\gamma$ and $H\succeq \mu I$, by applying Lemma~\ref{lemma:precondition_bound} we obtain that $(H_1 + \gamma I)^{-1}H$ is diagonalizable and
\[
\frac{\mu}{\mu + 2\gamma} \le \lambda_{\min}((H_1 + \gamma I)^{-1}H) \le \lambda_{\max} ((H_1 + \gamma I)^{-1}H) \le 1.
\]
Given the setting of $\beta = \left(1-\sqrt{\frac{\mu}{\mu+2\gamma}}\right)^2$, it is known from Lemma~\ref{lemma:spectral_radius_bound} (with $\eta=1$) that
\[
\rho(A) \le 1 - \sqrt{\frac{\mu}{\mu+2\gamma}}.
\]
Note that $\|r^{(t)}\|\le \varepsilon_t/(\mu + \gamma)$ holds for all $t$ which follows immediately from $\|\nabla P^{(t-1)}(\tilde w^{(t)})\| \le \varepsilon_t$ and $H_1 \succeq \mu I$. Then combining the above bound with~\eqref{inequat:proof_quadratic_thrm_key_1} and~\eqref{inequat:proof_quadratic_thrm_key_2} we obtain
\[
\begin{aligned}
&\|w^{(t)} - w^*\| \\
\le& \|u^{(t)}\| \le c\left(1-\frac{1}{2}\sqrt{\frac{\mu}{\mu+2\gamma}}\right)^t \|u^{(0)}\| + \frac{c}{\mu + \gamma}\sum_{\tau=0}^{t-1}\varepsilon_{t-1-\tau}\left(1-\frac{1}{2}\sqrt{\frac{\mu}{\mu+2\gamma}}\right)^\tau\\
\overset{\zeta_1}{\le}& \sqrt{2} c\left(1-\frac{1}{2}\sqrt{\frac{\mu}{\mu+2\gamma}}\right)^t \|w^{(0)} - w^*\| + \frac{c\sqrt{2}}{2}\sum_{\tau=0}^{t-1}\frac{1}{(t-\tau)^2}\left(1-\frac{1}{2}\sqrt{\frac{\mu}{\mu+2\gamma}}\right)^t \|w^{(0)} - w^*\|\\
\le& 2\sqrt{2} c\left(1-\frac{1}{2}\sqrt{\frac{\mu}{\mu+2\gamma}}\right)^t \|w^{(0)} - w^*\|  ,
\end{aligned}
\]
where in the inequality ``$\zeta_1$'' we have used $w^{(0)} = w^{(-1)}$ and the condition
\[
\begin{aligned}
\varepsilon_t &\le \frac{\sqrt{2}(\mu+\gamma)\|\nabla F(w^{(0)})\|}{2L(t+1)^2}\left(1-\frac{1}{2}\sqrt{\frac{\mu}{\mu+2\gamma}}\right)^{t+1} \\
&\le \frac{\sqrt{2}(\mu+\gamma)\|w^{(0)} - w^*\|}{2(t+1)^2}\left(1-\frac{1}{2}\sqrt{\frac{\mu}{\mu+2\gamma}}\right)^{t+1},
\end{aligned}
\]
and in the last inequality we have used $\sum_{\tau=0}^{t-1}\frac{1}{(t-\tau)^2}\le 1+ \int_{1}^{\infty} \frac{1}{x^2} dx\le 2$. By noting $(1-x)^t \le \exp{\{-xt\}}$ we can show that $\|w^{(t)} - w^*\| \le \epsilon$ is valid when
\[
t \ge 2\sqrt{\frac{\mu+2\gamma}{\mu}} \log \left(\frac{2\sqrt{2}c\|w^{(0)} - w^*\|}{\epsilon}\right).
\]
This concludes the proof.
\end{proof}

\DANEHBQuadraticCorol*
\begin{proof}
Since $H(w)\equiv H$ and $H_1(w)\equiv H_1$ in the quadratic case, we know from Lemma~\ref{lemma:hessian_close} that $\|H_1 - H\|\le \gamma=L\sqrt{\frac{32\log(p/\delta)}{n}}$ holds with probability at least $1-\delta$. By invoking Theorem~\ref{thrm:quadratic_hb} we obtain the desired bound.
\end{proof}

\subsection{Proof of Theorem~\ref{thrm:local_general_hb}}\label{append:proof_of_dane_hb_local}

Here we give a detailed proof of Theorem~\ref{thrm:local_general_hb} which is restated as in the following.
\DANEHBLocalrate*

\begin{proof}[of Theorem~\ref{thrm:local_general_hb}]
The proof mimics that of Lemma~\ref{lemma:local_unit} with proper adaptation to the heave-ball momentum formulation. For the sake of completeness, here we provide the full details of proof.  Since $\nabla F(w^*)=0$, we can show the following:
\[
\begin{aligned}
&\nabla P^{(t-1)}(\tilde w^{(t)})\\
=&\nabla F_1(\tilde w^{(t)}) -  \nabla F_1(w^{(t-1)}) + \nabla F(w^{(t-1)}) + \gamma (\tilde w^{(t)} - w^{(t-1)})\\
=&\nabla F_1(\tilde w^{(t)}) - \nabla F_1 (w^*) + \nabla F_1(w^*) - \nabla F_1(w^{(t-1)}) + \nabla F(w^{(t-1)}) - \nabla F(w^*) + \gamma (\tilde w^{(t)} - w^{(t-1)})\\
=& \Delta F_1(\tilde w^{(t)}, w^*) + \nabla^2 F_1(w^*)(\tilde w^{(t)}-w^*) - \Delta F_1(w^{(t-1)}, w^*) - \nabla^2 F_1(w^*)(w^{(t-1)}-w^*)\\
&+ \Delta F(w^{(t-1)}, w^*) + \nabla^2 F(w^*)(w^{(t-1)}-w^*) +  \gamma (\tilde w^{(t)} - w^{(t-1)})\\
=&  \Delta F_1(\tilde w^{(t)}, w^*) + (\nabla^2 F_1(w^*)+\gamma I)(\tilde w^{(t)}-w^*) - \Delta F_1(w^{(t-1)}, w^*) \\
&- (\nabla^2 F_1(w^*)+\gamma I)(w^{(t-1)}-w^*) + \Delta F(w^{(t-1)}, w^*) + \nabla^2 F(w^*)(w^{(t-1)}-w^*).
\end{aligned}
\]
Then by multiplying $(\nabla^2 F_1(w^*)+\gamma I)^{-1}$ on both sides of the above and after proper rearrangement we obtain
\[
\begin{aligned}
&\tilde w^{(t)} - w^* \\
=& \left(I - (\nabla^2 F_1(w^*)+\gamma I)^{-1}\nabla^2 F(w^*)\right)(w^{(t-1)} - w^*) \\
&+  (\nabla^2 F_1(w^*)+\gamma I)^{-1} \left( \nabla P^{(t-1)}(\tilde w^{(t)}) -  \Delta F_1(\tilde w^{(t)}, w^*) +  \Delta F_1(w^{(t-1)}, w^*) - \Delta F(w^{(t-1)}, w^*)\right)\\
=&  \left(I - (\nabla^2 F_1(w^*)+\gamma I)^{-1}\nabla^2 F(w^*)\right)(w^{(t-1)} - w^*) \\
&+  (\nabla^2 F_1(w^*)+\gamma I)^{-1} \left(\Delta F_1(w^{(t-1)}, w^*) - \Delta F(w^{(t-1)}, w^*)-  \Delta F_1(\tilde w^{(t)}, w^*)+\nabla P^{(t-1)}(\tilde w^{(t)})\right).
\end{aligned}
\]
Recall the update $w^{(t)} = \tilde w^{(t)} + \beta (w^{(t-1)} - w^{(t-2)})$. It follows that
\[
\begin{aligned}
&w^{(t)} - w^* \\
=& \left((1+\beta)I - (\nabla^2 F_1(w^*)+\gamma I)^{-1}\nabla^2 F(w^*)\right)(w^{(t-1)} - w^*) - \beta (w^{(t-2)} - w^*)\\
&+(\nabla^2 F_1(w^*)+\gamma I)^{-1} \left(\Delta F_1(w^{(t-1)}, w^*) - \Delta F(w^{(t-1)}, w^*)-  \Delta F_1(\tilde w^{(t)}, w^*) + \nabla P^{(t-1)}(\tilde w^{(t)})\right).
\end{aligned}
\]
Let $H^* = \nabla^2 F(w^*)$ and $H^* = \nabla^2 F_1(w^*)$. Similar to the previous analysis, we work on the three term recurrence in matrix form
\begin{equation}\label{equat:iteraiton_form}
u^{(t)} = A u^{(t-1)} + r^{(t-1)}
\end{equation}
where $u^{(t)}:=\left[\begin{array}{*{20}{c}}
   w^{(t)} - w^*  \\
   w^{(t-1)} - w^*  \\
\end{array}
\right]$, $A:=\left[\begin{array}{*{20}{c}}
   (1+\beta) I - (H_1^* + \gamma I)^{-1} H^* & -\beta I \\
   I & 0 \\
\end{array} \right]$ and
\[
r^{(t-1)}:=\left[\begin{array}{*{20}{c}}
(H_1^* + \gamma I)^{-1}\left(\Delta F_1(w^{(t-1)}, w^*) - \Delta F(w^{(t-1)}, w^*)-  \Delta F_1(\tilde w^{(t)}, w^*) +\nabla P^{(t-1)}(\tilde w^{(t)}) \right) \\
0\\
\end{array}
\right].
\]
Under the condition $\varepsilon_t \le \min\left\{(\gamma+\mu)^2, \|\nabla F(w^{(t-1)})\|^2/L^2\right\}$, using the similar argument as in the proof of Lemma~\ref{lemma:local_unit}, we can bound $\|r^{(t-1)}\|$ with respect to $\|u^{(t-1)}\|$ as
\[
\|r^{(t-1)}\| \le  \frac{6\nu + 1}{\gamma+\mu}\|u^{(t-1)}\|^2.
\]
Since $\|H_1^* - H^*\|\le\gamma$ and $H^*\succeq \mu I$, by applying Lemma~\ref{lemma:precondition_bound} we obtain that $(H_1^* + \gamma I)^{-1}H^*$ is diagonalizable and
\[
\frac{\mu}{\mu + 2\gamma} \le \lambda_{\min}((H_1^* + \gamma I)^{-1}H^*) \le \lambda_{\max} ((H_1^* + \gamma I)^{-1}H^*) \le 1.
\]
Given $\beta = \left(1-\sqrt{\frac{\mu}{\mu+2\gamma}}\right)^2$, it is known from Lemma~\ref{lemma:spectral_radius_bound} (with  $\eta =1$) that
\[
\rho(A) \le 1 - \sqrt{\frac{\mu}{\mu+2\gamma}}.
\]
Let $\delta = \frac{1-\rho(A)}{2}$. From Lemma~\ref{lemma:special_radius_norm} we know that there exists a constant $c=c(\delta)$ such that for all $t\ge 0$:
\begin{equation}\label{inequat:proof_general_thrm_key3}
\|A^t\| \le c(\rho(A)+\delta)^t = c\left(\frac{1+\rho(A)}{2}\right)^t \le c\left(1-\frac{1}{2}\sqrt{\frac{\mu}{\mu+2\gamma}}\right)^t.
\end{equation}
Without loss of generality we assume $c\ge 1$. In the following argument, to simplify notation, we abbreviate $\vartheta = \frac{6\nu+1}{\gamma+\mu}$ and $\rho =1-\frac{1}{2}\sqrt{\frac{\mu}{\mu+2\gamma}}$. Let us consider the following defined integer
\[
\tau= \left\lceil 2\sqrt{\frac{\mu+2\gamma}{\mu}} \log(2c)  \right\rceil
\]
such that $\|A^\tau\|\le \frac{1}{2}$. We now prove by induction that for any integer $k\ge 0$, $\max_{0\le i\le \tau-1}\|u^{(k\tau+i)}\|\le \frac{1}{4c\tau\vartheta}\left(\frac{3}{4}\right)^k$.  The assumption $\max_{-1\le i \le \tau-1}\|w^{(i)} - w^*\| \le \frac{1}{4\sqrt{2}c\tau\vartheta }$ guarantees that the bound is valid for the case $k=0$, i.e., $\max_{0\le i \le \tau-1}\|u^{(i)}\| \le \frac{1}{4c\tau\vartheta}$. Now assume that $\max_{0\le i\le \tau-1}\|u^{(k\tau+i)}\|\le \left(\frac{3}{4}\right)^k \frac{1}{4c\tau\vartheta}$ for some $k\ge 0$. By recursively applying~\eqref{equat:iteraiton_form} we obtain
\[
\begin{aligned}
\|u^{((k+1)\tau)}\| &= \left\|A^{\tau} u^{(k\tau)} + \sum_{i=0}^{\tau-1} A^i r^{(k\tau+\tau-1-i)}\right\| \le \|A^\tau\|\|u^{(k\tau)}\| + \sum_{i=0}^{\tau-1} \|A^i\| \|r^{(k\tau+\tau-1-i)}\| \\
&\overset{\zeta_1}{\le}  \frac{1}{2}\|u^{(k\tau)}\|+ \vartheta c\sum_{i=0}^{\tau-1} \|u^{(k\tau + \tau-1-i)}\|^2 \overset{\zeta_2}{\le} \frac{1}{2} \left(\frac{3}{4}\right)^k \frac{1}{4c\tau\vartheta} + \frac{1}{4\tau}\sum_{i=0}^{\tau-1}\|u^{k\tau+i}\| \\
&\overset{\zeta_3}{\le} \frac{1}{2} \left(\frac{3}{4}\right)^k \frac{1}{4c\tau\vartheta} + \frac{1}{4} \left(\frac{3}{4}\right)^k \frac{1}{4c\tau\vartheta} = \left(\frac{3}{4}\right)^{k+1} \frac{1}{4c\tau\vartheta},
\end{aligned}
\]
where ``$\zeta_1$'' is due to~\eqref{inequat:proof_general_thrm_key3} which implies $\|A^i\|\le c$ for all $i\ge 1$ and it also has used $\|r^{(t)}\| \le \vartheta \|u^{(t)}\|^2$, ``$\zeta_2$'' and ``$\zeta_3$'' are based on the induction step and $\|u^{k\tau+i}\|\le \frac{1}{4c\tau\vartheta}$ for all $0\le i \le \tau-1$. By using the same argument as the above, we can show that $\|u^{((k+1)\tau+i)}\|\le \frac{1}{4c\tau\vartheta}\left(\frac{3}{4}\right)^{k+1}$ for all $0\le i \le \tau-1$. This proves that $\max_{0\le i\le \tau-1} \|u^{(k\tau+i)}\|\le \frac{1}{4c\tau\vartheta}\left(\frac{3}{4}\right)^k$ holds for all $k\ge0$. Particularly, we obtain
\[
\|w^{(k\tau)} - w^*\|\le\|u^{k\tau}\|\le \frac{1}{4c\tau\vartheta}\left(\frac{3}{4}\right)^k.
\]
Therefore, to reach $\|w^{(t)} - w^*\|\le\epsilon$ we need $t \ge 4\tau \log \left(\frac{1}{4c\tau\vartheta \epsilon}\right)$. This completes the proof.
\end{proof}

\subsection{Proof of Theorem~\ref{thrm:global_dane_hb_lm}}\label{append:proof_of_global_dane_hb_lm}

In this subsection, we prove Theorem~\ref{thrm:global_dane_hb_lm} as restated below.
\DANEHBLMConvergence*

\begin{proof}
We first analyze the outer-loop iteration complexity. As defined in Algorithm~\ref{alg:dane_hb_lm} that at each time instance $t$ the quadratic subproblem is optimized to certain $\varepsilon_t$-suboptimality.
\[
Q^{(t-1)}(w^{(t)}) \le \min_w Q^{(t-1)}(w) + \varepsilon_t.
\]
The value of $\varepsilon_t$ will be specified shortly in the following analysis. Let us abbreviate $l_i(w^\top x_i) = l(w^\top x_i, y_i)$ with $l_i$ being a univariate function. For any $\eta \in [0,1]$, the smoothness of $l_i$ and the suboptimality of $w^{(t)}$ lead to
\[
\begin{aligned}
&F(w^{(t)}) \\
=&\tilde F(w^{(t)})+ \frac{\mu}{2}\|w^{(t)}\|^2 = \frac{1}{N}\sum_{i=1}^N l_i(x_i^\top w^{(t)}) + \frac{\mu}{2}\|w^{(t)}\|^2 \\
\le& \frac{1}{N}\sum_{i=1}^N \left\{l_i(x_i^\top w^{(t-1)}) + l_i'(x_i^\top w^{(t-1)})x_i^\top (w^{(t)} - w^{(t-1)})\right. \\
&\left. + \frac{\ell}{2}(w^{(t)} - w^{(t-1)})^\top x_i x_i^\top (w^{(t)} - w^{(t-1)})\right\} + \frac{\mu}{2}\|w^{(t)}\|^2 \\
=& \tilde F(w^{(t-1)}) + \langle \nabla \tilde F(w^{(t-1)}), w^{(t)} - w^{(t-1)}\rangle + \frac{\ell}{2N} (w^{(t)}-w^{(t-1)})^\top  X X^\top(w^{(t)} - w^{(t-1)}) \\
&+ \frac{\mu}{2}\|w^{(t)}\|^2\\
=& Q^{(t-1)}(w^{(t)}) \\
\le& Q^{(t-1)}((1-\eta)w^{(t-1)}+\eta w^*) + \varepsilon_t \\
=& \tilde F(w^{(t-1)}) + \eta \langle \nabla \tilde F(w^{(t-1)}), w^* - w^{(t-1)}\rangle + \frac{\eta^2\ell}{2N} (w^*-w^{(t-1)})^\top X X^\top (w^* - w^{(t-1)})\\
 &+ \frac{\mu}{2}\left( (1-\eta)w^{(t-1)}+\eta w^* \right)^2 + \varepsilon_t \\
=& \tilde F(w^{(t-1)}) + \eta \langle \nabla \tilde F(w^{(t-1)}), w^* - w^{(t-1)}\rangle + \frac{\eta^2\ell}{2N} (w^*-w^{(t-1)})^\top X X^\top (w^* - w^{(t-1)})\\
 &+ \frac{\mu}{2}\|w^{(t-1)}\|^2 + \mu\eta \langle w^{(t-1)}, w^* - w^{(t-1)}\rangle + \frac{\mu\eta^2}{2}\|w^{(t-1)} - w^*\|^2 + \varepsilon_t \\
=& F(w^{(t-1)}) + \eta \langle \nabla F(w^{(t-1)}), w^* - w^{(t-1)}\rangle \\
&+ \frac{\eta^2\ell}{2} (w^*-w^{(t-1)})^\top \left(\frac{ X X^\top}{N} + \frac{\mu}{\ell} I \right) (w^* - w^{(t-1)}) + \varepsilon_t .
\end{aligned}
\]
On the other side, from the strong-convexity of $l_i(\cdot)$ we can show that
\[
\begin{aligned}
&F(w^*)\\
 =& \frac{1}{N}\sum_{i=1}^N l_i(x_i^\top w^*) + \frac{\mu}{2}\|w^*\|^2 \\
\ge& \frac{1}{N}\sum_{i=1}^N \left\{f_i(x_i^\top w^{(t-1)}) + f_i'(x_i^\top w^{(t-1)})x_i^\top (w^* - w^{(t-1)})^\top + \frac{\sigma}{2}(w^* - w^{(t-1)})^\top x_i x_i^\top (w^* - w^{(t-1)})\right\} \\
& + \frac{\mu}{2}\|w^{(t-1)}\|^2 + \mu \langle w^{(t-1)}, w^* - w^{(t-1)}\rangle + \frac{\mu}{2}\|w^* - w^{(t-1)}\|^2\\
=& F(w^{(t-1)}) + \langle \nabla F(w^{(t-1)}), w^* - w^{(t-1)}\rangle + \frac{\sigma}{2} (w^*-w^{(t-1)})^\top\left(\frac{ X X^\top}{N} + \frac{\mu}{\sigma} I \right) (w^* - w^{(t-1)})\\
\ge& F(w^{(t-1)}) + \langle \nabla F(w^{(t-1)}), w^* - w^{(t-1)}\rangle + \frac{\sigma}{2} (w^*-w^{(t-1)})^\top\left(\frac{ X X^\top}{N} + \frac{\mu}{\ell} I \right) (w^* - w^{(t-1)}),
\end{aligned}
\]
where in the last inequality we have used $\ell\ge\sigma$. By setting $\eta = \sigma / \ell\in (0,1]$ and combining the above two inequalities we arrive at
\[
F(w^{(t)}) - F(w^*) \le \left(1-\frac{\sigma}{\ell}\right) (F(w^{(t-1)}) - F(w^*) ) + \varepsilon_t.
\]
Let us consider
\[
\varepsilon_t \le \frac{\sigma\mu}{4\ell L^2}\|\nabla F(w^{(t-1)})\|^2
\]
which implies
\begin{equation}\label{inequat:proof_dane_hb_lm_key1}
\varepsilon_t \le \frac{\sigma}{2\ell} (F(w^{(t-1)}) - F(w^*) )
\end{equation}
and thus
\[
F(w^{(t)}) - F(w^*) \le \left(1-\frac{\sigma}{2\ell}\right) (F(w^{(t-1)}) - F(w^*) ).
\]
Recursively applying the above recursion form yields
\[
F(w^{(t)}) - F(w^*) \le \left(1-\frac{\sigma}{2\ell}\right)^t (F(w^{(0)}) - F(w^*)) .
\]
Then for any desired precision $\epsilon>0$, the sub-optimality $F(w^{(t)}) - F(w^*) \le \epsilon$ holds provided that
\[
t \ge \frac{2\ell}{\sigma} \log \left(\frac{(F(w^{(0)}) - F(w^*))}{\epsilon}\right).
\]
From Theorem~\ref{thrm:quadratic_hb} and~\eqref{inequat:proof_dane_hb_lm_key1} we know that the condition $Q^{(t-1)}(w^{(t)}) \le \min_w Q^{(t-1)}(w) + \varepsilon_t$ is valid when the inner loop is sufficiently executed with $ \mathcal{O}\left(\sqrt{\frac{\gamma}{\mu}} \log \left(\frac{1}{\varepsilon_t}\right)\right)=\mathcal{O}\left(\sqrt{\frac{\gamma}{\mu}} \log \left(\frac{1}{\epsilon}\right)\right)$ rounds of iteration. Therefore, the overall inner-loop iteration complexity is $t\tau$ which is of the order
\[
\mathcal{O}\left(\frac{\ell }{\sigma}\sqrt{\frac{\gamma}{\mu}} \log^2 \left(\frac{1}{\epsilon}\right)\right).
\]
This proves the desired bound.
\end{proof}

\section{Proof of Auxiliary Lemmas}\label{ProofforAuxiliaryLemmas}

\subsection{Proof of Lemma~\ref{lemma:precondition_bound}}

\begin{proof}
Since both $A+\gamma I$ and $B$ are symmetric and positive definite, it is known that the eigenvalues of $(A+\gamma I)^{-1}B$ are positive real numbers and identical to those of $(A+\gamma I)^{-1/2}B(A+\gamma I)^{-1/2}$. Let us  consider the following eigenvalue decomposition of $(A+\gamma I)^{-1/2}B(A+\gamma I)^{-1/2}$:
\[
(A+\gamma I)^{-1/2}B(A+\gamma I)^{-1/2} = Q^\top \Lambda Q,
\]
where $Q^\top Q = I$ and $\Lambda$ is a diagonal matrix with eigenvalues as diagonal entries. It is then implied that
\[
(A+\gamma I)^{-1}B = (A+\gamma I)^{-1/2} Q^\top \Lambda Q(A+\gamma I)^{1/2},
\]
which is a diagonal eigenvalue decomposition of $(A+\gamma I)^{-1}B $. Thus $(A+\gamma I)^{-1}B$ is diagonalizable.

To prove the eigenvalue bounds of $(A+\gamma I)^{-1}B$, it suffices to prove the same bounds for $(A+\gamma I)^{-1/2}B(A+\gamma I)^{-1/2}$. Since $\|A-B\|\le \gamma$, we have $B \preceq A + \gamma I$ which implies $(A+\gamma I)^{-1/2}B(A+\gamma I)^{-1/2}\preceq I$ and hence $\lambda_{\max}((A+\gamma I)^{-1/2}B(A+\gamma I)^{-1/2})\le 1$. Moreover, since $B\succeq \mu I$, it holds that $\frac{2\gamma}{\mu} B - \gamma I \succeq \gamma I \succeq A - B$. Then we obtain $(A+\gamma I)^{-1/2}B(A+\gamma I)^{-1/2}\succeq \frac{\mu}{\mu + 2\gamma}I$ which implies $\lambda_{\min}((A+\gamma I)^{-1/2}B(A+\gamma I)^{-1/2})\ge \frac{\mu}{\mu + 2\gamma}$. Similarly, we can show that $\frac{\mu}{\mu + 2\gamma}I \preceq B^{1/2}(A+\gamma I)^{-1}B^{1/2} \preceq I$, implying $\|I - B^{1/2}(A+\gamma I)^{-1} B^{1/2}\| \le  \frac{2\gamma}{\mu + 2\gamma}$.
\end{proof}

\subsection{Proof of Lemma~\ref{lemma:spectral_radius_bound}}

\begin{proof}
Let $0<\mu\le\lambda_1\le \lambda_2 \le \cdots \le \lambda_d \le L$ be the eigenvalues of $A$ and $\Lambda$ be a diagonal matrix whose diagonal entries are $\{\lambda_i\}$ in a non-decreasing order. Since $A$ is diagonalizable, it can be verified that the eigenvalues of the following two $2d \times 2d$ matrices coincide:
\[
T_1 = \left[\begin{array}{*{20}{c}}
   (1+\beta) I -\eta A & -\beta I \\
   I & 0 \\
\end{array} \right], \quad T_2 = \left[\begin{array}{*{20}{c}}
   (1+\beta) I -\eta \Lambda & -\beta I \\
   I & 0 \\
\end{array} \right].
\]
It is possible to permute the matrix $T_2$ to a block diagonal matrix with $2\times 2$ blocks of the form
\[
\left[ { {\begin{array}{*{20}{c}}
   1 + \beta-\eta\lambda_i & -\beta \\
   1 & 0 \\
\end{array}}
} \right].
\]
Therefore we have
\[
\begin{aligned}
\rho\left(\left[ { {\begin{array}{*{20}{c}}
   (1+\beta)I-\eta A & -\beta I \\
   I & 0 \\
\end{array}}
} \right]\right)=& \rho\left(\left[ { {\begin{array}{*{20}{c}}
   (1+\beta)I - \eta \Lambda & -\beta I \\
   I & 0 \\
\end{array}}
} \right]\right) \\
 =& \max_{i\in [d]}\rho\left(\left[ { {\begin{array}{*{20}{c}}
   1 + \beta-\eta\lambda_i & -\beta \\
   1 & 0 \\
\end{array}}
} \right]\right).
\end{aligned}
\]
For each $i\in [d]$, the eigenvalues of the $2\times 2$ block matrices are given by the roots of
\[
\lambda^2 - (1+ \beta - \eta \lambda_i)\lambda + \beta = 0.
\]
Given that $\beta \ge |1-\sqrt{\eta \lambda_i}|^2$, the roots of the above equation are imaginary and both have magnitude $\sqrt{\beta}$. Since $\beta = \max\{|1-\sqrt{\eta \mu}|^2, |1-\sqrt{\eta L}|^2\}$, the magnitude of each root is at most $\max\{|1-\sqrt{\eta \mu}|, |1-\sqrt{\eta L}|\}$. This proves the desired spectral radius bound.
\end{proof}

\subsection{Proof of Lemma~\ref{lemma:grad_bound}}

\begin{proof}
From the local sub-optimality condition we have
\[
\|\nabla P^{(t-1)}(\tilde w^{(t)} )\| = \|\nabla F_1(\tilde w^{(t)}) + \nabla F(w^{(t-1)}) -  \nabla F_1(w^{(t-1)}) + \gamma (\tilde w^{(t)} - w^{(t-1)})\| \le \varepsilon_t.
\]
Then we can show that
\[
\begin{aligned}
&\|\nabla F(\tilde w^{(t)})\| \\
=& \|\nabla F(\tilde w^{(t)})- \nabla P^{(t-1)}(\tilde w^{(t)} ) + \nabla P^{(t-1)}(\tilde w^{(t)} )\| \\
\le&\|\nabla F(\tilde w^{(t)}) - \nabla F_1(\tilde w^{(t)}) - \nabla F(w^{(t-1)}) +   \nabla F_1(w^{(t-1)}) - \gamma (\tilde w^{(t)} - w^{(t-1)})\| + \varepsilon_t \\
\le& \|(\nabla^2 (F - F_1)(w') + \gamma I)(\tilde w^{(t)} - w^{(t-1)}) \| + \varepsilon_t \\
\le& 2\gamma\|\tilde w^{(t)} - w^{(t-1)}\| + \varepsilon_t,
\end{aligned}
\]
where in the last inequality we have used $\sup_{w}\|\nabla^2 F(w) - \nabla^2 F_1(w)\|\le \gamma$. This proves the first inequality. The second inequality follows readily from the strong convexity of $F$ such that $\mu\|\tilde w^{(t)} - w^*\|\le\|\nabla F(\tilde w^{(t)}) - \nabla F(w^*)\|=\|\nabla F(\tilde w^{(t)})\|$.
\end{proof}

\section{Computational complexity of DANE-HB}\label{append:computation_cost_dane_hb}

In addition to communication complexity, here we further provide a computational complexity analysis for DANE-HB in order to gain better understanding of its overall computational efficiency. We first restrict our attention to the quadratic setting in which the global convergence of DANE-HB is guaranteed. At each communication round $t$, the master machine needs to solve the local subproblem $\tilde w^{(t)} \approx \argmin_w  P^{(t-1)}(w)$ to certain desired precision. Inspired by Federated SVRG~\citep{konevcny2016federated} which essentially applies SVRG~\citep{johnson2013accelerating} to the local optimization of \InexactDane, we specify that the local minimization of DANE-HB is implemented with the SVRG solver. Clearly such a specification of DANE-HB only needs to access the first-order information of the loss functions. Following~\citep{johnson2013accelerating,zhang2017stochastic}, we employ the incremental first order oracle (IFO) complexity as the computational complexity metric for solving the finite-sum minimization problem~\eqref{eqn:general}.
\begin{definition}\label{def:IFO}
 An IFO takes an index $i \in [N]$ and a point $(x_i,y_i) \in \{x_j,y_j\}_{j=1}^N$, and returns the pair $(f(w; x_i,y_i),\nabla f(w; x_i,y_i))$.
\end{definition}

As a consequence of Corollary~\ref{corol:quadratic_hb}, the following result summaries the computational complexity of DANE-HB in the considered setting.

\begin{restatable}[Computational complexity of DANE-HB for quadratic objective]{corollary}{DANEHBComputationcost}\label{corol:quadratic_hb_ifo}
Assume the conditions in Corollary~\ref{corol:quadratic_hb} hold and the local subproblems are solved using SVRG. Then with high probability, the IFO complexity of DANE-HB for attaining estimation error $\|w^{(t)} - w^*\| \le \epsilon$ is of the order
\[
\mathcal{O} \left(\sqrt{\kappa} \left(n^{3/4} + n^{1/4}\right) \log^2 \left(\frac{1}{\epsilon}\right) + \sqrt{\kappa} n^{3/4} \log \left(\frac{1}{\epsilon}\right)\right).
\]
\end{restatable}
\begin{proof}
Recollect that $\gamma=L\sqrt{\frac{32\log(p/\delta)}{n}}$ in Corollary~\ref{corol:quadratic_hb}. It is standard to know that the IFO complexity of the inner-loop SVRG computation can be bounded with high probability by
\[
\mathcal{O}\left(\left(n + \frac{L+\gamma}{\gamma+\mu}\right)\log \left(\frac{1}{\varepsilon}\right)\right) = \mathcal{O}\left(\left(n + \sqrt{n}\right)\log \left(\frac{1}{\epsilon}\right)\right).
\]
From Corollary~\ref{corol:quadratic_hb} we know that with high probability, the outer-loop communication complexity is of the order
\[
\mathcal{O} \left(\frac{\sqrt{\kappa}}{n^{1/4}}\log \left(\frac{1}{\epsilon}\right)\right).
\]
For each communication round, each machine needs to compute the local batch gradient, which can be done in parallel. Combing the above inner-loop and outer-loop IFO bounds yields the following overall computation complexity bound
\[
\mathcal{O} \left(\sqrt{\kappa} \left(n^{3/4} + n^{1/4}\right)\log^2 \left(\frac{1}{\epsilon}\right) + \sqrt{\kappa} n^{3/4}\log \left(\frac{1}{\epsilon}\right)\right),
\]
which holds with high probability.
\end{proof}
For an instance, let us consider the conventional statistical learning setting where the condition number $\kappa$ is as large as $\mathcal{O}(\sqrt{N})=\mathcal{O}(\sqrt{mn})$. In this case, the above result implies that the IFO complexity bound of DANE-HB is
\[
\mathcal{O} \left(\left(m^{1/4}n + m^{1/4}n^{1/2}\right) \log^2 \left(\frac{1}{\epsilon}\right)+m^{1/4}n \log \left(\frac{1}{\epsilon}\right) \right).
\]
To comparison with SVRG, the expected IFO complexity bound of SVRG is given by
\[
\mathcal{O} \left(\left(mn + \sqrt{mn}\right)\log \left(\frac{1}{\epsilon}\right)\right).
\]
Since the sample size $mn$ dominates the condition number $\sqrt{mn}$ in this example, up to the logarithm factors, DANE-HB is roughly $\times m^{3/4}$ cheaper than SVRG in computational cost, which also matches the result established for MP-DANE~\citep{wang2017memory}

By combining Theorem~\ref{thrm:global_dane_hb_lm} and Corollary~\ref{corol:quadratic_hb_ifo}, we can readily establish the following result on the overall IFO complexity bound of DANE-HB-LM for linear models.
\begin{corollary}[Computation complexity of DANE-HB-LM]\label{corol:quadratic_hb_lm_ifo}
Assume the conditions in Corollary~\ref{corol:quadratic_hb} hold and the local subproblems are solved using SVRG. Then with high probability the IFO complexity of DANE-HB for the quadratic objective function is of the order
\[
\mathcal{\tilde O} \left(\frac{\ell\sqrt{\kappa}}{\sigma} \left(n^{3/4} + n^{1/4}\right)\log^3 \left(\frac{1}{\epsilon}\right) + \frac{\ell\sqrt{\kappa}}{\sigma} n^{3/4}\log^2 \left(\frac{1}{\epsilon}\right)\right).
\]
\end{corollary}

\bibliography{mybib}
\bibliographystyle{jmlr2e}

\end{document}